\documentclass[11pt]{article}

\usepackage{mystyle}

\usepackage{fullpage}
\usepackage{hyperref}
\usepackage{tikz}

\usetikzlibrary{decorations.pathreplacing}
\def\rd{\mathrm{d}}

\def\dive{\mathrm{div}}

\def\sC{\mathscr{C}}
\def\ow{\omega}

\begin{document}	

\title{A Mean-Field Analysis of Neural Stochastic Gradient Descent-Ascent for Functional Minimax Optimization}
\author{Yuchen Zhu\thanks{Georgia Institute of Technology. Email: \texttt{yzhu738@gatech.edu}.} \and Yufeng Zhang \thanks{Northwestern University. Email: \texttt{yufengzhang2023@u.northwestern.edu}.} \and Zhaoran Wang \thanks{Northwestern University. Email: \texttt{zhaoranwang@gmail.com}.}\and Zhuoran Yang\thanks{Yale University. Email: \texttt{zhuoran.yang@yale.edu}. } \and Xiaohong Chen\thanks{Yale University. Email: \texttt{xiaohong.chen@yale.edu}.}}
\date{}

\maketitle

\begin{abstract}

This paper studies minimax optimization problems defined over infinite-dimensional function classes of overparameterized two-layer neural networks.
In particular, we consider the minimax optimization problem stemming from estimating linear functional equations defined by conditional expectations, where the objective functions are quadratic in the functional spaces. 
We address (i) the convergence of the stochastic gradient descent-ascent algorithm and (ii) the representation learning of the neural networks. 
We establish convergence under the mean-field regime by considering the continuous-time and infinite-width limit of the optimization dynamics.
Under this regime, the stochastic gradient descent-ascent corresponds to a Wasserstein gradient flow over the space of probability measures defined over the space of neural network parameters. 
We prove that the Wasserstein gradient flow converges globally to a stationary point of the minimax objective at a $\cO(T^{-1} + \alpha^{-1} ) $ sublinear rate, and additionally finds the solution to the functional equation when the regularizer of the minimax objective is strongly convex.
Here $T$ denotes the time and $\alpha$ is a scaling parameter of the neural networks.
In terms of representation learning, our results show that the feature representation induced by the neural networks is allowed to deviate from the initial one by the magnitude of $\cO(\alpha^{-1})$, measured in terms of the Wasserstein distance.
Finally, we apply our general results to concrete examples including policy evaluation, nonparametric instrumental variable regression, asset pricing, and adversarial Riesz representer estimation.

\end{abstract}

%!TEX root = main.tex

\section{Introduction}
\iffalse 
\begin{itemize}
   \item Many problems in statistics and machine learning are formulated as minimax optimization. Examples include policy evaluation, instrumental variable regression, and Riesz representer estimation. 
   \item In this work, we focus on solving these minimax optimization problems using neural networks. 
   \item What's the challenge? 
   \item Nonconvex-nonconcave objective. Track the representation learning of the neural networks.
\end{itemize}
\fi 

Minimax optimization problems are ubiquitous in machine learning, statistics, economics, and other fields. Examples include generative adversarial networks (GANs) \citep{goodfellow2020generative, salimans2016improved}, adversarial training \citep{ganin2016domain, madry2017towards},  robust optimization \citep{ben2009robust, levy2020large}, and zero-sum games \citep{xie2020learning, zhao2022provably}.  The goal in minimax optimization is to find a solution $(f^*, g^*)$ to the problem 
$ \min_{f \in \cF} \max_{g \in \cG} \cL(f, g) $, where $\cL$ is a bivariate  objective function, and 
$\cF$ and $\cG$ are the feasible sets of the decision variables $f$ and $g$.
In modern machine learning applications,  $\cF$ and $\cG$ are often function classes flexibly parameterized by neural networks, and the objective $\cL(f,g)$ can be approximated using data. The minimax optimization problem is often solved using first-order optimization algorithms. Despite hugely successful in diverse applications, there is no global convergence theory for various popular first-order algorithms solving general minimax optimization using neural networks yet. \vspace{5pt}

\noindent In this work, we study the convergence of first-order algorithms for solving minimax optimization problems where $\cF$ and $\cG$ are both flexibly parameterized by two-layer neural networks, and the objective functional is quadratic in $f$ and $g$ up to regularization: 
\begin{align}\label{eq:minimax_obj1}
\min_{f\in\cF}\max_{g\in\cG} 
 \cL(f, g),~~ \cL(f, g)= \EE \bigl[ g(Z)\cdot \Phi(X, Z; f) - 1/2\cdot g(Z)^2 + \mathtt{Reg}(f) \bigr],
\end{align}
where $\mathtt{Reg}(f)$ is a convex regularizer that penalizes the complexity of $f\in \cF$.
Here the expectation is taken with respect to the joint distribution of random variables $(X, Z)$,   $g$ is a function of $Z$, and  $\Phi$ takes $(X, Z)$ and a function $f$ as its input and is linear in $f$. 
The objective function \eqref{eq:minimax_obj1} arises from solving a linear functional conditional moment equation of the form $\EE[\Phi(X,Z;  f)|Z= \cdot ] = 0$ if and only if $f=f^*\in  \cF$.
Here $X$ is a vector containing all the endogenous variables and $Z$ contains all the exogenous/pre-determent variables. 
This problem has ample applications, including policy evaluation \citep{cai2019neural, duan2020minimax, jin2021pessimism, chen2022well, ramprasad2022online}, nonparametric instrumental variable regression \citep{blundell2007semi, chen2012estimation, chen2018optimal, xu2020learning}, and asset pricing \citep{chen2009land, chen2014local, chen2024deep}.
The minimax objective in \eqref{eq:minimax_obj1} arises when we solve the conditional moment equation via adversarial estimation \citep{uehara2020minimax, duan2021risk, chernozhukov2020adversarial, liao2020provably, wai2020provably, bennett2019deep}, 
which introduces a dual function and transforms equation solving into a minimax optimization. \vspace{5pt}

\noindent We study the infinite-dimensional minimax optimization in \eqref{eq:minimax_obj1} over the space of overparameterized two-layer neural networks.  
Specifically, a neural network is represented by $f_{\mathtt{NN}} (\cdot; \btheta) = \alpha / N  \sum_{i=1}^N \phi(\cdot; \theta^i)$, where $N$ is the number of neurons, 
 $\phi(\cdot; \theta^i)$ denotes the $i$-th neuron,  $\{ \theta^i \}_{{i\in [N]}}$ are the network parameters, and $\alpha$ is a scaling parameter. 
 We aim to solve the minimax optimization in \eqref{eq:minimax_obj1} with both $f$ and $g$ are represented by overparameterized two-layer neural networks, which is favorable especially when $Z$ is a high-dimensional vector. 
 To solve this problem, we consider the arguably simplest first-order algorithm, stochastic gradient descent-ascent (SGDA), where the parameters of $f$ and $g$ are simultaneously updated using stochastic gradients of the objective functional. Specifically, we aim to address the following two questions:

\begin{itemize}
  \item \quad Does SGDA  with overparameterized neural networks converge to some solution? 
 \item  \quad Does SGDA  learn data-dependent features that yield a statistically accurate solution?
\end{itemize}

\noindent Answering these questions involves two intricate challenges in terms of optimization and representation learning using neural networks.  
First, the minimax objective is nonconvex-nonconcave with respect to the neural network parameters of $f$ and $g$, it is unclear whether first-order algorithms converge. 
Second, the representation of the neural network evolves during the course of optimization, and it is unclear how to track and assess the data-dependent features learned by the neural networks.
While there are some existing works on neural network optimization using the technique of neural tangent kernel (NTK) \citep{ jacot2018neural, du2018gradient, cai2019neural, xu2020finite, wang2022and}, such an approach suggests that the feature representation of the neural networks is fixed throughout training and is only determined by the initialization of the network parameters. 
Despite being an elegant theoretical framework, the NTK approach is limited in its ability to capture the representation learning aspect of neural network optimization. 
To show that the neural network optimization algorithms learn useful data-dependent features, in addition to establishing convergence, more importantly, we need to show that (i) the algorithm approximately finds a proper solution concept, e.g.,    a stationary point or a local or global optimizer of the minimax objective function, and  (ii) the representation of the neural networks moves from the initialization by a considerable amount. \vspace{5pt} 

\noindent In this paper, we tackle both challenges by leveraging the framework of mean-field analysis of overparameterized neural networks \citep{chizat2018global, mei2018mean, mei2019mean, zhang2020can, lu2020meanfield, zhang2021wasserstein, sirignano2020mean, sirignano2020mean1, sirignano2022mean, chen2020mean, fang2021modeling}. 
In particular,  we focus on the continuous-time and infinite width limit of the SGDA algorithm, where the stepsize goes to zero and the width $N$ goes to infinity. 
From the mean-field lens, a neural network  
$f(\cdot; \btheta)$ can be identified with a probability measure $\mu$ by writing $f(\cdot; \btheta) = \alpha \cdot \int _{\theta} \phi(\cdot; \theta) ~\mu(\ud \theta)$, where $\mu$ is the empirical distribution of $\{ \theta^i \}_{{i\in [N]}}$ and $\alpha $ is the scaling parameter of the neural network. 
Thus,   parameter updates of SGDA can be regarded as updates of the probability measure $\mu$. 
From this perspective, we  prove that in the  continuous-time and infinite width limit, SGDA corresponds to 
a gradient flow of the minimax objective $\cL$ in the Wasserstein space, i.e., the space of probability measures over the parameter space equipped with the Wasserstein-2 distance. 
Besides, by defining a proper potential function that characterizes the stationary point of the minimax objective, we prove that the Wasserstein gradient flow converges to a stationary point at a sublinear rate of $\cO(1/T + 1/\alpha )$, where $T$ is the time horizon and $\alpha$ is a scaling parameter of the neural network.
Moreover, we prove that the Wasserstein distance between the parameter distribution found by SGDA  and its initialization is $\cO(\alpha^{-1})$, which shows that the representation of the neural networks is allowed to move from the initialization by a considerable amount. Such a behavior is not captured by the NTK analysis, in which the representation is shown to be fixed at the initialization. 
Furthermore, when the regularization on $f$ satisfies a version of strong convexity, we prove that the Wasserstein gradient flow converges to the global optimizer $f^*$  at a sublinear  $\cO(1/T + 1/\alpha )$ rate. 

\vspace{5pt}

\noindent To the best of our knowledge, our work provides the first theoretical analysis of an optimization algorithm solving functional conditional moment equations using neural networks with representation learning. 
We apply our general theory to three important examples: policy evaluation, instrumental variables regression, and asset pricing.
and adversarial Riesz representer estimation. 
In these examples, we prove that the SGDA algorithm finds the global solution with overparameterized neural networks. Moreover, SGDA learns data-dependent features that enable these statistically accurate estimators.

\subsection{Related Works}

\noindent \textbf{Minimax Optimization.} 
Our work is closely related to the literature on first-order methods for solving minimax optimization problems. 
These works establish the convergence rate or iteration complexity of first-order methods under various assumptions on the objective function. 
In particular, most of the existing works focus on finite-dimensional parameter spaces and one of the following objective functions: 
(i) convex-concave \citep{lin2020near, ibrahim2019lower, ouyang2021lower, alkousa2019accelerated, luo2021near, xie2020lower, han2024lower, li2023nesterov, jin2022sharper}, (ii) nonconvex-concave \citep{jin2019minmax, lin2020gradient, lu2020hybrid, ostrovskii2021efficient, zhao2023primal, huang2022efficiently,luo2020stochastic, zhang2021complexity, nouiehed2019solving, thekumparampil2019efficient}, and (iii) nonconvex-nonconcave \citep{li2022nonsmooth, diakonikolas2021efficient, ostrovskii2021nonconvex, yang2022faster, grimmer2022limiting, hajizadeh2024linear, grimmer2023landscape, yang2020global}. \vspace{5pt}

\noindent Our work can be viewed as an extension of convex-concave minimax optimization to the infinite-dimensional functional space. 
In particular, our objective is a regularized quadratic functional with respect to the input functions, which is then restricted to the class of overparameterized neural networks. Note that the objective of interest is in fact nonconvex-nonconcave in the neural network parameter space. 
Compared with the work on general nonconvex-nonconcave minimax optimization problems, our setting has a better underlying structure in the functional space in terms of convexity. This structure enables us to lift the network parameter updates to the Wasserstein space and analyze the gradient flow in the space of distributions. Our approach leverages the hidden convexity-concavity behind the seemingly nonconvex-nonconcave objective function and thus achieves better results in terms of algorithm convergence and complexity. \vspace{5pt}

\noindent \textbf{Mean-field Analysis in Deep Learning.} 
Our work is closely related to the recent study of neural network training via gradient-based methods. One line of research establishes the convergence of gradient-based algorithms for training overparameterized neural networks under the ``lazy training'' regime, where the neural networks behave similarly to random kernel functions. Such a regime is also known as the as the neural tangent kernel regime \citep{jacot2018neural, allen2019learning, allen2019convergence, chen2020generalized, frei2021proxy, zou2019improved, du2018gradient, du2019gradient, arora2019exact, arora2019fine, huang2020dynamics}. Our work is more related to another line of research based on the perspective of mean-field approximation \citep{mei2018mean, mei2019mean, chizat2018global,sirignano2020mean, sirignano2020mean1, sirignano2022mean, chen2020mean, fang2021modeling, chen2019much}. 
Under the mean-field view, the neural network parameters are identified as a distribution over the parameter space. 
As a result, the evolution of parameters by gradient-based updates is captured by a differential equation that governs the evolution of the corresponding distribution. By elevating the training dynamics to an infinite Wasserstein space, 
the optimization objective often enjoys a benign landscape, which yields admits a more tractable analysis and global convergence. See, e.g, \cite{zhang2020can, zhang2021wasserstein, fang2021modeling, lu2020meanfield, fang2019over, chizat2022mean, hu2021mean, nitanda2022convex} and the references therein. Also, see \cite{fang2021mathematical} for a recent survey. \vspace{5pt}

\noindent Our work is especially related to the mean-field analysis of the Neural Temporal Difference (TD) \citep{, zhang2020can} and the Neural Actor-Critic (AC) \citep{zhang2021wasserstein} in reinforcement learning. These previous works have provided an analysis of the global convergence of the TD and AC algorithm with two-layered overparameterized neural networks. The optimization problem in these two tasks is the minimization of an objective where only one neural network is involved. Rather different from these works, we focus on minimax optimization, which requires neural network parameterization of both the primal function and the dual function. This brings new challenges to the analysis as the gradient dynamics of the primal and dual neural networks give birth to a coupled system of PDEs. To the best of our knowledge, our paper is the first to apply the mean-field limit to study the convergence of algorithms in solving the general form of functional conditional moment equations using neural networks.
 
\vspace{5pt}
\noindent \textbf{Adversarial Estimation.}
Our work is also related to the literature on adversarial estimation, 
a method that solves a functional conditional moment equation by introducing a dual function and reformulating the original problem into a minimax optimization. 
Our work studies this type of minimax optimization with overparameterized neural networks. Thus,  our work is more related to the study of adversarial estimation within neural network function classes \citep{dikkala2020minimax, chernozhukov2020adversarial, bennett2019deep, xu2021deep}. Compared with our work, these studies focus on statistical errors pertinent to neural networks, assuming the optimization problem is solved perfectly. We instead study the optimization algorithm and establish the convergence of stochastic gradient-descent-ascent with neural networks. \vspace{5pt}

\noindent Several previous works have also explored the convergence of optimization dynamics in adversarial estimation with neural networks. In particular, Neural GTD \citep{wai2020provably} and Neural SEM \citep{liao2020provably} analyze respectively the convergence for off-policy evaluation and structural equation models estimation with overparameterized two-layered neural network. 
However, their analyses are based on the idea of neural tangent kernel (NTK), 
where the employed neural network has a fixed representation during training, and the representation is completely determined by the initialization. In contrast, our work adopts the mean-field approach, which enables learning a data-dependent representation.

%!TEX root = main.tex

\section{Preliminaries}
\label{sec:pre}

The functional conditional moment equations cover many important examples in statistics, machine learning, economics, and causal inference. 
In this section,  we first introduce the general formulation of the functional conditional moment equations and then reformulate them into a minimax optimization problem. 
Then, we present a few concrete examples of function conditional moment equations such as policy evaluation, nonparametric instrumental variables regression, asset pricing, and Riesz representers estimation. Finally, we introduce the background of mean-field neural networks and Wasserstein space, which are essential for the convergence analysis of the SGDA algorithm.

\subsection{Functional Conditional Moment Equations} \label{sec:conditional moment model}
In this section, we introduce the general formulation of functional conditional moment equations.
Let $X \in \mathcal{X}$ be a vector that includes all the endogenous variables,  let $Z \in \mathcal{Z}$ denote all the exogenous variables, and  
let  $\cD \in \sP(\cX \times \cZ)$ denote the joint distribution of $(X, Z)$. We let $\EE_{\cD}[\cdot]$  denote the expectation taken with respect to 
the joint distribution of $(X, Z)$ and $\EE_{X|Z}[\cdot]$ denote the conditional expectation using the conditional distribution of $X $ given $Z$. 
Let $W \in \mathcal{W} \subseteq \mathcal{X} \times \mathcal{Z}$ be a subset of variables that may contain both the endogenous and exogenous variables, and let $L^2 (\cW)$ denote a Hilbert space of measurable functions of $W$ with finite second moment. Let $\cF :=\{f : \mathcal{W} \rightarrow \mathbb{R} \}\subset L^2 (\cW)$ denote a class of functions defined on $\cW$. 
In a \emph{functional conditional moment equation} problem, we 
aim to find a function $f_0\in \cF$ that solves the following 
functional equation involving the conditional distribution of $X$ given $Z$ over $\cF$: 
\begin{align}\label{eq:cmr}
    \EE_{X|Z}\bigl[\Phi(X, Z; f_0 )\Biggiven Z = z \Bigr] = 0, \qquad \forall z \in \cZ, 
\end{align}
where $\Phi\colon \cX \times \cZ \times \cF \rightarrow \RR$ is a known functional.\vspace{5pt}

\noindent For any function $ f\in \cF $ and any   \(z\in \cZ\), we define a functional $\bar \delta \colon \cZ \times \cF \rightarrow \RR  $ as 
\begin{align}
	\label{eq:cmrerr}
    \bar \delta (z; f) := \EE_{X|Z} \bigl[\Phi(X, Z; f) \biggiven Z = z \bigr], \qquad \forall f\in \cF , z \in \cZ .
\end{align}
In other words, 
the conditional moment equation problem in  \eqref{eq:cmr}
boils down to finding a function $f_0 \in \cF $ such that 
$\bar \delta (\cdot ; f_0) $ is a zero function on $\cZ$. Therefore an equivalent way to solve $f_0\in \cF$ in \eqref{eq:cmr} is by solving 
$
\inf_{f\in\cF} \EE[(\bar \delta (Z; f))^2]
$
\citep{ai2003efficient, chen2012estimation}.
To control the complexity of the function class $\cF$, \cite{ai2003efficient} propose  to use flexible sieve spaces $\cF_{k(n)}$ that becomes dense in $\cF$ as the sieve dimension $k(n)$ grows to infinity with data sample size $n$, and proposed the so-called sieve minimum distance criterion
$
 \min_{f\in \cF_k(n) } \EE \bigl[\bar\delta(Z; f)^2\bigr]/2.
$
In particular, \cite{ai2003efficient} allow for two-layer NNs, splines, wavelets, Fourier series, and all kinds of polynomial sieves $\cF_{k(n)}$ to approximate functions in $\cF \subseteq L^2(\cW)$. Alternatively \cite{chen2012estimation} propose  the following penalized (or regularized) minimum distance criterion:
\begin{align}\label{eq:mse}
    \min_{f\in \cF } \EE \bigl[\bar\delta(Z; f)^2\bigr]/2 + \lambda \cdot \cR(f),
\end{align}
where $\lambda \geq 0$ is a regularization parameter, $\cR(f)$ is a regularizer on function $f\in \cF$. They allow that $\cR(f)$ to be any convex or lower-semicompact regularizer. In the minimum distance approach, for any fixed $f$, the authors first estimate $\bar \delta (z; f)$ by the following least squares criterion:
\begin{align*}
\argmin_{\delta\in L^2(\cZ)}\EE \Big[1/2 \cdot \big(\Phi(X, Z; f)-\delta(Z)\big)^2\Big]
=\argmax_{\delta\in L^2(\cZ)}\EE \Big[\Phi(X, Z; f)\delta(Z) - 1/2 \cdot \delta(Z)^2 \Big]  
\end{align*}
Furthermore,   we assume that the functional $\Phi$ is affine in $f$, which captures several important applications in machine learning and causal inference listed in Section \ref{sec:examples}. 
Specifically, we define 
$\tilde \Phi (x,z,f) = \Phi(x,z, f) - \Phi(x,z, 0)$, where $0$ stands for the zero function on $\cW$. 
Then for any two functions 
$f_1, f_2 \in \mathcal{F}$ and any  $a, b \in \mathbb{R}$, we have 
\begin{align}\label{eq:fme-form}
& \tilde    \Phi (x, z; af_1 + bf_2) = a \cdot \tilde \Phi (x, z; f_1) + b \cdot \tilde \Phi (x, z; f_2) ,\qquad  \forall (x,z) \in \cX \times \cZ. 
\end{align}

{\noindent \bf  Solving \eqref{eq:cmr} with Overparameterized Neural Networks.}
In the sequel, we aim to solve the problem in \eqref{eq:cmr} based on i.i.d. data points sampled from $\cD$,   with $\cF$ being a class of overparameterized neural networks.
In this case, it is possible that \eqref{eq:cmr} does not have a solution within~$\cF$.
Furthermore, for the choice of regularizer, we consider the following specific form of $\cR(f)$: 
\begin{align}\label{eq:reg-form}
\cR(f) = \EE_{\cD}[\Psi(X,Z;f)]
\end{align}
where for any given $(x,z) \in \cX \times \cZ $, $\Psi(x, z;f): \cF \rightarrow {\RR}_{+}$ is a convex functional of $f$ that maps each function $f$ to a scalar. Moreover, $\Psi$ satisfies
\begin{align} 
 \Psi(x,z; 0) & = 0, \qquad \Psi(x,z;f) \geq 0 , \qquad  \forall f \in \cF,  \label{eq:reg-property1} \\
\frac{\delta \Psi(x, z; af_{1} + bf_{2})}{\delta f} & = a \cdot  \frac{\delta \Psi(x, z; f_{1})}{\delta f} + b \cdot  \frac{\delta \Psi(x, z; f_{2})}{\delta f} , \qquad  \forall f_1, f_2 \in \cF, \; a,b \in \RR . \label{eq:reg-property2}
\end{align}

\noindent Equation \eqref{eq:reg-property1} requires that $\Psi(X, Z;f)$ is a non-negative functional of $f$ that is equal to $0$ if and only $f = 0$. 
Equation \eqref{eq:reg-property2} requires that the functional derivative of $\Psi(X, Z;f)$ with respect to $f$, is linear in $f$. One example of $\Psi$ is the  $L_{2}$-regularizer of the following type, $\Psi(x,z;f) = f(w)^2$. Here $w \in \cW$ is a subset of variables that contain values from both the endogenous variables $x$ and exogenous variables $z$.

\vspace{5pt}
{\noindent \bf Minimax Estimation.} 
To solve the optimization problem in \eqref{eq:mse}, we first transform it into a \emph{unconditional} moment formulation by introducing a dual function.
By Fenchel duality, we can rewrite the objective function $J(f)$ as follows,
\begin{align}
\label{eq:obj cmr1}
J(f) & = \EE_{\cD} \Bigl[ 1/2 \cdot \bar \delta(z; f)^2 + \lambda \Psi(X,Z;f) \Bigr]\\
&  = \EE_{\cD} \Bigl[ \max_{g:\cZ\rightarrow \RR} \left( g(z) \cdot \EE \bigl[ \Phi(X,z;f) \biggiven z \bigr] - 1/2\cdot g(z)^2\right) + \lambda \Psi(X,Z;f) \Bigr] \nonumber\\
&  =\max_{g:\cZ\rightarrow \RR} \EE_{\cD} \Bigl[ g(Z)\cdot \Phi(X, Z; f) - 1/2\cdot g(Z)^2 + \lambda \Psi(X, Z;f) \Bigr]. \nonumber
\end{align}
The formulation in \eqref{eq:obj cmr1} leads to  the following minimax optimization problem:
\begin{align}
	\label{eq:minimax obj}
	\min_{f}\max_g\cL(f, g) = \EE_{\cD} \Bigl[ g(Z)\cdot \Phi(X, Z; f) - 1/2\cdot g(Z)^2 + \lambda \Psi(X, Z;f) \Bigr].
\end{align}
We note that $\cL$ is a convex-concave functional with respect to function $f$ and $g$. We denote by $(f^{*}, g^{*})$ the unique saddle point of \eqref{eq:minimax obj}. Here the uniqueness of $f^*$ comes from the convexity of regularization $\Phi(X,Z;f)$, and $g^*(z) = \mathbb{E}[\Phi(X,Z;f^*) | Z = z]$ implies the uniqueness of $g^*$. Without the regularization, i.e., $\lambda = 0$, the saddle point of \eqref{eq:minimax obj} is $f^* = f_{0}$ and $g^* = 0$.

\subsection{Examples of Functional Conditional Moment Equation} \label{sec:examples}
In this section, we discuss several important applications of the functional conditional moment equation, which serve as running examples of this paper. 

\vspace{5pt}
\noindent \textbf{Policy Evaluation. } We consider a Markov decision process given by $ (\cS, \cA, \cP, r, \gamma)$, where $ \cS \subseteq \RR^{d} $ is the state space, $ \cA  $ is the action space, $ \cP:\cS\times\cA\rightarrow \sP(\cS) $ is the transition kernel, $r : \cS\times\cA \rightarrow [0, 1] $ is the reward function, $ \gamma \in (0, 1) $ is the discount factor. Given a policy $ \pi:\cS\rightarrow \sP(\cA) $, an agent interacts with the environment in the following manner. At a state $ s_t $, the agent takes an action $ a_t\sim \pi(\cdot \given s_t) $ and receives a reward $ r_t = r(s_t, a_t) $. Then, the agent transits to the next state $ s_{t+1} \sim \cP(\cdot \given s_t, a_t) $. We denote the transition kernel induced by policy $ \pi $ by $ \cP^\pi(s'\given s) = \int_\cA \cP(s'\given s, a) \pi(a\given s) \rd a $ for any $s, s' \in \cS$.
In policy evaluation, we aim to estimate the value function $ V^\pi:\cS\rightarrow \RR $ defined as follows,
\begin{align*} % \label{eq:vpi}
	V^\pi(s) = \EE_\pi\Bigl[ \sum_{i=0}^{\infty}\gamma^i r(s_i, a_i) \Biggiven s_0 = s \Bigr],
\end{align*}
where the expectation $ \EE_\pi $ is taken with respect to $ a_t\sim \pi(\cdot \given s_t) $ and $ s_{t+1} \sim \cP(\cdot \given s_t, a_t) $ for~$ t \ge 0 $. By the Bellman equation \citep{sutton2018reinforcement}, it holds for any $ s \in \cS $ that
\begin{align}
\label{eq:bell}
V^\pi(s) - \cT^\pi V^\pi(s) = 0, \quad \cT^\pi f(s) = \EE_{a\sim \pi(\cdot \given s)} \bigl[r(s, a)\bigr] + \gamma \EE_{s'\sim  \cP^\pi(\cdot \given s) } \bigl[f(s')\bigr].
\end{align}
Corresponding to the Bellman equation in \eqref{eq:bell}, let $\cD$ denotes the joint distribution of the state-action tuple $(s,a,s')$ under policy $\pi$, the value function $V^{\pi}$ satisfies the following functional conditional moment equation, 
\begin{align}
\label{eq:fme-policy}
\EE_{s'|s} \Bigl[r(s,a) - V^{\pi}(s) + \gamma \cdot V^{\pi}(s') \Biggiven s \Bigr] = 0.
\end{align}
We notice that \eqref{eq:fme-policy} is a special case of the functional conditional moment equation in \eqref{eq:cmr} by setting the exogenous variable $Z$ to be the current state $s$, the endogenous variable $X$ to be the next state $s'$ and the function to be estimated $f: \cS \rightarrow \mathbb{R}$ to be defined on the state space $\cS$. In this case, the functional is $\Phi(X, Z; f) = r + \gamma \cdot f(X) - f(Z)$, where $r$ is the reward function. We remark that the reason function $f$ can be evaluated simultaneously on $X$ and $Z$ is that both $X$ and $Z$ are variables defined on $\cS$. Following the same derivation of \eqref{eq:obj cmr1}, policy evaluation can be formulated as the following minimax optimization problem,
\begin{align*}
% \label{eq:minimax-policy}
\min_{f}\max_g\cL(f, g) = \EE_{\cD} \Bigl[ g(Z) \cdot \bigl(r + \gamma \cdot f(X) - f(Z)\bigr) - 1/2\cdot g(Z)^2 + \lambda \Psi(X, Z;f) \Bigr].
\end{align*}

\vspace{5pt}
\noindent \textbf{Nonparametric Instrumental Variables Regression. } The nonparametric instrumental variables model is common and useful in statistics and economics. The model can be described simply by a line of equation
\begin{align*} % \label{eq:cme-model}
Y = f_{0}(X) + \varepsilon, \quad \EE \bigl[ \varepsilon \biggiven Z \bigr] = 0.
\end{align*}
where $Y$ in an observed outcome, $X$ is the endogenous variable, $Z$ is the exogenous variable, $f_{0}$ is the true model that characterize the relationship between $Y$ and $X$ and is also the function we want to estimate. In this model, $\varepsilon$ is a noise possibly correlated with the endogenous $X$ but uncorrelated with the exogenous $Z$. It's straightforward to see that NPIV model fits into the framework of the functional conditional moment equation by plugging the model equation into the equation about $\varepsilon$,
\begin{align}
\label{eq:fme-cme}
    \EE_{\cD} \Bigl[ Y - f_{0}(X) \biggiven Z \Bigr] = 0.
\end{align}
We notice that \eqref{eq:fme-cme} is a special case of functional conditional moment equation in \eqref{eq:fme-form} by identifying $X$, $Z$ with the endogenous and exogenous variable respectively and setting the functional as $\Phi(X, Z;f) = Y - f(X)$. Following the same derivation of \eqref{eq:obj cmr1}, the problem of NPIV is equivalent to the following minimax optimization problem,
\begin{align*}
% \label{eq:minimax-cme}
\min_{f}\max_g\cL(f, g) = \EE_{\cD} \Bigl[ g(Z)\cdot \bigl(Y - f(X) \bigr) - 1/2\cdot g(Z)^2 + \lambda \Psi(X, Z;f) \Bigr].
\end{align*}

\vspace{5pt}
\noindent \textbf{Asset Pricing. } Asset pricing refers to the process of determining the fair value of financial assets. This field is fundamental in finance and underpins much of the work in investment, portfolio management, and risk assessment. 
Semiparametric Consumption Captial Asset Pricing Model (CCAPM) is a foundational model in asset pricing that describes the relationship between systematic risk and expected asset returns, which also incorporates the influence of the consumption preference of investors over time. Moreover, CCAPM can be characterized through a functional conditional moment equation \citep{chen2014local, chen2009land}. To describe the model, let $C_t$ denote the consumption level at time $t$, $c_t \equiv C_{t} / C_{t-1}$ the consumption growth. The marginal utility of consumption at time $t$ is given by $\text{MU}_{t} = C_{t}^{-\gamma_0} f_{0}(c_t)$,  where $\gamma_0 > 0$ is the discount factor, $f_0: \cC \to \mathbb{R}$ is the nonparametric structural demand function, which is an unknown positive function of our interest and is defined on $\cC$, the space of consumption growth. The unknown function $f_0$ can be understood as a taste shifter that describes how
the marginal utility of consumption changes with the state of the economy in terms of consumption growth. \vspace{5pt} 

\noindent Now, consider the growth-return tuple $(c_t, \tilde r_{t+1}, c_{t+1})$ for $t \in \mathbb{N}^{+}$ with joint distribution $\cD$, where $c_t$ is the consumption growth at the current time $t$, and $c_{t+1}$ is the consumption growth at the next time $t+1$. $\tilde r_{t+1}$ is a modified return observed in this period, which is a known function of the actual return $r_{t+1}$ and the consumption growth $c_{t+1}$ at time $t + 1$. We consider the scenario where the time series of consumption growth $\{c_t\}_{t \geq 0}$ follows a time-homogenous Markov chain with a smooth transition kernel. That being said, both conditional transition probabilities $c_{t+1} | c_{t}$ and $c_t | c_{t+1}$ admit a smooth density function. The CCAPM model captures the behavior of $f_0$ through the following equation:

\begin{align}
\label{eq:fme-ccapm}
    \EE_{c_{t+1} | c_t}\big[ \tilde r_{t+1} \cdot f_0(c_{t+1}) - f_0(c_t) \biggiven c_t \big] = 0,
\end{align}
where the modified return can be further expressed as $\tilde r_{t+1} = \delta_0 \cdot r_{t+1} \cdot c_{t+1}^{-\gamma_0}$, $\delta_0 \in (0,1]$ is the rate of time preference.
We focus on a setting where $\cC \subseteq \mathbb{R}$ is a compact set, and the modified return $\tilde r_{t+1}$ is bounded for all $t \geq 0$. We notice that \eqref{eq:fme-ccapm} is a special case of the functional conditional moment equation in \eqref{eq:fme-form}. We can identify the exogenous variable $Z$ with $c_t$, the consumption growth at the current time $t$, and the endogenous variable $X$ with $c_{t+1}$, the consumption growth at the next time $t+1$. In this scenario, we identify the space $\cW$ with $\cC$, the space of consumption growth and the function to be estimated $f: \cC \to \mathbb{R}$ is defined on $\cC$. The functional is $\Phi(X,Z;f) = \tilde r_{t+1} \cdot f(X) - f(Z)$, where $\tilde r_{t+1}$ again denotes the modified return. Similar to the scenario of policy evaluation, the reason function $f$ can be evaluated simultaneously on $X$ and $Z$ is that both $X$ and $Z$ are variables defined on $\cC$. Following the same derivation of \eqref{eq:obj cmr1}, the problem of asset pricing through CCAPM is equivalent to the following minimax optimization problem,
\begin{align*}
% \label{eq:minimax-asset}
\min_{f} \max_{g} \cL(f,g) = \EE_{\cD} \Big[g(Z) \cdot (\tilde r_{t+1} \cdot f(X) - f(Z))- 1/2 \cdot g(Z)^2 + \lambda \Psi(X,Z;f) \Big]
\end{align*}

\vspace{5pt}
\noindent \textbf{Adversarial Riesz representer Estimation. } Many problems in statistics, causal inference, and finance involve  the task of learning a continuous linear functional in the following form,
\begin{align}
\label{eq:def-funct}
\cV(g) = \EE \bigr[m(V; g)\bigl].
\end{align}
where function $g \in \cG: \cX \rightarrow \RR$, $\cF$ is defined on a function space $\cG$, and $V$ is a random vector of which we have access to observations and represents the source of randomness in the functional. Moreover, suppose such continuous linear functional $\cF(\cdot)$ is also mean-square continuous with respect to $L^{2}$ norm. In that case, it can be written in a more benign and useful manner, which is also often the case. Formally speaking, for such linear functional $\cV$, there exists function $f_{0}$ such that for any $g \in \cG$,
\begin{align}
\label{eq:def-riesz}
\cV(g) = \EE \bigl[f_{0}(X) g(X)\bigr].
\end{align}
The function $f_{0}$ here is called the Riesz representer of the linear functional $\cV$, and the equation \eqref{eq:def-riesz} is known as the Riesz representation theorem. Information about the Riesz representation of such linear functional is crucial to numerous applications and learning tasks. Therefore, we aim to estimate $f_{0}$ by exploiting the relationship characterized by the equation. We have the following trivial observation that the true Riesz representer $f_{0}$ can be recovered by solving the following equation,
\begin{align}
\label{eq:fme-riesz}
\EE \Bigl[f_{0}(X) - f(X) \biggiven X \Bigr] = 0.
\end{align}
Of course, $f = f_{0}$ will solve the equation above, and therefore the true Riesz representer is achieved. We remark that this is indeed a special case since the expectation taken in \eqref{eq:fme-riesz} is unconditioned. In the equation, we only involve the endogenous variable $X$, which also indicates that the exogenous variable $Z$ coincides with $X$. While special, the problem still fits in the framework discussed here. By setting $\Phi(X, Z;f) = f(X) - f_{0}(X)$, we recovered the intractable formulation of Riesz representer estimation.

\vspace{5pt}
\noindent However, unlike the previous examples where have access to observations of each term in the equation, here we have no direct access to values of $f_{0}$, making the problem seemingly intractable. Fortunately, the alternative formulation of the original problem as a minimax optimization problem solves this difficulty. When written in the minimax formulation, we will again see the linear functional $\cV$ show up in the equation in the form of \eqref{eq:def-funct}, which can be approximated using empirical values calculated from accessible observations of the random vector $V$. Following the same derivation of \eqref{eq:obj cmr1} and the definition of Riesz representer in \eqref{eq:def-riesz}, the problem of adversarial Riesz representer estimation is equivalent to the following minimax optimization problem,
\begin{align}
\label{eq:minimax-riesz}
\min_{f} \max_{g} \cL(f,g) = \EE_{\cD} \Bigl[m(V; g) - f(X) \cdot g(X) - 1/2\cdot g(X)^2 + \lambda \Psi(X, X;f) \Bigr].
\end{align}
Again, we stress that in \eqref{eq:minimax-riesz}, the absence of $Z$ is due to the fact both the endogenous and exogenous variables are described by $X$ and the objective is computationally tractable since we have access to both observations of $X$ and $V$. 

\subsection{Mean-Field Neural Network and Wasserstein Space}\label{sec:w-space}
In the sequel, we will consider functions in the neural network function class. Consider a neural function defined on a given state space $\Omega$, $\sigma: \Omega \times \RR^{D} \rightarrow \RR$ that takes an input $x \in \Omega$ and parameter $\theta \in \RR^D$ and outputs a value in $\RR$. For $\boldsymbol{\theta} = (\theta_1, \dots, \theta_{N})$ where $\theta_i \in \RR^{D}$, we can define an overparameterized two-layered neural network function $h$ using neuron function $\sigma$,
\begin{align*}
    h(x, \boldsymbol{\theta}) = \frac{1}{N} \sum_{i = 1}^{N} \sigma(x; \theta_i), \quad \forall x \in \Omega.
\end{align*}
For such a form, we can further consider the infinite width limit when $N \rightarrow \infty$. When taking such a limit, the neural network function $h$ becomes a mean-field neural network and can be parameterized with probability measure over the parameter space, $\mu \in \sP(\RR^{D})$.
\begin{align*}
    h(x;\mu) = \int_{\RR^{D}} \sigma(x; \theta) \rd \mu(\theta), \quad \forall x \in \Omega.
\end{align*}
When considering such a limit, the optimization problem over the neural network function class is turned from a finite-dimensional problem over the parameter space into an infinite-dimensional problem over the space of probability measures. Therefore, we will need to track the convergence of probability measures over the Wasserstein space when analyzing the convergence of algorithms. \vspace{5pt}

\noindent We now introduce the background knowledge of the Wasserstein space for the reader's information. Let  \(\sP_p(\RR^D)\) be the space of all the probability measures over the $D$-dimensional Euclidean space \(\RR^D\) with finite \(p\)-th order moments. The Wasserstein-\(p\) distance between two probability measures \(\mu, \nu \in \sP_p(\RR^D)\) is defined as follows,
\begin{align}
	\label{eq:w-p}
    \cW_p(\mu, \nu) = \inf \biggl\{ \Bigl(\int \norm{x - y}^p \rd \gamma(x, y)\Bigr)^{1/p} \Biggiven \gamma \in \sP_p(\RR^D \times\RR^D), x_\sharp \gamma = \mu, y_\sharp \gamma = \nu \biggr\},
\end{align}
where the infimum is taken over all the coupling of \(\mu\) and \(\nu\). Here we denote by \(x_\sharp \gamma\) and \(y_\sharp \gamma\) the marginal distributions of \(\gamma\) with respect to \(x\) and \(y\), respectively. We call \(\cM_{p} = (\sP_p ( \RR^D), \cW_p)\) the Wasserstein-\(p\) space. For any $1 \le p\le q$, due to the relation that \(\EE[|X|^p]^{1/p} \le \EE[|X|^q]^{1/q}\), we have that \(W_p(\mu, \nu)\le W_q(\mu, \nu)\) for two measures \(\mu, \nu\). In this paper, we focus on the cases when \(p = 1, 2\). Without further clarification, we refer to the distance with $p = 2$ as the Wasserstein distance in the sequel. \vspace{5pt}

\noindent The Wasserstein-2 space \(\cM_2 = (\sP_2(\RR^D), \cW_2)\) can be viewed as an infinite-dimensional Riemannian manifold \citep{villani2008optimal}.  Formally,  the tangent space  at  point $\rho \in \sP_2(\RR^D)$ is defined as 
\begin{align*}
    \mathrm{Tan}_\rho\bigl(\sP_2(\RR^D)\bigr) = \Bigl\{v\in L^2(\rho)\Biggiven \int \inp{v}{u}d\rho = 0, \forall u \in L^2(\rho) \text{ s.t. } \dive(u\rho) = 0\Bigr\}. 
\end{align*}
Then, for any absolutely continuous curve \(\rho: [0, 1] \rightarrow \sP_2(\RR^D)\) on the Wasserstein-2 space, there exists a  family of vector fields  $v_t \in \mathrm{Tan}_{\rho_t}(\sP_2(\RR^D)) $  such that the continuity equation 
\begin{align}
    \partial_t \rho_t + \dive (v_t\rho_t) = 0
    \label{eq:continuity}
\end{align}
holds in the sense of distributions. 
For any two absolutely continuous curves $\rho, \tilde\rho: [0, 1] \rightarrow \sP_2(\RR^D)$, we define the inner product between $\partial_t \rho_t, \partial_t\tilde\rho_t$ for any $t\in [0, 1]$ as follows,
\begin{align}
	\label{eq:inner}
    \inp[]{\partial_t\rho_t}{\partial_t\tilde \rho_t}_{\rho_t} = \int \inp[]{v_t}{\tilde v_t} \rd \rho_t,
\end{align}
where \(\inp[]{v_{t}}{\tilde v_t}\) is the inner product over \(\RR^D\), $(\rho_t, v_t)$ and $(\tilde \rho_t, \tilde v_t)$ satisfy the continuity equation in \eqref{eq:continuity}.
Note that \eqref{eq:inner} yields a Riemannian metric over \(\cM_2\). Furthermore, the Riemannian metric induces a norm \(\norm[]{\partial_t \rho_t}_{\rho_t} = \inp[]{\partial\rho_t}{\partial_t\rho_t}_{\rho_t}^{1/2}\).

%!TEX root = main.tex

\section{Algorithms}\label{sec:w-grad}
In this section, we introduce the stochastic gradient descent-ascent algorithm (SGDA) and its mean-field limit, which is characterized by the continuity equation.  

\vspace{5pt}
\noindent \textbf{Stochastic Gradient Descent-Ascent Algorithm. } We solve the minimax optimization problem in~\eqref{eq:minimax obj} via SGDA. Recall that in the minimax objective, we have two functions simultaneously involved, where the primal function $f$ represents the true model of interest and the dual function $g$ represents an adversarial player. Specifically, we parameterize both $f$ and $g$ with neural networks with width $N$ and parameters $\btheta = (\theta^{1}, \theta^{2}, \dots, \theta^{N}) \in \RR^{D \times N}$ and $\bow = (\ow^{1}, \ow^{2}, \dots, \ow^{N}) \in \RR^{D \times N}$
\begin{align}
	\label{eq:nn-fin}
	f(\cdot; \btheta) = \frac{\alpha}{N}\sum_{i=1}^{N} \phi(\cdot; \theta^i), \quad g(\cdot; \bow) = \frac{\alpha}{N}\sum_{i=1}^{N} \psi(\cdot; \ow^i).
\end{align}
where we use bold symbols $\btheta$ and $\bow$ to denote the whole parameter used by each neural net and unbold symbols $\theta$ and $\ow$ to denote the parameter used by each neuron. Here, $\phi(\cdot; \theta): \cW \times \RR^{D} \rightarrow \RR$, $\psi(\cdot; \ow): \cZ \times \RR^{D} \rightarrow \RR$ are the functions for neurons. In particular, we can recover the general setting of two-layer neural networks parameterization for $f$ and $g$ when we choose $\phi, \psi$ to be the following specific form,
\begin{align*}
    \phi(w;\beta, W) = \beta \cdot \sigma_f(w; W), \quad \psi(z; \beta, W) = \beta \cdot \sigma_g(z; W),
\end{align*}
where $\sigma_f: \cW \times \RR^{D} \rightarrow \RR$, $\sigma_g: \cZ \times \RR^{D} \rightarrow \RR$ are activation functions with input $w$ and $z$ respectively and parameters $W$. We note that it's not necessary to choose the same width $N$ for $f$ and $g$, and activation functions $\sigma_f, \sigma_g$ need not have the same parameter dimension $D$. Here we use the same width $N$ and parameter dimension $D$ to keep notations simple as these won't affect the validity of the results presented in this paper. \vspace{5pt}

\noindent Besides, we have also introduced a scaling factor $\alpha > 0$ in \eqref{eq:nn-fin}. Setting the scaling parameter $\alpha = \sqrt{N}$ in \eqref{eq:nn-fin} recovers the neural tangent kernel regime \citep{jacot2018neural}. Setting the parameter $\alpha = 1$ recovers the mean-field regime \citep{mei2018mean,mei2019mean}. In a discrete-time finite-width (DF) scenario, at the $k$th iteration, the primal function $f$ and adversarial player $g$ are updated as follows,
\begin{align}
\label{eq:gd-fin}
	\text{DF-GD: }\quad & \btheta_{k+1} = \btheta_k - \eta \cdot g(z_k; \ow_k) \cdot \nabla_{\btheta} \Phi(x_k, z_k; f(\cdot; \btheta_k)) - \eta \lambda \cdot \nabla_{\btheta} \Psi(x_k, z_k; f(\cdot; \btheta_k)), \nonumber\\
    \text{DF-GA: } \quad & \bow_{k+1} = \bow_k + \eta \cdot \Phi(x_k, z_k; f(\cdot; \btheta_k)) \cdot \nabla_{\bow} g(z_k; \bow_k)) - \eta \cdot g(z_k; \bow_k) \cdot \nabla_{\bow} g(z_k; \bow_k),
\end{align}
where $\btheta_k, \bow_k$ denotes the state of the parameters at iteration $k$, $\eta > 0$ is the step-size, and the data samples $ \{(x_k, z_k)\}_{k=0}^\infty $ are collected by independently sampling from the data distribution $\cD$. When $f,g$ are two-layered neural networks with width $N$, we can plug in the form for $f,g$ as is described in \eqref{eq:nn-fin}. The update for the parameter of $i$-th neuron at $k$-th iteration can be further specified to the following,
\begin{align}
\label{eq:gdnn-fin}
\theta_{k+1}^{i} &= \theta_k^{i} - \eta \alpha \epsilon \cdot g(z_k; \bow_k) \cdot \nabla_{\theta} \Phi(x_k, z_k; \phi(\cdot, \theta_{k}^i)) - \eta \lambda \epsilon \cdot \frac{\delta \Psi(x_k, z_k; f(\cdot, \btheta_k))}{\delta f} \cdot \nabla_\theta \phi(x_k; \theta_k^{i}), \nonumber\\
\ow_{k+1}^{i} &= \ow_k^{i} + \eta  \alpha \epsilon \cdot  \Phi(x_k, z_k; f(\cdot, \btheta_k)) \cdot \nabla_\ow \psi(z_k; \ow_k^{i}) - \eta \alpha \epsilon \cdot g(z_k; \bow_k) \cdot \nabla_\ow \psi(z_k; \ow_k^{i}),
\end{align}
where $\btheta_k = (\theta^{1}_k, \theta^{2}_k, \dots, \theta^{N}_k)$ and $\bow_k = (\ow^{1}_k, \ow^{2}_k, \dots, \ow^{N}_k)$, $\delta \Psi/\delta f$ denotes the variation of $\Psi$ with respect to $f$. Here, $\alpha$ is the neural network scaling parameter and $\epsilon = 1/N$ is the stepsize scale. Both $\alpha$ and $\epsilon$ show up in \eqref{eq:gdnn-fin} due to the finite width parameterization of two-layered neural networks described in \eqref{eq:nn-fin}. \vspace{5pt}

\noindent For a given space $\cS$, let $\cH$ define a set of functions defined on $\cS \rightarrow \RR$. For a functional defined over the function class $\cH$,  $F: \cH \rightarrow \RR$, its variation at $f \in \cH$ is a function $\frac{\delta F}{\delta f}: \cS \rightarrow \RR$, such that for any test function $h \in \cH$,
\begin{align}
\label{eq:func-deriv}
\Big[ \frac{\rd}{ \rd \varepsilon} F(f + \varepsilon h)\Big]_{\varepsilon = 0} = \int_{\cS} \frac{\delta F}{\delta f}(s) \cdot h(s) ~\rd s.
\end{align}
We initialize the parameters with $\theta_0^i \sim \mu_0$ and $w_0^i \sim \nu_0$, with $\mu_{0}, \nu_{0} = \cN(0, I_{D})$ be standard Gaussian distribution in $\RR^{D}$. In addition, to keep track of the evolution of the parameter distribution, we  denote the empirical  distribution of $\btheta$ and $\bow$ at the $k$th iteration by,
\begin{align*}
\hat \mu_k(\theta) = \frac{1}{N} \sum_{i=1}^{N} \delta_{\theta^i_k}(\theta), \quad \hat \nu_k(\ow) = \frac{1}{N} \sum_{i=1}^{N} \delta_{\bow^{i}_{k}}(\ow),
\end{align*}
where $\delta$ is the Dirac mass function.

\vspace{4pt}

\noindent{\bf Mean-Field (MF) Limit.} To analyze the convergence of the Stochastic Gradient Descent-Ascent Algorithm for solving functional conditional moment equations with neural networks, we employ an analysis that studies the mean-field limit regime \citep{mei2018mean,mei2019mean} of the discrete-time dynamics described in \eqref{eq:gd-fin}. Here, by the mean-field limit, we are referring to an infinite-width limit, i.e., when $N \rightarrow \infty$ for the neural network width and a continuous time, i.e., $t = k\epsilon$ where the step scale $\epsilon \rightarrow 0$ in \eqref{eq:gdnn-fin}. 
In what follows, we introduce the mean-field limit of the SGDA dynamics, which refers to the infinite-width and continuous limit of \eqref{eq:gdnn-fin}. For $\btheta = \{\theta^i\}_{i = 1}^{N}$ and $\bow = \{\ow^i\}_{i=1}^{N}$ independently sampled respectively from $\mu, \nu \in \sP(\RR^{D})$, we can write the infinite width limit of neural networks used in \eqref{eq:nn-fin} as
\begin{align}
\label{eq:nn} 
f(\cdot; \mu) = \alpha \int \phi(\cdot; \theta) \mu(\rd \theta), \quad g(\cdot; \nu) = \alpha \int \psi(\cdot; \ow) \nu(\rd \ow).
\end{align}
From now on, we denote by $\mu_{t}$ the distribution of $\theta_{t}^{i}$ and $\nu_{t}$ the distribution of $\ow_{t}^{i}$ for the infinite-width and continuous limit of the neural networks at time $t$. For notational simplicity, we overload the notation of the objective function in \eqref{eq:minimax obj} via $\cL(\mu, \nu) = \cL(f(\cdot; \mu), g(\cdot; \nu))$. This is to further emphasize the dependence of objective $\cL$ on $(\mu, \nu)$ when we parameterize the function pair $(f,g)$ using distributions on the parameter space. By Otto's calculus \citep{villani2008optimal}, the mean-field limit of the update direction takes the following form,
\begin{align}
\label{eq:dir}
v^f(\theta; \mu, \nu) &=-\nabla_\theta \frac{\delta \cL(\mu, \nu)}{\delta \mu}(\theta) \nonumber\\
&= \alpha\EE_{\cD} \Bigl[ -g(Z; \nu) \cdot \big\langle \frac{\delta \Phi(X,Z; f(\cdot; \mu))}{\delta f}, \nabla_{\theta} \phi(\cdot; \theta) \big\rangle_{L^{2}} - \lambda \cdot \Big\langle  
\frac{\delta \Psi(X,Z; f(\cdot; \mu))}{\delta f}, \nabla_{\theta} \phi(\cdot; \theta) \Big\rangle_{L^{2}} \biggr], \nonumber \\
v^g(\ow; \mu, \nu)  & = \nabla_\ow \frac{\delta \cL(\mu, \nu)}{\delta \nu}(\ow)\nonumber \\
& = \alpha \EE_{\cD} \Bigl[ \Phi(X,Z;f(\cdot, \mu)) \cdot \nabla_{\ow} \psi(Z; \ow)  - g(Z; \nu) \cdot \nabla_{\ow} \psi(Z;\ow) \Bigr].
\end{align}
Here $\langle \cdot, \cdot \rangle_{L_{2}}$ is the inner product on $L^{2}(\cX \times \cZ)$ with respect to the Lebesgue measure. Recall that $\cD$ is the data distribution of random variables $(X, Z) \in \cX \times \cZ$, we denote by $\rho_{\cX, \cZ}$ the density of $\cD$ with respect to the Lebesgue measure on $\cX \times \cZ$ and we use $\langle \cdot, \cdot \rangle_{\cD}$ to represent the inner product on $L^{2}(\cX \times \cZ)$ with respect to the probability distribution $\cD$. That is to say, for any two function $h_1, h_2 \in L^{2}(\cX \times \cZ)$, $\langle h_1, h_2\rangle_{\cD} = \int_{\cX \times \cZ} h_1 h_{2} ~\mathrm{d}\rho_{\cX, \cZ}$. \vspace{5pt}

\noindent In the sequel, we will also slightly abuse this notation and use $\langle \cdot, \cdot \rangle_{\cD}$ to denote the inner product on sub-spaces of $L^{2}(\cX \times \cZ)$, with the measure being the marginals of $\cD$ on these sub-spaces. In \eqref{eq:dir}, 
$\delta \Phi / \delta f$ and $\delta \Psi / \delta f$ is the variation of $\Phi$ and $\Psi$ over $f$ under $\langle \cdot, \cdot \rangle_{L^2}$, where the test functions are chosen over the function class $\cF$. In the same way, $\delta \cL / \delta\mu$ and $\delta \cL/ \delta\nu$ respectively denote the variation of the objective $\cL$ with respect to distributions $\mu$ and $\nu$ under $\langle \cdot, \cdot \rangle_{L^2}$, following definition in \eqref{eq:func-deriv} with the test function chosen over $\sP(\cX \times \cZ)$. We also remark that we can also define the variation under $\langle \cdot, \cdot \rangle_{\cD}$, which will only differ from the variation under $\langle \cdot, \cdot \rangle_{L^{2}}$ by a constant function factor that corresponds to the density of the marginals of $\cD$. Then, the mean-field limit of the SGDA update in \eqref{eq:gd-fin} is characterized by the continuity equation, which is a system of PDEs given by,
\begin{align}
\label{eq:pde}
\partial_t \mu_t(\theta) = - \eta \cdot \dive_\theta \bigl(\mu_t(\theta) v^f(\theta; \mu_t, \nu_t)\bigr) , \;\partial_t \nu_t(\ow) &=- \eta \cdot \dive_\ow \bigl(\nu_t(\ow) v^g(\ow;\mu_t, \nu_t)\bigr), 
\end{align}
where $\dive_\theta$, $\dive_\ow$ denotes the divergence with respect to $\theta$, $\ow$ respectively. Note that the initialization $\mu_0$ and $\nu_0$ are the same as the initialization of the discrete-time dynamics in \eqref{eq:gdnn-fin}, i.e. $\mu_{0} = \cN(0, I_{D})$, $\nu_{0} = \cN(0, I_{D})$ are taken to be the distribution of standard Gaussian random variables in $\RR^{D}$.

%!TEX root = main.tex

\section{Main Results}\label{sec:main}

In this section, we introduce the main theoretical results of the stochastic gradient descent-ascent dynamics. 
We first present the assumptions in \S\ref{sec:assumptions}. 
Then in \S\ref{sec:discrete} we show that the SGDA dynamics converge to a mean-field limit when the network with $N$ goes to infinity and the stepsize scale $\epsilon$ goes to zero. Finally, in \S\ref{sec:conv} we prove that the mean-field limiting dynamics converge to a globally optimal solution of the primal objective $J$ under proper assumptions.  
Moreover, we will show that the mean-field dynamics learns a data-dependent representation that is $\mathcal{O}(\alpha^{-1})$ away from the initial representation.   
 
\subsection{Assumptions}\label{sec:assumptions}
We consider two types of assumptions in this work. 
The first type of assumption is about the function class in which we search for solutions to the minimax optimization problem. In this category, Assumption \ref{asp:reg} and Assumption \ref{asp:reliability} discuss the richness and regularity of the two-layered neural network function class. The second type of assumption is about the feasible class of problems to apply our framework. In this category, Assumption \ref{asp:compact} discusses several technical assumptions on the data space and the regularity/smoothness of the functionals.  \vspace{5pt}

\noindent We start with the discussion of the two-layered neural network function class. Consider the neuron function $\phi$ and $\psi$ with the following form,
\begin{align}
\label{eq:neuro-def}
\phi(w; \theta) = b(\beta) \cdot \sigma({\tilde \theta}^{\top}(w, 1)), \quad \psi(z; \ow) = b(\beta) \cdot \sigma(\tilde \ow^{\top}(z, 1)),
\end{align}
 where $\theta = (\beta, \tilde \theta) \in \RR \times \RR^{1 + \operatorname{dim}(\cW)}$, $\ow = (\beta, \tilde \ow) \in \RR \times \RR^{1 + \operatorname{dim}(\cZ)}$ contains the parameters in the output layer and the hidden layer, $b: \RR \rightarrow \RR$ is an odd re-scaling function and $\sigma: \RR \rightarrow \RR$ is the activation function. Note that such a form of activation function satisfies the condition of universal function approximation theorem (Theorem~3.1 in \cite{pinkus1999approximation}) if $\sigma$ is not a polynomial. For notational simplicity, we write $\sigma(w; \tilde \theta) = \sigma({\tilde \theta}^{\top}(w, 1))$.
The re-scaling function $b: \RR\rightarrow \RR$ is introduced to ensure that the value of the neural network is upper bounded. When $b(\RR) =
(-B_{0}, B_{0})$, the function class induced by the neural network in \eqref{eq:nn} is equivalent to the following class,
\begin{align}
    \label{eq:func-class}
    \cF = \Bigl\{ f:\cW \rightarrow \RR \Biggiven  f(w) = \int \beta' \cdot \sigma(w; \tilde \theta) \; \rd \mu(\beta', \tilde \theta), \mu \in \sP_2\bigl((-B_{0}, B_{0}) \times \RR^{d+ 1}\bigr)\Bigr\},
\end{align}
where $d = \operatorname{dim}(\cW)$. This captures a rich function class due to the universal function approximation theorem \citep{barron1993universal,pinkus1999approximation}. We remark that we introduce the re-scaling function $b(\beta)$ in \eqref{eq:neuro-def} to
avoid the study of the space of probability measures over $(-B_0, B_0) \times \RR^{d+1}$, which has a boundary and thus lacks regularity in the study of optimal transport. Moreover, note that a scaling hyperparameter $\alpha > 0$ is introduced in the definition of the mean-field neural nets in \eqref{eq:nn}. When $\alpha > 1$, this causes an effect of overparameterization. In brief, $\alpha$ controls the error between the $(f(\cdot; \mu_t), g(\cdot; \mu_t))$ and optimizer $(f^*, g^*)$ according to Theorem \ref{th:main}. Furthermore, the overparameterization scale $\alpha$ has an influence through Lemma \ref{lem:stat}, which shows that the Wasserstein distance between the Gaussian initialization $(\mu_0, \nu_0)$ and the optimal distribution $(\mu^*, \nu^*)$ is upper-bounded by $\cO(\alpha^{-1})$. 
Next, we impose the following regularity assumptions on the neural network functions $\phi$ and $\psi$.
\begin{assumption}[Regularity of Neural Networks]
\label{asp:reg}
We assume that there exist absolute constants $B_{0} >0$, $B_{1} > 0 $ and $ B_{2} > 0$ such that
\begin{align*}
    & |\phi(w; \theta)| \le B_{0}, \quad \norm[\big]{\nabla_\theta \phi(w; \theta)} \le B_{1} , \quad \norm[\big]{\nabla^2_{\theta\theta} \phi(w; \theta)}_{F} \le B_{2}, \qquad \forall w \in \cW, \; \theta \in \RR^{D}, \\
    & |\psi(z; \ow)| \le B_{0}, \quad \norm[\big]{\nabla_\ow \psi(z; \ow)} \le B_{1} , \quad \norm[\big]{\nabla^2_{\ow\ow} \phi(z; \ow)}_{F} \le B_{2}, \qquad \forall z \in \cZ, \; \ow \in \RR^{D},
\end{align*}
where $\nabla^{2}_{\theta\theta}, \nabla^{2}_{\ow\ow}$ denotes the hessian with respect to $\theta$ and $\ow$ respectively, $\|\cdot\|$ denotes the vector $2-$norm, and $\|\cdot\|_{F}$ denotes the matrix Frobenius norm. Moreover, we assume that the rescaling function $b: \RR\rightarrow \RR$ is odd and its range satisfies that $b(\RR) = (-B_{0}, B_{0})$.
\end{assumption}
\noindent Assumption~\ref{asp:reg} is satisfied by a broad class of neuron functions. For example, it is satisfied when we set the activation function $\sigma(x) = \operatorname{sigmoid}(x)$ and rescaling function $b(\beta) = \tanh(\beta)$.

\noindent We also make the following assumption regarding the realizability of the saddle point solution $(f^*, g^*)$ to \eqref{eq:minimax obj}.

\begin{assumption}[Realizability] \label{asp:reliability}
    We assume that the saddle point solution $(f^*, g^*)$ of \eqref{eq:minimax obj} belongs to the function class defined in \eqref{eq:func-class}, i.e., $ f^*, g^* \in \cF $.
\end{assumption}
\noindent In general, problem \eqref{eq:minimax obj} may not admit a saddle point within the given neural network function class. Therefore, Assumption \ref{asp:reliability} is introduced to guarantee that the discussion in this paper is meaningful. By universal function approximation theorem \citep{barron1993universal,pinkus1999approximation}, the function class defined in \eqref{eq:func-class} captures a rich class of functions. Therefore, such an assumption is quite general and does not restrict the influence of the applications of our results. \vspace{5pt}

\noindent We impose the following assumptions on the integrability of the functional $\Phi$ and $\Psi$ and their variations, as well as the compactness of the data space $\cX$ and $\cZ$.

\begin{assumption}[Data regularity and Functional Integrability] \label{asp:compact}
\item \textbf{(i)} For the data space $\cX$, $\cZ$, we assume that $\cX \times \cZ$ is compact, in the sense that there exists a positive constant $C_{1} > 0$ such that for any data tuple $(x, z) \in \cX \times \cZ$, it satisfies that $\|(x, z)\| \leq C_1$. Moreover, the data distribution $\cD$ admits a positive, smooth density $\rho_{\cD}$ with respect to the Lebesgue measure on $\cX \times \cZ$. 
\item \textbf{(ii)} For the functionals $\Phi(x,z;f)$ and $\Psi(x,z;f)$, there exists a positive constant $C_{2} > 0$ such that
\begin{align*}
\int_{\cW} \Big |\frac{\delta \Phi(x,z;f)}{\delta f}(w') \Big| \rd w' \leq C_2, \quad \int_{\cW} \Big| \frac{\delta \Psi(x,z;f)}{\delta f}(w') \Big| \rd w' \leq C_2, \quad \forall (x,z) \in \cX \times \cZ.
\end{align*}
\item \textbf{(iii)}
We assume that $\int_{\cW} \frac{\delta \Psi(x,z;f)}{\delta f}(w') \rd w'$  as a linear functional of $f$ is upper-bounded by constant times of values of $f$. That is to say, there exists $w \in \cW$ as a part of the data tuple $(x,z)$ and a positive constant $C_{\Psi} > 0 $ such that
\begin{align*}
    \int_{\cW} \Big| \frac{\delta \Psi(x,z;f)}{\delta f}(w')\Big| \rd w' \leq C_{\Psi} \cdot \big|f(w)\big|.
\end{align*}

\item \textbf{(iv)}
We assume that the variation of minimax objective $\cL(f,g)$ with respective to $f$ and $g$ are continuous functions defined on $\cW$ and $\cZ$. That is to say,
\begin{align*}
\frac{\delta \cL(f,g)}{\delta f}\in \sC(\cW), \quad \frac{\delta \cL(f,g)}{\delta g} \in \sC(\cZ).
\end{align*}

\end{assumption}

\noindent Item (i) of Assumption \ref{asp:compact} restricts our scenarios to data spaces with bounded values and smooth densities for technical reasons. Item (ii) and (iii) of Assumption \ref{asp:compact} is an integrability condition that we additionally require to avoid discussion of improper functionals that potentially have singularities with exploding values. Item (iv) is a smoothness condition that requires the variation of the minimax objective averaged over data to be continuous on respective space. We also remark that a sufficient condition for item (iv) to hold is the variation of $\Phi$ and $\Psi$ with respect to $f$ averaged under the marginal of $\cD$ on $\cW$ is continuous. We will also use this condition to verify item (iv) in practice. These are general and reasonable assumptions widely satisfied by various applications in machine learning, causal inference, and statistics.

\subsection{Convergence of SGDA dynamics to the Mean-Field Limit}\label{sec:discrete}
In the following proposition, we show that the empirical distribution of the parameters $\hat \mu_k$ and $\hat \nu_k$ weakly converges to the mean-field limit in \eqref{eq:pde} as the width $N$  goes to infinity and the stepsize scale $\epsilon$ goes to zero. Let $\rho_{t}(\theta, \ow) = \mu_t(\theta) \otimes \nu_{t}(\ow)$, where $(\mu_{t}, \nu_{t})$ is the PDE solution to the continuous deterministic dynamics in \eqref{eq:pde} and $\hat \rho_{k} = N^{-1} \cdot \sum_{i = 1}^{N} \delta_{\theta_{k}^{i}} \cdot \delta_{\ow_{k}^{i}}$ corresponds to the empirical distribution of $(\boldsymbol{\theta}_k, \boldsymbol{\ow}_{k})$, which is $k$-th iterate of the discrete time stochastic dynamics in \eqref{eq:gdnn-fin} with stepsize scale $\epsilon$. The following proposition proves that the PDE solution $\rho_t$ in \eqref{eq:pde} well approximates the discrete time stochastic gradient descent-ascent dynamics in \eqref{eq:gdnn-fin}.

\begin{proposition}[Convergence of SGDA to Mean-Field Limit] \label{prop:weak-formal-formal} 
Let $\{\rho_t\}_{t \geq 0}$ be solution to \eqref{eq:pde} with $\rho_0 = \cN(0, I_{D}) \otimes \cN(0, I_{D})$,  $\{\hat \rho_{k} \}_{k \geq 0}$ be solution to \eqref{eq:gdnn-fin} with $\hat \rho_0 = \cN(0, I_{D}) \otimes \cN(0, I_{D})$. Under Assumption \ref{asp:reg} and \ref{asp:compact}, $\widehat{\rho}_{\lfloor t/\epsilon\rfloor}$ converges weakly to $\rho_t$ as $\epsilon\rightarrow 0^+$ and $N\rightarrow \infty$. 
It holds for any Lipschitz continuous, bounded function $ F: \RR^D \times \RR^D \rightarrow \RR$ that
\begin{align*}
% \label{eq:weak}
\lim_{\epsilon\rightarrow 0^+, N\rightarrow \infty}\int F(\theta, \ow) \rd \widehat{\rho}_{\lfloor t/\epsilon\rfloor}(\theta, \ow) = \int F(\theta, \ow) \rd \rho_{t}(\theta,\ow).
\end{align*}
\end{proposition}

\begin{proof}
See \S\ref{sec:pf-prop-weak} for a detailed proof.
\end{proof}

\noindent The proof of Proposition \ref{prop:weak-formal-formal} is based on the propagation of chaos \citep{mei2018mean, mei2019mean, araujo2019mean, zhang2020can, sznitman1991topics}. We deferred the detailed proof of Proposition \ref{prop:weak-formal-formal} to Appendix \ref{sec:pf-prop-weak}. Proposition \ref{prop:weak-formal-formal} allows us to convert the discrete-time SGDA dynamics over finite dimensional parameter space to its continuous-time, infinite-dimensional counter-part in Wasserstein space, in which the training is amenable to analysis since our infinitely wide neural network $f(\cdot; \mu)$ and $g(\cdot;\nu))$ in \eqref{eq:nn} is linear in $\mu$ and $\nu$ respectively.

\subsection{Global Optimality and Convergence of the Mean-Field Limit} \label{sec:conv}

In this section, we will introduce our main results that characterize the global optimality and convergence of the mean-field neural networks, parameterized by the parameter distribution $\rho_t = (\mu_t, \nu_t)$. The proof contains two steps. We first show that it is sufficient to find a stationary point of the Wasserstein gradient flow defined in \eqref{eq:pde} in order to solve the minimax optimization problem in \eqref{eq:minimax obj}, then we characterize the convergence of $\rho_t$ to the stationary point. Before presenting the two stages of the proof, we would need to further clarify the notions of stationarity regarding the Wasserstein gradient flow. We introduce the following definition,

\begin{definition}[Stationary point of Wasserstein Gradient Flow]
\label{def:stat-wass}
A distribution pair $(\mu, \nu)$ is called a stationary point of the Wasserstein gradient flow \eqref{eq:pde} if it satisfies
\begin{align*}
    v^{f}(\theta; \mu, \nu) = v^{g}(\ow; \mu, \nu) = 0, \quad \forall \theta, \ow \in \RR^{D}.
\end{align*}
\end{definition}
\noindent From Definition \ref{def:stat-wass}, the stationary point of Wasserstein gradient flow \eqref{eq:pde} is a distribution pair $(\mu, \nu)$, at which the associated vector field $(v^{f}(\cdot; \mu, \nu), v^{g}(\cdot; \mu, \nu))$ is a zero function on the parameter space $\RR^{D} \times \RR^{D}$. Moreover, for the Wasserstein gradient flow following vector field $(v^f, v^g)$  and initial condition $(\mu, \nu)$, the solution to its associated continuity equation $(\mu_t, \nu_t)$ is a constant flow such that for all $t \geq 0$, $\mu_t = \mu, \nu_t = \nu$. Now, we have the following important supporting lemma that characterizes the relation between stationary points of Wasserstein gradient flow \eqref{eq:pde} and saddle points of \eqref{eq:minimax obj}.

\begin{lemma}
\label{lem:stat}
Under Assumptions~\ref{asp:reg} and \ref{asp:reliability}, the following two properties hold.
\begin{itemize}
    \item[(i)] Suppose that $(\mu^*, \nu^*)$ is a stationary point of the Wasserstein gradient flow as is defined in Definition \ref{def:stat-wass}. Then, the corresponding function $(f(\cdot; \mu^*), g(\cdot; \nu^*))$ is the saddle point of the objective function $\cL(f,g)$ defined in \eqref{eq:minimax obj}.
    \item[(ii)] There exists a stationary distribution pair $(\mu^*, \nu^*)$ and constant $\bar D > 0$ such that
    \begin{align*}
        W_2(\mu_0, \mu^*) \le \alpha^{-1} \bar D, \quad W_2(\nu_0, \nu^*) \le \alpha^{-1} \bar D.
    \end{align*}
\end{itemize}
\end{lemma}
\noindent Lemma~\ref{lem:stat} demonstrates that the stationary point of the Wasserstein gradient flow in \eqref{eq:pde} achieves global optimality as a solution to the minimax objective \eqref{eq:minimax obj}. Lemma ~\ref{lem:stat} allows us to bypass the hardness of solving the nonconvex-nonconcave optimization problem \eqref{eq:minimax obj} of finding saddle points in the space of neural network parameters $(\btheta, \bow)$ by searching for a stationary point of the Wasserstein gradient flow instead. Moreover, there exist good pairs of stationary points that are close to the Gaussian initialization $(\mu_0, \nu_0)$, with Wasserstein distance upper bounded by order $\cO(\alpha^{-1})$.

\begin{proof}
See \S\ref{sec:pf-lem-stat} for a detailed proof.
\end{proof}

\noindent We are now ready to show our main results. The following theorem characterizes the global optimality and convergence of the Wasserstein gradient flow $\rho_t$.
\begin{theorem}[Global Convergence to Saddle Point]
\label{th:main} Let $(\mu_t, \nu_t)$ be the solution to the Wasserstein gradient flow \eqref{eq:pde} at time $t$ with $\eta = \alpha^{-2}$ and initial condition $\mu_0 = \nu_0 = \mathcal{N}(0, I_D)$, $(f^*, g^*)$ the saddle point of the minimax objective $\cL(f,g)$ in \eqref{eq:minimax obj}. Under Assumptions~\ref{asp:reg}, \ref{asp:reliability}, and \ref{asp:compact}, it holds that
\begin{align}
\label{eq:th47}
    & \inf_{t \in [0,T]} \EE_{\cD} \Bigl[ \lambda \Psi\big(X,Z; f(\cdot; \mu_t) - f^*(\cdot)\big) + \bigl(g(Z; \nu_t) - g^*(Z)\bigr)^2  \Bigr] \le \cO(T^{-1} + \alpha^{-1}).
\end{align}
\end{theorem}
\begin{proof}
See \S\ref{sec:pf-th-main} for a detailed proof.
\end{proof}

\noindent Theorem \ref{th:main} says that the optimality gap between $(f(\cdot; \mu_t), g(\cdot; \nu_t))$ and $(f^*, g^*)$, quantified by the $\Psi$-induced distance and $L^2$ distance respectively, decays to zero at a sublinear rate in terms of time $T$ up to the error of $\cO(\alpha^{-1})$, where $\alpha > 0$ is the scaling parameter in \eqref{eq:nn-fin} and \eqref{eq:nn}. In order to prove the convergence, we construct a potential $V(\mu,\nu) = \EE_{\cD} \Bigl[ \lambda \Psi\big(X,Z; f(\cdot; \mu) - f^*(\cdot)\big) + \bigl(g(Z; \nu) - g^*(Z)\bigr)^2  \Bigr]$, with $V(\mu, \nu) = 0$ if and only if $(\mu, \nu) = (\mu^{*}, \nu^{*})$. Such a potential characterizes the saddle point of the minimax objective. We show that the Wasserstein gradient flow decreases the potential at a sublinear rate, thus suggesting the convergence of the gradient flow to the saddle point. Moreover, varying $\alpha$ allows a trade-off between the error of order $\cO(\alpha^{-1})$ in the optimality gap and the maximum deviation between $\rho_t$ and the Gaussian initialization $\rho_0$ for all $t$. In the proof of item (ii) of Lemma \ref{lem:stat}, we proved that the deviation of $\rho_t$ from $\rho_0$ quantified by the Wasserstein distance is of order $\cO(\alpha^{-1})$. Regarding representation learning, this suggests that SGDA induces a data-dependent representation that is significantly different from the initialization. Choosing a small $\alpha$ of order $\cO(1)$ will correspond to the mean-field regime \citep{mei2018mean, mei2019mean} that allows $\rho_t$ to move further away from the initialization, with the potential drawback of yielding a large error of order $\cO(\alpha^{-1})$. On the other hand, choosing a large $\alpha$ of order $\cO(\sqrt{N})$ will correspond to the NTK regime \citep{jacot2018neural}, and this causes the Wasserstein flow $\rho_t$ to stay close to the initial distribution $\rho_0$ along the trajectory, inducing a data-independent representation. \vspace{5pt}

\noindent As we have commented before, an important class of regularizer $\Psi(X, Z;f)$ is the $L^2$ regularizer. In this scenario, the left-hand side of \eqref{eq:th47} should be understood as a weighted $L^2$ distance between the gradient flow iterate at time $t$ to the optimal solution $(f^*, g^*)$. As $T$ and $ \alpha$ go to infinity, such a distance will shrink to $0$, thus the gradient flow converges globally in the minimal distance sense to the optimal solution. Due to this observation, in the sequel we will discuss several additional results in the case where the regularizer $\Psi(X, Z;f)$ is strongly convex, in the sense that it's bounded below by a quadratic function. We formalize the additional constraint in this case with the following assumption,

\begin{assumption}[Strong Convexity] \label{asp:strong-convex}
We assume that the regularizer $\Psi(X, Z; f)$ is $c_{\Psi}$-strongly convex, in the sense that there exists a constant $c_{\Psi} > 0$ such that for any $f \in \cF$,
\begin{align*}
\Psi(x,z; f) \geq c_{\Psi} \cdot |f(w)|^{2}, \quad \forall (x,z) \in \cX \times \cZ,
\end{align*}
where $w \in \cW$ is part of the data tuple $(x,z)$.
\end{assumption}
\noindent Assumption \ref{asp:strong-convex} implies that regularizer $\Psi(X, Z;f)$ is equivalent to a quadratic regularizer because $\Psi$ is simultaneously bounded above and below by quadratic functionals. We have the following strengthened version of Theorem \ref{th:main} in such case,
\begin{align}
\label{eq:th-main-strong-convex}
    \inf_{t \in [0,T]} \EE_{\cD} \Bigl[ \lambda c_{\Psi} \cdot \big(f(\cdot; \mu_t) - f^*(\cdot)\big)^{2} + \bigl(g(Z; \nu_t) - g^*(Z)\bigr)^2  \Bigr] \le \cO(T^{-1} + \alpha^{-1}).
\end{align}
Equation \eqref{eq:th-main-strong-convex} shows that the iterates $(f(\cdot; \mu_t), g(\cdot; \nu_t))$ converges to the saddle point solution $(f^*, g^*)$ as a weighted $L^{2}$ distance decays to zero at a sublinear rate up to an error of $\cO(\alpha^{-1})$. With Assumption \ref{asp:reliability}, the saddle point $f^{*}$ is the global optimizer of the primal functional $J(f)$ defined in \eqref{eq:mse}. Therefore, as a direct consequence of Theorem \ref{th:main}, when the regularizer $\Psi$ is strongly convex, $f(\cdot; \mu_{t})$ converges globally to $f^*$ at a sublinear rate in terms of $T$ up to an error of $\cO(\alpha^{-1})$. \vspace{5pt}

\noindent Under Assumption \ref{asp:strong-convex}, we can also quantify the optimality gap between $J(f_t)$ and $J(f^*)$, in terms of the minimal distance $\inf_{t \in [0,T]} J(f_{t}) - J(f^{*})$. The following theorem characterize the global convergence of $J(f_t)$ to $J(f^*)$,

\begin{theorem}[Global Convergence to Primal Solution]
	\label{th:global}
  Let $(\mu_t, \nu_t)$ be the solution to the Wasserstein gradient flow \eqref{eq:pde} at time $t$ with $\eta = \alpha^{-2}$ and initial condition $\mu_0 = \nu_0 = \mathcal{N}(0, I_D)$. Under Assumptions~\ref{asp:reg}, \ref{asp:reliability}, \ref{asp:compact} and \ref{asp:strong-convex}, let $f_{t} = f(\cdot; \mu_{t})$, it holds that
	\begin{align*}
		& \inf_{t \in [0,T]} J(f_{t}) - J(f^{*}) \le \cO(T^{-1/2} + \alpha^{-1/2}),
	\end{align*}
	where $f^{*}$ is the global minimizer of the objective function defined in \eqref{eq:mse}.
\end{theorem}
\begin{proof}
See \S\ref{sec:pf-th-global} for a detailed proof.
\end{proof}

\noindent Theorem \ref{th:global} proves that under the additional strong convexity assumption on the regularizer $\Psi(X, Z;f)$, the primal objective $J(f_t)$ as is defined in \eqref{eq:mse} decays to zero at rate of $T^{-1/2}$ in terms of time horizon $T$, up to an error of $\cO(\alpha^{-1/2})$. Here we use $f^*$ to denote the global minimizer instead of the saddle point. However, this will not create any confusion since for each $f^*$ global minimizer of the primal objective \eqref{eq:mse}, we can find $g^* \in \cF$ such that $(f^*, g^*)$ is a saddle point of \eqref{eq:minimax obj}.

%!TEX root = main.tex

\section{Applications}
In this section, we present the applications of Theorem \ref{th:main} and Theorem \ref{th:global} to several special cases of functional conditional moment equation, such as the problem of policy evaluation, instrumental variables regression, asset pricing, and adversarial Riesz representer estimation. In Section \ref{sec:examples}, we already discussed why these problems are special cases of functional conditional moment equations, thus Theorem \ref{th:main} and Theorem \ref{th:global} are potentially feasible to apply. We will recall the problem settings and examine the technical assumptions for these cases.

\subsection{Application 1: Policy Evaluation}

Let $\cD$ denote the joint distribution of the state-action tuple $(S, A, S')$ under policy $\pi$. In this scenario, the endogenous variable $X = S'$ is the next state while the exogenous variable  $Z = S$ is the current state. Therefore, $\cX = \cS$, $\cZ = \cS$ and $\cW = \cS$. We attempt to estimate the value function $V$, which is defined on $\cW = \cS \rightarrow \RR$. The functional $\Phi$ and regularizer $\Psi$ adopted in this case are,
\begin{align*}
    \Phi(s',s;f) = r + \gamma \cdot f(s') - f(s), \quad \Psi(s',s;f) = f(s')^2.
\end{align*}
Here, the regularizer we adopt is a $L^{2}$ regularizer that penalizes the squared value of the estimator evaluated at the next state $s'$. With these specific choices of functional $\Phi$ and regularizer $\Psi$, the SGDA algorithm identifies with the Gradient Temporal Difference Learning (GTD) algorithm \citep{wai2020provably}. Therefore, the application of our general framework to the problem of policy evaluation contributes to the reinforcement learning literature by providing an analysis of the neural GTD algorithm in the mean-field regime. 
Before presenting the theoretical results, we first verify that  Assumption \ref{asp:compact} and Assumption \ref{asp:strong-convex} hold. \vspace{5pt}

\noindent\textbf{Verify item (i) of Assumption \ref{asp:compact}.}
  For item (i) of Assumption \ref{asp:compact}, it's reasonable to assume that $\|(x,z)\| \leq 1$ since we can always re-scale the state space without changing the nature of the problem, therefore the compactness assumption is inherently satisfied. \vspace{5pt}
  
\noindent \textbf{Verify item (ii) of Assumption \ref{asp:compact}}. For item (ii) of Assumption \ref{asp:compact}, we first compute the variation of the functional $\Phi$ and $\Psi$,
\begin{align*}
    \frac{\delta \Phi(s',s;f)}{\delta f}(w') = \gamma \delta_{s'}(w') - \delta_{s}(w'), \quad \frac{\delta \Psi(s',s;f)}{\delta f}(w') = 2f(s')\delta_{s'}(w').
\end{align*}
Therefore, the desired integrability conditions hold since
\begin{align}
\label{eq:app-policy-1}
\int_{\cW} \Big| \frac{\delta \Phi(s',s;f)}{\delta f}(w') \Big| \rd w' \leq \gamma + 1, \quad \int_{\cW} \Big| \frac{\delta \Psi(s',s;f)}{\delta f}(w') \Big| \rd w' \leq 2 \cdot |f(s')|.
\end{align}
\noindent \textbf{Verify item (iii) of Assumption \ref{asp:compact}}. For item (iii) of Assumption \ref{asp:compact}, we choose $w = s'$, $C_{\Psi} = 2$. The desired condition holds due to \eqref{eq:app-policy-1}. \vspace{5pt}

\noindent \textbf{Verify item (iv) of Assumption \ref{asp:compact}}. For item (iv) of Assumption \ref{asp:compact}, we first compute the variations of $\cL(f,g)$ in explicit forms,
\begin{align*}
& \frac{\delta \cL(f,g)}{\delta f}(w') = \EE_{S|S'}\Big[\gamma \cdot g(S)  \biggiven S' = w'\Big] -g(w')\rho_{s}(w') + 2\lambda \cdot f(w')\rho_{S'}(w'), \quad \forall w' \in \cS, \\
& \frac{\delta \cL(f,g) }{\delta g}(z') = \EE_{S'|S}\Big[r + \gamma \cdot f(s')  \biggiven S = z'\Big] - f(z') - g(z') \rho_{S}(z'), \quad \forall z' \in \cS,
\end{align*}
where $\rho_{S}$, $\rho_{S'}$ denotes the density of the marginal distribution of $\cD$ with respect to the current state $S$ and next state $S'$ respectively. Due to the item (i) of Assumption \ref{asp:compact}, the variations of $\cL$ with respect to $f$ and $g$ are both continuous since the density of the conditional transition $S'  \biggiven S$ and $S  \biggiven S'$ are both smooth and the functions $f,g$ are also continuous by construction. Therefore, item (iv) is satisfied. \vspace{5pt}

\noindent \textbf{Verify Assumption \ref{asp:strong-convex}}. For Assumption \ref{asp:strong-convex}, we choose $c_{\Psi} = 1$ and $w = s'$. The desired condition holds by definition of our choice of regularizer $\Psi(s',s;f) = f(s')^2$. \vspace{5pt}

\noindent We have checked that the technical Assumption \ref{asp:compact} and Assumption \ref{asp:strong-convex} hold for the case of policy evaluation. Assumption \ref{asp:compact} allows us to apply Theorem \ref{th:main}. This implies the global convergence of the estimated value function to the minimizer of the primal objective \eqref{eq:mse} applied in this case. The convergence is quantified in a weighted $L^2$ distance. Additionally, Assumption \ref{asp:strong-convex} enables us to apply Theorem \ref{th:global} and further characterize such convergence using the optimality gap between the value of primal objectives. We summarize the conclusions in the following corollary.

\begin{corollary}[Global Convergence of Mean-field Neural Nets in Policy Evaluation]
\label{cor:policy} Let $f^{*}$ be the minimizer of primal objective $J(f)$ defined in \eqref{eq:obj cmr1} with $\Phi(S',S;f) = r + \gamma \cdot f(S') - f(S)$, $\cR(f) = \EE_{\cD}[f(S')^{2}]$. Let $(\mu_t, \nu_t)$ be the solution to the Wasserstein gradient flow \eqref{eq:pde} at time $t$ with $\eta = \alpha^{-2}$ and initial condition $\mu_0 = \nu_0 = \mathcal{N}(0, I_D)$. Under Assumption ~\ref{asp:reg}, \ref{asp:reliability}, \ref{asp:compact}, and \ref{asp:strong-convex}, it holds that
\begin{align*}
& \inf_{t \in [0,T]} \EE_{\cD} \Bigl[(f(S'; \mu_t) - f^{*}(S'))^{2} \Bigr] \leq \cO(T^{-1} + \alpha^{-1}), \\
& \inf_{t \in [0,T]} J(f(\cdot; \mu_t)) - J(f^{*}(\cdot)) \leq \cO(T^{-1/2} + \alpha^{-1/2}).
\end{align*}
\end{corollary}
\begin{proof}
We apply Theorem \ref{th:main} and Theorem \ref{th:global} to the setting of policy evaluation. As we have examined above, the Assumption \ref{asp:reg}, \ref{asp:reliability}, \ref{asp:compact}, and \ref{asp:strong-convex} are all satisfied. Thus, by Theorem \ref{th:main} and Theorem \ref{th:global}, the desired results hold directly.
\end{proof}
\noindent Corollary \ref{cor:policy} proves that in the setting of policy evaluation, the $L^2$ distance between the mean-field neural network $f(\cdot; \mu_t)$ at time $t$ and the global minimizer $f^*$ decays to zero at a sub-linear rate, up to an error of order $\cO(\alpha^{-1})$. Moreover, the optimality gap $\inf_{t \in [0, T]} J(f(\cdot; \mu_t)) - J(f^{*}(\cdot))$ in terms of primal objective values decays to zero at the rate of $\cO(T^{-1/2})$, up to an error $\cO(\alpha^{-1/2})$ caused by overparameterization. Corollary \ref{cor:policy} allows us to efficiently and globally solve the policy evaluation problem using overparameterized two-layer neural networks. We also remark that in such a scenario, the primal objective $J(f)$ is also known as the regularized mean-squared Bellman error (MSBE) in the literature of reinforcement learning. As we have commented before, in the setting of policy evaluation, applying the SGDA algorithm within neural network function classes is equivalent to applying the neural GTD algorithm. Therefore, Corollary \ref{cor:policy} states that, in the mean-field regime, the neural GTD algorithm converges globally to the minimizer at a sublinear rate up to an additional overparameterization error $\cO(\alpha^{-1})$. The neural GTD algorithm also reduces regularized MSBE at the rate of $\cO(T^{-1/2})$ up to an additional overparameterization error $\cO(\alpha^{-1/2})$. Moreover, The global convergence of mean-field neural networks also implies the global convergence of the discrete dynamics in \eqref{eq:gdnn-fin} due to the proximity between the discrete dynamics and continuous dynamics, which is proved in Proposition \ref{prop:weak-formal-formal}.

\subsection{Application 2: Nonparametric Instrumental Variables Regression}

Let $\cD$ denote the joint distribution of the endogenous variable $X$, the exogenous variable $Z$, and the observed outcome $Y$. In this scenario, the endogenous variable is defined in space $\cX$, the exogenous variable is defined in space $\cZ$, and $\cW = \cX$. We attempt to estimate the model function $f_0$, which is defined on $\cW = \cX \rightarrow \RR$. The functional $\Phi$ and regularizer $\Psi$ adopted in this case are, 
\begin{align*}
    \Phi(x,z;f) = y - f(x), \quad \Psi(x,z;f) = f(x)^2.
\end{align*}
Here, the regularizer we adopt is a $L^{2}$ regularizer that penalizes the squared value of the estimator of the model function evaluated at the endogenous variable $x$. We examine Assumption \ref{asp:compact} and Assumption \ref{asp:strong-convex} in order to apply results from Section \ref{sec:conv}. \vspace{5pt}

\noindent \textbf{Verify item (i) of Assumption \ref{asp:compact}}. For item (i) of Assumption \ref{asp:compact}, the NPIV problem with compact data space captures a large class of important applications, therefore the scenarios considered are still general while imposing this assumption.  \vspace{5pt}

\noindent \textbf{Verify item (ii) of Assumption \ref{asp:compact}}. For item (ii) of Assumption \ref{asp:compact}, we first compute the variation of the functional $\Phi$ and $\Psi$,
\begin{align*}
    \frac{\delta \Phi(x,z;f)}{\delta f}(w') = -\delta_{x}(w'), \quad \frac{\delta \Psi(x,z;f)}{\delta f}(w') = 2f(x)\delta_{x}(w').
\end{align*}
Therefore, the desired integrability conditions hold since
\begin{align}
\label{eq:app-cond-1}
\int_{\cW} \Big| \frac{\delta \Phi(x,z;f)}{\delta f}(w') \Big| \rd w' \leq 1, \quad \int_{\cW} \Big| \frac{\delta \Psi(x,z;f)}{\delta f}(w') \Big| \rd w' \leq 2 \cdot |f(x)|.
\end{align}
\noindent \textbf{Verify item (iii) of Assumption \ref{asp:compact}}. For item (iii) of Assumption \ref{asp:compact}, we choose $w = x$, $C_{\Psi} = 2$. The desired condition holds due to \eqref{eq:app-cond-1}. \vspace{5pt}

\noindent \textbf{Verify item (iv) of Assumption \ref{asp:compact}}. For item (iv) of Assumption \ref{asp:compact}, we first compute the variations of $\cL(f,g)$ in explicit forms,
\begin{align*}
& \frac{\delta \cL(f,g)}{\delta f}(w') = \EE_{Z|X}\Big[-g(Z)  \biggiven X = w' \Big] + 2\lambda \cdot f(w')\rho_{X}(w'), \quad \forall w' \in \cX,\\
& \frac{\delta \cL(f,g) }{\delta g}(z') = \EE_{X|Z} \Big[Y - f(X)  \biggiven Z = z'\Big] - g(z') \rho_{Z}(z'), \quad \forall z' \in \cZ,
\end{align*}
where $\rho_{X}$, $\rho_{Z}$ denotes the density of the marginal distribution of $\cD$ with respect to the endogenous variable $X$ and the exogenous variable $Z$ respectively. Due to the item (i) of Assumption \ref{asp:compact}, the variations of $\cL$ with respect to $f$ and $g$ are both continuous since the density of the conditional transition $Z \biggiven X$ and $X \biggiven Z$ are both smooth and the functions $f,g$ are also continuous by construction. Therefore, item (iv) is satisfied. \vspace{5pt}

\noindent \textbf{Verify Assumption \ref{asp:strong-convex}}. For Assumption \ref{asp:strong-convex}, we choose $c_{\Psi} = 1$ and $w = s'$. The desired condition holds by definition of our choice of regularizer $\Psi(x,z;f) = f(x)^2$. \vspace{5pt}

\noindent We have checked that the technical Assumption \ref{asp:compact} and Assumption \ref{asp:strong-convex} hold for the case of nonparametric instrumental variables regression. Theorem \ref{th:main} can be applied in this case due to the establishment of Assumption \ref{asp:compact}. This implies the global convergence of the estimated model function to the minimizer of the primal objective. The convergence is quantified in a weighted $L^{2}$ distance. The choice of quadratic regularizer implies the establishment of Assumption \ref{asp:strong-convex}, which further enables us to apply Theorem \ref{th:global} and characterize the convergence in terms of primal objective value. We summarize the conclusions in the following corollary.

\begin{corollary}[Global Convergence of Mean-field Neural Nets in NPIV]
\label{cor:cme} Let $f^{*}$ be the minimizer of primal objective $J(f)$ defined in \eqref{eq:obj cmr1} with $\Phi(X,Z;f) = Y - f(X)$, $\cR(f) = \EE_{\cD}[f(X)^{2}]$. Let $(\mu_t, \nu_t)$ be the solution to the Wasserstein gradient flow \eqref{eq:pde} at time $t$ with $\eta = \alpha^{-2}$ and initial condition $\mu_0 = \nu_0 = \mathcal{N}(0, I_D)$. Under Assumption ~\ref{asp:reg}, \ref{asp:reliability}, \ref{asp:compact}, and \ref{asp:strong-convex}, it holds that
\begin{align*}
& \inf_{t \in [0,T]} \EE_{\cD} \Bigl[(f(X; \mu_t) - f^{*}(X))^{2}\Bigr] \leq \cO(T^{-1} + \alpha^{-1}),\\
& \inf_{t \in [0,T]} J(f(\cdot; \mu_t)) - J(f^{*}(\cdot)) \leq \cO(T^{-1/2} + \alpha^{-1/2}).
\end{align*}
\end{corollary}

\begin{proof}
We apply Theorem \ref{th:main} and Theorem \ref{th:global} to the setting of NPIV. As we have examined above, the Assumption \ref{asp:reg}, \ref{asp:reliability}, \ref{asp:compact}, and \ref{asp:strong-convex} are all satisfied. Thus, by Theorem \ref{th:main} and Theorem \ref{th:global}, the desired results hold directly.
\end{proof}
\noindent Corollary \ref{cor:cme} proves that in the setting of NPIV, the $L^2$ distance between the mean-field neural network $f(\cdot; \mu_t)$ at time $t$ and the global minimizer $f^*$ decays to zero at a sub-linear rate, up to an error of order $\cO(\alpha^{-1})$. Moreover, the optimality gap $\inf_{t \in [0,T]} J(f(\cdot; \mu_t)) - J(f^{*}(\cdot))$ decays to zero at the rate of $\cO(T^{-1/2})$, up to an error $\cO(\alpha^{-1/2})$. Corollary \ref{cor:cme} allows us to solve the NPIV problem globally using overparameterized two-layer neural networks. We also want to remark that when the true model function is linear in the input, we recover the setting of instrumental variables regression as an important special instance of NPIV.
Therefore, Corollary \ref{cor:cme} also implies IV regression can be globally solved efficiently by using overparameterized two-layer neural networks.

\subsection{Application 3: Asset Pricing}
Let $\cD$ denote the joint distribution of the growth-return tuple $(c_t, \tilde r_{t+1}, c_{t+1})$. In this scenario, the exogenous variable $Z = c_t$ is the consumption growth at the current time $t$, and the endogenous variable $X = c_{t+1}$ is the consumption growth at the next time $t+1$. Therefore, $\cX = \cZ = \cC$, $\cW = \cC$ where $\cC$ is the space of consumption growth and is also a compact subset of $\mathbb{R}$. Here, we consider the scenario where the modified return $\tilde r_{t+1}$ is also bounded for all $t \geq 0$, i.e., $\|\tilde r_{t+1}\| \leq R$ for some $R > 0$. 
We attempt to estimate the function $f_0$, which is defined on $\cW = \cS \to \mathbb{R}$. The functional $\Phi$ and regularizer $\Psi$ adopted in this case are,
\begin{align*}
    \Phi(c_{t+1}, c_t; f) = \tilde r_{t+1} \cdot f(c_{t+1}) - f(c_t), \quad \Psi(c_{t+1}, c_t;f) = f(c_{t+1})^{2}.
\end{align*}
Here, the regularizer we adopt is a $L^{2}$ regularizer that penalizes the squared value of the estimator evaluated at the consumption growth of the next time $c_{t+1}$. Before presenting the theoretical results, we first verify that Assumption \ref{asp:compact} and Assumption \ref{asp:strong-convex} hold. \vspace{5pt}

\noindent \textbf{Verify item (i) of Assumption \ref{asp:compact}.} For item (i) of Assumption \ref{asp:compact}, since we assume that the space of consumption growth $\cC$ is a compact subset of $\mathbb{R}$, therefore there exists $C_{1} > 0$ such that for all $t \geq 0$, $\|(c_{t+1}, c_t)\| \leq C_{1}$. Moreover, it is reasonable to assume that the consumption growth is bounded since the data often fluctuates within certain regimes in practice. \vspace{5pt}

\noindent \textbf{Verify item (ii) of Assumption \ref{asp:compact}.} For item (ii) of Assumption \ref{asp:compact}, we first compute the variation of the functional $\Phi$ and $\Psi$,
\begin{align*}
\frac{\delta \Phi(c_{t+1}, c_t;f)}{\delta f}(w') = \tilde r_{t+1} \cdot \delta_{c_{t+1}}(w') - \delta_{c_t}(w'), \quad \frac{\delta \Psi(c_{t+1}, c_t;f)}{\delta f}(w') = 2f(c_{t+1})\cdot \delta_{c_{t+1}}(w').
\end{align*}
Therefore, the desired integrability condition holds since,
\begin{align}
\label{eq:app-ccapm-1}
\int_{\cW} \Big| \frac{\delta \Phi(c_{t+1}, c_t;f)}{\delta f}(w') \Big| \rd w' \leq R + 1, \quad \int_{\cW} \Big| \frac{\delta \Psi(c_{t+1}, c_t;f)}{\delta f}(w') \Big| \rd w' \leq 2 \cdot |f(c_{t+1})|.
\end{align}
\noindent \textbf{Verify item (iii) of Assumption \ref{asp:compact}.} For item (iii) of Assumption \ref{asp:compact}, we choose $w = \tilde c$, $C_{\Psi} = 2$. The desired property holds due to \eqref{eq:app-ccapm-1}. \vspace{5pt}

\noindent \textbf{Verify item (iv) of Assumption \ref{asp:compact}.} For item (iv) of Assumption \eqref{asp:compact}, we first compute the variations of $\cL(f,g)$ in explicit forms,
\begin{align*}
& \frac{\delta \cL(f,g)}{\delta f}(w') = \EE_{c_t | c_{t+1}} \Big[\tilde r_{t+1} \cdot g(c_t)  \biggiven \tilde c_t = w' \Big] - g(w') \rho_{c_{t}}(w') + 2\lambda \cdot f(w') \rho_{c_{t+1}}(w'), \quad \forall w' \in \cS, \\
& \frac{\delta \cL(f,g)}{\delta g}(z') = \EE_{c_{t+1} | c_t} \Big[\tilde r_{t+1} \cdot f(c_{t+1})  \biggiven c_t = z' \Big] - f(z') - g(z') \rho_{c_t}(z'), \quad \forall z' \in \cS,
\end{align*}
where $\rho_{c_t}, \rho_{c_{t+1}}$ denotes the density of the marginal distribution of $\cD$ with respect to the current time consumption growth $c_t$ and the next time consumption growth $c_{t+1}$ respectively. The variations of $\cL$ with respect to $f$ and $g$ are both continuous since the density of the conditional transition $c_{t+1} \biggiven c_t$ and $c_t  \biggiven c_{t+1}$ are both smooth, and the function $f,g$ are also continuous by construction. Therefore, item (iv) is satisfied. \vspace{5pt}

\noindent \textbf{Verify Assumption \ref{asp:strong-convex}.} For Assumption \ref{asp:strong-convex}, we choose $c_{\Psi} = 1$ and $w = c_{t+1}$. The desired condition holds by definition of our choice of regularizer $\Psi(c_{t+1}, c_t;f) = f(c_{t+1})^2$. \vspace{5pt}

\noindent We have checked that the technical Assumption \ref{asp:compact} and Assumption \ref{asp:strong-convex} hold for the case of asset pricing with CCAPM model. Theorem \ref{th:main} can be applied in this case due to the establishment of Assumption \ref{asp:compact}. This implies the global convergence of the estimated function to the minimizer of the primal objective. The convergence is quantified in a weighted $L^{2}$ distance. Since Assumption \ref{asp:strong-convex} holds, we can apply Theorem \ref{th:global} and characterize the convergence in terms of primal objective value. We summarize the conclusions in the following corollary.

\begin{corollary}[Global Convergence of Mean-field Neural Nets in Asset Pricing]
\label{cor:ccapm} Let $f^{*}$ be the minimizer of primal objective $J(f)$ defined in \eqref{eq:obj cmr1} with $\Phi(c_{t+1}, c_t; f) = \tilde r_{t+1} \cdot f(c_{t+1}) - f(c_t)$, $\cR(f) = \EE_{\cD}[f(c_{t+1})^{2}]$. Let $(\mu_t, \nu_t)$ be the solution to the Wasserstein gradient flow \eqref{eq:pde} at time $t$ with $\eta = \alpha^{-2}$ and initial condition $\mu_0 = \nu_0 = \mathcal{N}(0, I_D)$. Under Assumption ~\ref{asp:reg}, \ref{asp:reliability}, \ref{asp:compact}, and \ref{asp:strong-convex}, it holds that
\begin{align*}
& \inf_{t \in [0,T]} \EE_{\cD} \Big[(f(c_{t+1}; \mu_t) - f^{*}(c_{t+1}))^{2} \Big] \leq \cO(T^{-1} + \alpha^{-1}), \\
& \inf_{t \in [0,T]} J(f(\cdot; \mu_t)) - J(f^{*}(\cdot)) \leq \cO(T^{-1/2} + \alpha^{-1/2}).
\end{align*}
\end{corollary}
\begin{proof}
We apply Theorem \ref{th:main} and Theorem \ref{th:global} to the setting of asset pricing. As we have examined above, the Assumption \ref{asp:reg}, \ref{asp:reliability}, \ref{asp:compact}, and \ref{asp:strong-convex} are all satisfied. Thus, by Theorem \ref{th:main} and Theorem \ref{th:global}, the desired results hold directly.
\end{proof}
\noindent Corollary \ref{cor:ccapm} proves that in the setting of asset pricing, the $L^2$ distance between the mean-field neural network $f(\cdot; \mu_t)$ at time $t$ and the global minimizer $f^*$ decays to zero at a sub-linear rate, up to an error of order $\cO(\alpha^{-1})$. Moreover, the optimality gap $\inf_{t \in [0,T]} J(f(\cdot; \mu_t)) - J(f^{*}(\cdot))$ decays to zero at the rate of $\cO(T^{-1/2})$, up to an error $\cO(\alpha^{-1/2})$. Corollary \ref{cor:ccapm} allows us to solve the CCAPM model globally by estimating the nonparametric structural demand function with overparameterized two-layer neural networks. Since the return on investment is linked to the marginal utility of consumption through the CCAPM equation, we can price fairly the assets by considering consumption risk and utilizing the marginal utility information.

\subsection{Application 4: Adversarial Riesz Representer Estimation}
Let $\cD$ denote the joint distribution of the endogenous variable $X$ and the random vector $V$. In this scenario, the exogenous variable $Z$ coincides with the endogenous variable $X$, therefore the problem is essentially unconditional. The endogenous variable is defined in space $\cX$, the exogenous variable is defined on $\cZ = \cX$, and $\cW = \cX$. We attempt to estimate the Riesz representer $f_0$, which is defined on $\cW = \cX \rightarrow \RR$.  The functional $\Phi$ and regularizer $\Psi$ adopted in this case are,
\begin{align*}
\Phi(x,x;f) = f_0(x) - f(x), \quad \Psi(x,x;f) = f(x)^2.
\end{align*}
Here, the regularizer we adopt is a $L^{2}$ regularizer that penalizes the squared value of estimator of the Riez representer evaluated at the variable $x$. We examine Assumption \ref{asp:compact} and Assumption \ref{asp:strong-convex} in order to apply results from Section \ref{sec:conv}.\vspace{5pt}

\noindent \textbf{Verify item (i) of Assumption \ref{asp:compact}}. For item (i) of Assumption \ref{asp:compact}, we restrict our attention to estimating Riesz represented of functionals defined on a compact space. In practice, such an assumption is very general since we often treat data distribution on an unbounded space with exponential decay as a distribution defined on a compact space. \vspace{5pt}

\noindent \textbf{Verify item (ii) of Assumption \ref{asp:compact}}. For item (ii) of Assumption \ref{asp:compact}, we first compute the variation of the functional $\Phi$ and $\Psi$,
\begin{align*}
    \frac{\delta \Phi(x,x;f)}{\delta f}(w') = -\delta_{x}(w'), \quad \frac{\delta \Psi(x,x;f)}{\delta f}(w') = 2f(x)\delta_{x}(w').
\end{align*}
Therefore, the desired integrability conditions hold since
\begin{align}
\label{eq:app-riesz-1}
\int_{\cW} \Big| \frac{\delta \Phi(x,x;f)}{\delta f}(w') \Big| \rd w' \leq 1, \quad \int_{\cW} \Big| \frac{\delta \Psi(x,x;f)}{\delta f}(w') \Big| \rd w' \leq 2 \cdot |f(x)|.
\end{align}

\vspace{5pt}
\noindent \textbf{Verify item (iii) of Assumption \ref{asp:compact}}. For item (iii) of Assumption \ref{asp:compact}, we choose $w = x$, $C_{\Psi} = 2$. The desired condition holds due to \eqref{eq:app-riesz-1}. \vspace{5pt}

\noindent \textbf{Verify item (iv) of Assumption \ref{asp:compact}}. For item (iv) of Assumption \ref{asp:compact}, we first compute the variations of $\cL(f,g)$ in explicit forms,
\begin{align*}
& \frac{\delta \cL(f,g)}{\delta f}(w') = \EE_{Z|X}\Big[-g(Z)  \biggiven X = w' \Big] + 2\lambda \cdot f(w')\rho_{X}(w'), \quad \forall w' \in \cX,\\
& \frac{\delta \cL(f,g) }{\delta g}(z') = \EE_{X|Z} \Big[f_0(X) - f(X)  \biggiven Z = z'\Big] -  g(z') \rho_{Z}(z'), \quad \forall z' \in \cZ,
\end{align*}
where $\rho_{X}$, $\rho_{Z}$ denotes the density of the marginal distribution of $\cD$ with respect to the endogenous variable $X$ and the exogenous variable $Z$ respectively. Due to the item (i) of Assumption \ref{asp:compact}, the variations of $\cL$ with respect to $f$ and $g$ are both continuous since the density of the conditional transition $Z  \biggiven X$ and $X  \biggiven Z$ are both smooth and the functions $f,g$ are also continuous by construction. Therefore, item (iv) is satisfied. \vspace{5pt}

\noindent \textbf{Assumption \ref{asp:strong-convex}}. For Assumption \ref{asp:strong-convex}, we choose $c_{\Psi} = 1$ and $w = s'$. The desired condition holds by definition of our choice of regularizer $\Psi(x,x;f) = f(x)^2$. 

\vspace{5pt}
\noindent We have checked that the technical Assumption \ref{asp:compact} and Assumption \ref{asp:strong-convex} hold for the case of adversarial Riesz representer estimation. Theorem \ref{th:main} can be applied in this case due to the establishment of Assumption \ref{asp:compact}. This implies the global convergence of the estimated Riesz representer to the minimizer of the primal objective. The convergence is quantified in a weighted $L^{2}$ distance. The choice of quadratic regularizer implies the establishment of Assumption \ref{asp:strong-convex}, which further enables us to apply Theorem \ref{th:global} and characterize the convergence in terms of primal objective value. We summarize the conclusions in the following corollary.

\begin{corollary}[Global Convergence of Mean-field Neural Nets in Adversarial Riesz Representer Estimation]
\label{cor:riesz}  Let $f^{*}$ be the minimizer of primal objective $J(f)$ defined in \eqref{eq:obj cmr1} with $\Phi(x,x;f) = f_0(x) - f(x)$, $\cR(f) = \EE_{\cD}[f(x)^{2}]$. Let $(\mu_t, \nu_t)$ be the solution to the Wasserstein gradient flow \eqref{eq:pde} at time $t$ with $\eta = \alpha^{-2}$ and initial condition $\mu_0 = \nu_0 = \mathcal{N}(0, I_D)$. Under Assumption ~\ref{asp:reg}, \ref{asp:reliability}, \ref{asp:compact}, and \ref{asp:strong-convex}, it holds that
\begin{align*}
& \inf_{t \in [0,T]} \EE_{\cD} \Bigl[(f(X; \mu_t) - f^{*}(X))^{2}\Bigr] \leq \cO(T^{-1} + \alpha^{-1}), \\
& \inf_{t \in [0,T]} J(f(\cdot; \mu_t)) - J(f^{*}(\cdot)) \leq \cO(T^{-1/2} + \alpha^{-1/2}).
\end{align*}
\end{corollary}

\noindent Corollary \ref{cor:riesz} proves that in the setting of adversarial Riesz representer estimation, the $L^2$ distance between the mean-field neural network $f(\cdot; \mu_t)$ at time $t$ and the global minimizer $f^*$ decays to zero at a sub-linear rate, up to an error of order $\cO(\alpha^{-1})$. Moreover, the optimality gap $\inf_{t \in [0,T]} J(f(\cdot; \mu_t)) - J(f^{*}(\cdot))$ decays to zero at the rate of $\cO(T^{-1/2})$, up to an error $\cO(\alpha^{-1/2})$. Corollary \ref{cor:riesz} allows us to estimate the Riesz representer of a given functional using overparameterized two-layer neural networks.

\section{Conclusion} 
In this paper, we focus on the minimax optimization problem derived from solving functional conditional moment equations using overparameterized two-layer neural networks. 
For such a problem, we first prove that the stochastic gradient descent-ascent algorithm converges to a mean-field limit as the stepsize goes to zero and the network width goes to infinity. 
In this mean-field limit, the optimization dynamics is characterized by a Wasserstein gradient flow in the space of probability distributions. We further establish the global convergence of the Wasserstein gradient flow, and prove that the feature representation induced by the neural networks is allowed to move by a considerable distance from the initial value. We further apply our general results to policy evaluation with high dimensional state space, nonparametric instrumental variables regression with high dimensional endogenous and exogenous variables, and asset pricing with a nonparametric structural demand function, and general Riesz representer estimation.
Our analysis opens avenues for studying functional minimax optimization problems with more complicated objectives, such as nonlinear functional conditional moment equations. 
We leave the study of the convergence properties of the algorithm in such a general setting to future research. This setting includes nonparametric quantile instrumental variables regression as a leading and important application. 

\newpage

\bibliographystyle{ims}
\bibliography{main}

\newpage
\appendix

\section{Proof of Main Results}\label{sec:pf-main}
In this section, we provide proofs for the main theorems and technical lemmas in our work.
\subsection{Proof of Lemma \ref{lem:stat}}\label{sec:pf-lem-stat}
\vspace{4pt}

\noindent{\bf Proof of (i).}
The proof for Claim (i) will be two-stage. First, we will show that if function pair $(f^*, g^*)$ is a stationary point for $\cL(f,g)$ with respect to $(f,g)$, then it's a saddle point for the same objective as well. Then we will show that the distribution pair $(\mu^*, \nu^*)$ being a stationary point of $\cL(\mu, \nu)$ implies that the corresponding $(f^*, g^*)$ is a stationary point for $\cL(f,g)$, which concludes the claim. We will start with the first part. We define the following functional $\cL_1$ and $\cL_2$,
\begin{align*}
    \cL_{1}(f, g) = \EE_{\cD}\Big[g(Z) \cdot \Phi(X,Z;f)\Big], \quad \cL_{2}(f,g) = \EE_{\cD}\Big[-1/2 \cdot g(Z)^{2} + \lambda \Psi(X, Z;f)\Big].
\end{align*}
We see that the minimax objective in \eqref{eq:minimax obj} is indeed the sum of such two functionals,
\begin{align*}
    \cL(f, g) = \cL_{1}(f,g) + \cL_{2}(f,g).
\end{align*}
For any function pair $(f,g)$, we can verify that the following chain of equalities holds,
\begin{align}
    \label{eq:st1}
\max_{g'}~\cL(f, g') - \min_{f'}~\cL(f', g) &= \max_{g'}~\bigl( \cL(f, g') - \cL(f, g) \bigr) + \max_{f'}~\bigl(\cL(f, g) - \cL(f', g) \big).
\end{align}
We considered the function space $L^{2}(\cW)$ and $L^{2}(\cZ)$ equipped with inner product $\langle \cdot, \cdot \rangle_{L^2}$, which are also Hilbert spaces. Since $\cX \times \cZ$ are compact, continuous function $f$ and $g$ parameterized in the form of \eqref{eq:nn} are square-integrable, thus naturally belong to $L^{2}(\cW)$ and $L^{2}(\cZ)$. \vspace{5pt}

\noindent For a fixed $f$, $\cL_{1}(f,g)$ is a continuous linear functional in $g$ defined on $L^{2}(\cZ)$. Thus, there exists function $h_{f}$ in $L_{2}(\cZ)$ such that $\cL_{1}(f,g) = \inp[\big]{h_{f}}{g}_{L^2}$. Similarly, for a fixed $g$, $\cL_{1}(f,g)$ is a continuous linear functional in $f$, thus there exists function $h_{g}$ in $L_{2}(\cW)$ such that $\cL_{1}(f,g) = \inp[\big]{h_{g}}{f}_{L^2}$. In fact, $h_f$ and $h_g$ matches the variation of $\cL_1$ with respect to $g$ and $f$.
\begin{align*}
    h_f = \frac{\delta \cL_1(f,g)}{\delta g}, \quad h_f = \frac{\delta \cL_2(f,g)}{\delta f}.
\end{align*}
Since $\cL_{2}$ is a concave functional with respect to $g$, we apply Jensen's inequality and it holds that,
\begin{align}\label{eq:st11}
    \cL(f, g') - \cL(f, g) &= \cL_{1}\bigl(f, g' \bigr) - \cL_{1}\bigl(f, g \bigr) + \cL_{2}\bigl(f, g' \bigr) - \cL_{2}\bigl(f, g \bigr) \nonumber\\
    &\le  \inp[\Big]{\frac{\delta \cL_1(f,g)}{\delta g}}{g' - g}_{L^2} + \inp[\Big]{\frac{\delta \cL_2(f,g)}{g}}{g' - g}_{L^2} \nonumber \\
    & = \inp[\Big]{\frac{\delta \cL(f, g)}{\delta g}}{ g' - g}_{L^2}.
\end{align}
Follow a similar reasoning, using the fact that $\cL_{1}$ is a linear functional with respect to $f$ and $\cL_{2}$ is a convex functional with respect to $f$, it holds that
\begin{align}
    \label{eq:st12}
    \cL(f, g) - \cL(f', g) \le \inp[\Big]{\frac{\delta \cL(f, g)}{\delta f}}{ f - f'}_{L^2}.
\end{align}
Plugging \eqref{eq:st11} and \eqref{eq:st12} into \eqref{eq:st1}, we re-write the minimax expression in \eqref{eq:st1} using the variation of $\cL(f,g)$, the following inequality holds,
\begin{align}
\label{eq:st10}
\max_{g'} ~\cL(f, g') - \min_{f'} ~\cL(f', g) \le \max_{f', g'}~ \inp[\Big]{\frac{\delta \cL(f, g)}{\delta g}}{ g' - g}_{L^2} + \inp[\Big]{\frac{\delta \cL(f, g)}{\delta f}}{ f - f'}_{L^2}.
\end{align}
Thus, if $(f^*, g^*)$ is the stationary point, i.e.,
\begin{align}
\label{eq:st101}
    \frac{\delta \cL(f^*, g^*)}{\delta f} = \frac{\delta \cL(f^*, g^*)}{\delta g} = 0, \quad \text{a.s.},
\end{align}
then \eqref{eq:st10} suggests that for such stationary point $(f^*,g^*)$, for any function pair $(f',g')$, the following inequality holds,
\begin{align}
\label{eq:st13}
\max_{g'}~\cL(f^*, g') - \min_{f'}~\cL(f', g^*) \leq 0.
\end{align}
Equation \eqref{eq:st13} proves that $(f^*, g^*)$ is a saddle point for the minimx objective $\cL(f,g)$. Therefore, the stationarity of $(f^*, g^*)$ implies that it's a saddle point for objective $\cL(f,g)$.
\vspace{5pt}

\noindent Now, we proceed to show the second stage of the proof. We now show that if $ (\mu^*, \nu^*) $ is the stationary point of $\cL$, i.e., $v^f(\cdot; \mu^*, \nu^*) = v^g(\cdot; \mu^*, \nu^*) = 0$, the corresponding function pair $(f(\cdot; \mu^*), g(\cdot; \nu^*))$ is the stationary point of $\cL(f,g)$ with respect to $(f,g)$. We recall that the correspondence between $(\mu^*, \nu^*)$ and $(f(\cdot; \mu^*), g(\cdot; \nu^*))$ is through \eqref{eq:nn}. Let $(\mu^*, \nu^*)$ be a stationary point of \eqref{eq:minimax obj}, that is
\begin{align}
    \label{eq:st2}
    \nabla_{\theta} \frac{\delta \cL(\mu^*, \nu^*)}{\delta \mu }(\theta) = \nabla_{\ow} \frac{\delta \cL(\mu^*, \nu^*)}{\delta \nu }(\ow) = 0, \quad \forall \theta, \ow \in \RR^{D}
\end{align}
We can also compute the variation of $\cL(\mu, \nu)$ explicitly.
\begin{align*}
\frac{\delta \cL(\mu^*, \nu^*)}{\delta \mu}(\theta)
&= \EE_{\cD} \Bigl[ \alpha \Big\langle g(Z; \nu^*) \cdot  \frac{\delta \Phi(X,Z; f(\cdot; \mu^*))}{\delta f} + \lambda \cdot \frac{\delta \Psi(X,Z; f(\cdot; \mu^*))}{\delta f}, \phi(\cdot; \theta) \Big\rangle_{L^{2}} \Bigr], \nonumber \\
\frac{\delta \cL(\mu^*, \nu^*)}{\delta \nu}(\ow)  & = \EE_{\cD} \Bigl[ \alpha \big(\Phi(X,Z;f(\cdot, \mu^*)) - g(Z; \nu^*)\big) \cdot \psi(Z; \ow) \Bigr].
\end{align*}
By the oddness of $b$ in Assumption~\ref{asp:reg}, we have that $\phi(\cdot; \boldsymbol{0}) = 0$, This implies that the variation of $\cL(\mu^*, \nu^*)$ with respect to $\mu$ and $\nu$ are $0$ when $\theta = \ow = \boldsymbol{0}$, i.e., 
\begin{align*}
    \frac{\delta \cL(\mu^*, \nu^*)}{\delta \mu }(\boldsymbol{0}) = \frac{\delta \cL(\mu^*, \nu^*)}{\delta \nu } (\boldsymbol{0}) = 0.
\end{align*}
Combined with \eqref{eq:st2}, we deduced that
\begin{align*}
    \frac{\delta \cL(\mu^*, \nu^*)}{\delta \mu }(\theta ) = \frac{\delta \cL(\mu^*, \nu^*)}{\delta \nu } (w) = 0 \quad \forall \theta, \ow \in \RR^{D}.
\end{align*}
Note that we can expand the variation of $\mathcal{L}$ with respect to $\mu$,
\begin{align}\label{eq:st4}
    \alpha \inp[\Big]{\frac{\delta \cL(f(\cdot; \mu^*), g(\cdot; \nu^*))}{\delta f}}{\phi(\cdot; \theta)}_{L^2} = \frac{\delta \cL(\mu^*, \nu^*)}{\delta \mu}(\theta) = 0.
\end{align}
By the universal function approximation theorem (Lemma~\ref{lem:uat}), since $\frac{\cL(f(\cdot;\mu^*),g(\cdot;\nu^*))}{\delta f}$ is in $\sC(\cW)$ as is assumed in item (iv) of Assumption \ref{asp:compact}, there exists $\{\phi_n\}_{n=1}^\infty \in \cG(\phi) $ such that $\phi_n \rightarrow \frac{\cL(f(\cdot;\mu^*),g(\cdot;\nu^*))}{\delta f}$ uniformly. Here, $\cG(\phi)$ denotes the space of functions that are linearly spanned by $\phi(\cdot, \theta)$ By \eqref{eq:st4}, it holds that
\begin{align}
\label{eq:st5}
\inp[\Big]{\frac{\cL(f(\cdot;\mu^*),g(\cdot;\nu^*))}{\delta f}(\cdot)}{\phi_n(\cdot)}_{L^2} = 0.
\end{align}
Following a similar strategy, we can show that there exists $\{\psi_n\}_{n=1}^\infty \in \cG(\psi)$ such that $\psi_n \rightarrow \frac{\cL(f(\cdot;\mu^*),g(\cdot;\nu^*))}{\delta g}$, where for each $\psi_n$, it holds that
\begin{align}
\label{eq:st52}
\inp[\Big]{\frac{\cL(f(\cdot;\mu^*),g(\cdot;\nu^*))}{\delta g}(\cdot)}{\psi_n(\cdot)}_{L^2} = 0.
\end{align} 
We take the limit of \eqref{eq:st5} and \eqref{eq:st52} by passing $n \rightarrow \infty$ and conclude,
\begin{align}
\label{eq:st51}
\frac{\delta \cL(f(\cdot; \mu^*), g(\cdot; \nu^*))}{\delta f} = 0, \quad \frac{\cL(f(\cdot;\mu^*),g(\cdot;\nu^*))}{\delta g}(\cdot) = 0. \quad \text{a.s.}
\end{align}
Equation \eqref{eq:st51} proves that if $(\mu^*, \nu^*)$ is a stationary point of the Wasserstein gradient flow, then the associated function pair $(f(\cdot;\mu^*, g(\cdot; \nu^*)))$ is a stationary point of the minimax objective $\cL(f,g)$, which matches the conditions we conclude in \eqref{eq:st101}. Therefore, we prove that $(f(\cdot; \mu^*), g(\cdot; \nu^*))$ is a saddle point of the minimax objective $\cL(f,g)$. We complete the proof of item (i).

\vspace{4pt}
\noindent{\bf Proof of (ii).} We now show that there exists good solution pair $(\mu^*, \nu^*)$ that is both optimal as well as close to initialization $(\mu_0, \nu_0)$ in Wasserstein distance. 
\vspace{5pt}
By Assumption~\ref{asp:reliability}, there exists distribution $ \mu^\dagger, \nu^\dagger \in \sP_2(\RR^D)$ such that the optimal solution to the optimization problem \eqref{eq:minimax obj}$( f^*, g^*)$ satisfies the following,
\begin{align*}
f^*(w) = \int \phi(w; \theta) \rd \mu^\dagger(\theta), g^*(z) = \int \psi(z; \ow) \rd \nu^\dagger(\ow), \quad \forall w \in \cW, z \in \cZ
\end{align*}
Recall that $\alpha > 0$ is the scaling parameter in neural network parameterization. We can construct $(\mu^*, \nu^*)$ using a convex combination of $(\mu^\dagger, \nu^\dagger)$ and the initialization $(\mu_0, \nu_0)$,
\begin{align}
    \label{eq:rs1}
    \mu^*(\theta) = \alpha^{-1} \mu^\dagger(\theta) + (1- \alpha^{-1}) \mu_0(\theta), \quad \nu^*(w) = \alpha^{-1} \nu^\dagger(\ow) + (1-\alpha^{-1}) \nu_0(\ow).
\end{align}
We claim that $(\mu^*, \nu^*)$ constructed in \eqref{eq:rs1} satisfies all the desired requirements.
Since $ \mu_0, \nu_0 $ are standard Gaussian distribution, the integration of $\phi(\cdot; \theta)$ with respect to $\mu_0$ and $\psi(\cdot; \ow)$ with respect to $\nu_0$ are identically $0$ due to oddness of neuron functions,
\begin{align*}
    \int_{\cW} \phi(w;\theta) \rd \mu_0(\theta) = 0, \quad \int_{\cZ} \psi(z; \ow) \rd \nu_0(\ow) = 0. \quad \forall w \in \cW, z \in \cZ
\end{align*}
Thus, the expressions for $(f^*, g^*)$ are simplified to
\begin{align*}
f^*(w) = \alpha \int \phi(w; \theta) \rd \mu^*(\theta), \quad g^*(z) = \alpha \int \psi(z; \ow) \rd \nu^*(\ow).
\end{align*}
By Talagrand's inequality (Lemma \ref{lem:talagrand}), the following chain of inequalities holds,
\begin{align}
    \label{eq:rs3}
    \cW_2(\mu_0, \mu^*)^2 &\le 2 D_\kl(\mu^* \,\|\, \mu_0) \le D_{\chi^2}(\mu^*\,\|\,\mu_0) \nonumber\\
    & = \int  \biggl( \frac{\mu^*(\theta)}{\mu_0(\theta)} - 1 \biggr)^2 \,\rd \mu_0(\theta) = \int \biggl( \frac{(1-\alpha^{-1}) \cdot \mu_0(\theta) + \alpha^{-1}\cdot  \mu^\dagger(\theta)}{\mu_0(\theta)} - 1\biggr)^2 \,\rd\mu_0(\theta) \nonumber\\
    & = \alpha^{-2} D_{\chi^2}(\mu^\dagger \,\|\, \mu_0).
\end{align}
A similar bound on $\cW_{2}(\nu_0, \nu^*)$ also applies,
\begin{align}
    \label{eq:rs4}
    \cW_2(\nu_0, \nu^*)^2 \le \alpha^{-2} D_{\chi^2}(\nu^\dagger \,\|\, \nu_0).
\end{align}
Let $\bar D = \max\{D_{\chi^2}(\mu^\dagger \,\|\, \mu_0)^{1/2}, D_{\chi^2}(\nu^\dagger \,\|\, \nu_0)^{1/2}\}$, we conclude the proof of item (ii).

\subsection{Proof of Theorem \ref{th:main}} \label{sec:pf-th-main}

By Lemma \ref{lem:stat}, there exists distribution $(\mu^*, \nu^*)$ that is a stationary point of Wasserstein gradient flow \eqref{eq:pde} and simultaneously satisfying the distance bound in item (ii) of Lemma~\ref{lem:stat}. For such $(\mu^*, \nu^*)$, we denote $\rho^*(\theta, \ow) = \mu^*(\theta) \nu^*(\ow)$ as their product measure. Moreover, for any distribution pair $(\mu, \nu)$, we use $\rho(\theta, \ow) = \mu(\theta) \nu(\ow)$ as their product measure for simplicity. To rewrite the Wasserstein gradient flow for $(\mu, \nu)$ into the flow for $\rho$, we define vector the stacked vector field $v$ as,
\begin{align}
    \label{eq:dir2}
    v(\theta, \ow; \mu, \nu) = \bigl( v^f(\theta; \mu, \nu), v^g(\ow; \mu, \nu) \bigr).
\end{align}
  Following from Lemma \ref{lem:w2prod}, \eqref{eq:rs3}, and \eqref{eq:rs4}, it holds that $\cW_2(\rho^*, \rho_0) \le \alpha^{-1} \bar D$, where $\bar D$ is defined in Lemma~\ref{lem:stat}. Note that
\begin{align*}
	& f(w; \mu) = \alpha \int \phi(w; \theta) \mu(\theta) \rd \theta = \alpha \int \phi(w; \theta) \rho(\theta, \ow) \rd (\theta,\ow), \quad \forall w \in \cW, \nonumber\\
	& g(z; \nu) = \alpha \int \psi(z; \ow) \nu(\ow) \rd \ow= \alpha \int \psi(z; \ow) \rho(\theta, \ow) \rd (\theta, \ow), \quad \forall z \in \cZ.
\end{align*}
Thus, we overload the notation to write $ f(\cdot; \rho) = f(\cdot; \mu) $ and $ g(\cdot; \rho) = g(\cdot; \nu) $ for $ \rho \in \sP_{2}(\RR^{D} \times \RR^{D})$. By writing $\rho_t = (\mu_t, \nu_t)$, the update in \eqref{eq:pde} takes the following form
\begin{align*}
	\partial_t \rho_t(\theta, \ow) = -\dive\bigl(\rho_t(\theta, \ow) v(\theta, \ow; \rho_t) \bigr), \quad \rho_{0} = (\mu_0, \nu_0).
\end{align*}
Before we prove Theorem \ref{th:main}, we first show the following important technical lemma.
\begin{lemma}
\label{lem:main}
We assume $\cW_{2}(\rho_{t}, \rho^{*}) \leq 2\cW_{2}(\rho_{0}, \rho^{*})$. Under Assumptions~\ref{asp:reg}, \ref{asp:compact}, \ref{asp:reliability}, it holds that
\begin{align}
\label{eq:main lemma}
\frac{1}{2}\frac{\rd \cW_2(\rho_t, \rho^*)^2}{\rd t}  & \; \le - \eta \cdot \EE_{\cD} \Bigl[ \lambda \Psi\big(X,Z; f(\cdot; \mu_t) - f^*(\cdot)\big) + \bigl(g(Z; \nu_t) - g^*(Z)\bigr)^2  \Bigr]  + C_{*} \cdot \eta\alpha^{-1}.
\end{align}
where $C_{*} >0 $ is a constant depending on $B_0, B_1, B_2$, $\lambda$, and $\bar D$
\end{lemma}

\begin{proof}
Let $ \{\beta_s\}_{s\in [0, 1]}  $ be the geodesic connecting $ \rho_t $ and $ \rho^*$ with $\beta_0 = \rho_t$ and $\beta_1 = \rho^*$. Let $ u $ be the corresponding veclocity field such that $ \partial_s \beta_s = - \dive(\beta_s u_s) $.
By the first variation formula of Wasserstein distance in Lemma \ref{lem:diff}, it holds that
\begin{align}
	\label{eq:main1}
	\frac{1}{2} \frac{\rd \cW_2(\rho_t, \rho^*)^2}{\rd t} &= - \eta \inp[\big]{v(\cdot; \rho_t)}{ u_0}_{\rho_t} = - \eta \inp[\big]{v(\cdot;\rho^*)}{u_1}_{\rho^*} + \eta \int_0^1 \partial_s \inp[\big]{v(\cdot; \beta_s)}{u_s}_{\beta_s} \rd s  \\
	& = \eta \underbrace{\int_0^1 \inp[\big]{\partial_s v(\cdot; \beta_s)}{u_s}_{\beta_s}\rd s}_{\displaystyle\mathrm{(i)}} + \eta\underbrace{\int_{0}^{1} \int \inp[\big]{v(\theta, \ow; \beta_s)}{\partial_s(u_s(\theta, w)\beta_s(\theta, \ow))} \rd (\theta, \ow) \rd s}_{\displaystyle\mathrm{(ii)}} \nonumber .
\end{align}
where the notation $\inp[\big]{h_1}{h_2}_{\rho} = \int h_1 \cdot h_2\rd \rho$ for any distribution $\rho$ and functions $h_1, h_2$. We will provide bounds for term (i) and (ii) separately in the sequel. \vspace{5pt}

\noindent{\bf Upper bounding term (i).} 
For term (i) of \eqref{eq:main1}, by the definitions of \(v\), \(v^f\), and \(v^g\) in \eqref{eq:dir2} and \eqref{eq:dir}, we have that
\begin{align*}
\partial_s v^f(\theta, \ow; \beta_s) & = \alpha \partial_s \EE_{\cD} \Bigl[ -g(Z; \beta_s) \cdot \Big\langle \frac{\delta \Phi(X,Z; f(\cdot; \beta_s))}{\delta f}, \nabla_{\theta} \phi(\cdot; \theta) \Big\rangle_{L^{2}} - \lambda \cdot \Big\langle  
\frac{\delta \Psi(X,Z; f(\cdot; \beta_s))}{\delta f}, \nabla_{\theta} \phi(\cdot; \theta) \Big\rangle_{L^{2}} \Bigr] \nonumber \\
& = \alpha \nabla_{\theta} \EE_{\cD} \Bigl[ -g(Z; \partial_s \beta_s) \cdot \Big\langle \frac{\delta \Phi(X,Z; f(\cdot; \beta_s))}{\delta f}, \phi(\cdot; \theta) \Big\rangle_{L^{2}} - \lambda \cdot \Big\langle  
\frac{\delta \Psi(X,Z; f(\cdot; \partial_s \beta_s))}{\delta f}, \phi(\cdot; \theta) \Big\rangle_{L^{2}} \Bigr].
\end{align*}
where the second inequality holds since $\frac{\delta \Phi(X,Z;f)}{\delta f}$ a constant, $s-$independent function, $\frac{\delta \Psi(X, Z; f)}{\delta f}$ is linear in $f$, and $\partial_s f(\cdot; \beta_s), \partial_s g(\cdot; \beta_s)$ satisfies
\begin{align*}
& \partial_s f(w; \beta_s) = \int \partial_s \big(\phi(w; \theta) \beta_s(\theta, \ow)\big) \rd (\theta, \ow) = \int \phi(w; \theta) \partial_s \beta_s \rd (\theta, \ow) = f(w; \partial_s \beta_s), \quad \forall w \in \cW\\
& \partial_s g(z; \beta_s) = \int \partial_s \big(\psi(z; \ow) \beta_s(\theta, \ow)\big) \rd (\theta, \ow) = \int \psi(z; \ow) \partial_s \beta_s \rd (\theta, \ow) = g(z; \partial_s \beta_s), \quad \forall z \in \cZ
\end{align*}
A similar computation for $\partial_s v^{g}(\theta, \ow; \beta_s)$ gives
\begin{align*}
\partial_s v^{g}(\theta, \ow;\beta_s) = \alpha \nabla_{\ow} \EE_{\cD} \Bigl[ \tilde \Phi(X,Z;f(\cdot, \partial_s \beta_s)) \cdot \phi(Z; \ow)  - g(Z; \partial_s \beta_s) \cdot \phi(Z;\ow) \Bigr]
\end{align*}
We recall that $\tilde \Phi(x,z;f) = \Phi(x,z;f) - \Phi(x,z;\boldsymbol{0})$ is the linear component in $\Phi$. We note that the variation of $\tilde \Phi$ is the same as the variation of $\Phi$ with respect to $f$, $\frac{\delta \Phi(X,Z; f)}{ \delta f} = \frac{\delta \tilde \Phi(X,Z; f)}{ \delta f}.$ \vspace{5pt}

\noindent We define the potential $\cV(\theta, \ow; \partial_s \beta_s)$ as 
\begin{align*}
    \cV(\theta, \ow; \partial_s \beta_s) & = \EE_{\cD} \Bigl[ g(Z; \partial_s \beta_s) \cdot \Big\langle \frac{\delta \Phi(X,Z; f(\cdot; \beta_s))}{\delta f}, \phi(\cdot; \theta) \Big\rangle_{L^{2}} + \lambda \cdot \Big\langle  
\frac{\delta \Psi(X,Z; f(\cdot; \partial_s \beta_s))}{\delta f}, \phi(\cdot; \theta) \Big\rangle_{L^{2}} \Bigr] \\
\quad & - \EE_{\cD}\Bigl[ \tilde \Phi(X,Z;f(\cdot, \partial_s \beta_s)) \cdot \psi(Z; \ow)  - g(Z; \partial_s \beta_s) \cdot \psi(Z;\ow) \Bigr]
\end{align*}
Then, the vector field $\partial_s v(\theta, \ow; \beta_s)$ is the gradient of such potential $\cV(\theta, \ow;\partial_s \beta_s)$
\begin{align*}
\partial_s v(\theta, \ow; \beta_s) = \begin{pmatrix}
    &\partial_s v^f(\theta; \beta_s) \\
    & \partial_s v^g(\ow; \beta_s)
\end{pmatrix}
= -\alpha \nabla \cV(\theta, w; \partial_s \beta_s), 
\end{align*}
where the gradient operator $\nabla = (\nabla_{\theta}, \nabla_{\ow})$. Then, by Stoke's formula and integration by parts, we have
\begin{align*}
\inp[\big]{\partial_s v(\cdot; \beta_s)}{u_s}_{\beta_s} &  = -\int \alpha \nabla \cV(\theta, \ow; \partial_s \beta_s) u_s(\theta, w)\beta_s(\theta, w) \rd (\theta, w) \nonumber \\
& \quad = \int \alpha \cV(\theta, w; \partial_s \beta_s) \dive(u_s\beta_s) \rd (\theta, w) = -\int \alpha \cV(\theta, w; \partial_s \beta_s) \partial_s\beta_s \rd (\theta, w)
\end{align*}
Integrating potential $\cV$ with respect to $\partial_s \beta_s$ simplied the expression to
\begin{align}
\label{eq:main4}
& \int \alpha \cV(\theta, \ow; \partial_s \beta_s) \partial_s\beta_s \rd (\theta, \ow) \nonumber \\ 
& \qquad = \EE_{\cD} \Big[ g(Z; \partial_s \beta_s) \cdot \Big\langle \frac{\delta \Phi(X,Z; f(\cdot; \beta_s))}{\delta f}, \int \alpha \phi(\cdot; \theta) \partial_s \beta_s (\rd \theta) \Big\rangle_{L^{2}} \Big]\nonumber \\
& \qquad \qquad + \EE_{\cD} \Big[ \lambda \cdot \big\langle  
\frac{\delta \Psi(X,Z; f(\cdot; \partial_s \beta_s))}{\delta f}, \int \alpha \phi(\cdot; \theta) \partial_s \beta_s(\rd \theta) \big\rangle_{L^{2}} \Big] \nonumber \\
& \qquad \qquad - \EE_{\cD}\Bigl[ \tilde \Phi(X,Z;f(\cdot, \partial_s \beta_s)) \cdot \int \alpha \psi(Z; \ow) \partial_s \beta_s(\rd \ow) - g(Z; \partial_s \beta_s) \cdot \int \alpha \psi(Z; \ow) \partial_s \beta_s(\rd \ow) \Bigr]  \nonumber \\
& \qquad = \EE_{\cD} \Bigl[\lambda \cdot \Big\langle  
\frac{\delta \Psi(X,Z; f(\cdot; \partial_s \beta_s))}{\delta f}, f(\cdot; \partial_s \beta_s) \Big\rangle_{L^{2}} + g(Z; \partial_s \beta_s)^{2}\Bigr].
\end{align}
By convexity of $\Psi(x, z; f)$ and $\Psi(x, z; \boldsymbol{0}) = 0$ for all $(x,z) \in \cX \times \cZ$, it holds that
\begin{align}
\label{eq:psi-convex}
    \Psi(x, z; f(\cdot; \partial_s \beta_s)) \leq \inp[\Big]{\frac{\delta \Psi(x,z; f(\cdot; \partial_s \beta_s))}{\delta f}}{f(\cdot; \partial_s \beta_s)}_{L^2}, \quad \forall (x,z) \in \cX \times \cZ.
\end{align}
Integrating \eqref{eq:main4} with respect to \(s \in [0, 1]\), we have that
\begin{align} \label{eq:ma1}
	\int_0^1 \inp[\big]{\partial_s v(\cdot; \beta_s)}{u_s}_{\beta_s}\rd s 
	& = -\int_0^1 \EE_{\cD} \Bigl[\lambda \cdot \big\langle  
\frac{\delta \Psi(X,Z; f(\cdot; \partial_s \beta_s))}{\delta f}, f(\cdot; \partial_s \beta_s) \big\rangle_{L^{2}} + g(Z; \partial_s \beta_s)^{2}\Bigr] \rd s \nonumber \\
    & \leq -\EE_{\cD} \Bigl[\lambda \cdot \Psi(X, Z; f(\cdot; \partial_s \beta_s)) + g(Z; \partial_s \beta_s)^{2}\Bigr] \rd s \nonumber \\
    & \leq -\EE_{\cD} \Bigl[\lambda \cdot \Psi(X, Z; f(\cdot; \rho_t) - f(\cdot; \rho^*)) + \Big(g(Z; \rho_t) - g(Z; \rho^*)\Big)^{2}\Bigr] \rd s \nonumber \\
	& = -\EE_{\cD} \Bigl[ \lambda \cdot \Psi(X, Z; f(\cdot; \rho_t) - f^*(\cdot)) + \bigl(g(Z; \rho_t) - g^*(Z) \bigr)^2  \Bigr].
\end{align}
where the first inequality holds due to \eqref{eq:psi-convex}, and the second holds by Jensen's inequality. \vspace{5pt}

\noindent{\bf Upper bounding term (ii).} By Lemma \ref{lem:euler}, for term (ii) in \eqref{eq:main1}, it holds that
\begin{align} \label{eq:ma2}
&\int \inp[\big]{v(\theta, \ow; \beta_s)}{\partial_s(u_s(\theta, \ow)\beta_s(\theta, \ow))} \rd (\theta, \ow) \nonumber \\
& \qquad = \int \inp[\big]{\nabla v(\theta, \ow; \beta_s)}{u_s(\theta, \ow)\otimes u_s(\theta, \ow)\beta_s(\theta, \ow)} \rd (\theta, \ow) \nonumber\\
&\qquad \le \sup_{\theta, \ow} ~ \norm[\bigl]{\nabla v(\theta, \ow; \beta_s)}_{F} \cdot \norm{u_s}_{\beta_s}^2 .
\end{align}
where $ \norm[\bigl]{\cdot}_{F}$ denotes the Frobenius norm. Since $u_s$ is the velocity field corresponding to the geodesic connecting $\rho^* $, by assumptions, it holds that
\begin{align}
	\label{eq:ma21}
	\norm{u_s}_{\beta_s}^2 = \cW_2(\rho_t, \rho^*)^2 \leq 4 \cW_{2}(\rho_{0}, \rho^{*})^{2} = 4 \alpha^{-2} \bar D^{2} = \cO(\alpha^{-2})
\end{align}
On the other hand, by the definition of $ v $ in \eqref{eq:dir2}, we have that
\begin{align}\label{eq:ma22}
	\norm[\bigl]{\nabla v(\theta, \ow; \beta_s)}_{F}^2 = \norm[\bigl]{\nabla_\theta v^f(\theta; \beta_s)}_{F}^2 + \norm[\bigl]{\nabla_\ow v^g(\ow; \beta_s)}_{F}^2
\end{align}
By the definition of $ v^f $ in \eqref{eq:dir}, we have that
\begin{align}
\label{eq:maa1}
\norm[\big]{\nabla_\theta v^f(\theta; \beta_s)}_{F} & \le \alpha \cdot  \EE_{\cD}\Big[ \Big| g(Z; \beta_s) \cdot \int_{\cW} \frac{\delta \Phi(X,Z; f(\cdot; \beta_s))}{\delta f}(w') \rd w' \Big| \Big] \cdot \underset{w \in \cW}{ \sup}\norm[\big]{\nabla^{2}_{\theta,\theta} \phi(w; \theta)}_{F}^2\nonumber \\ 
& \qquad + \alpha \cdot \EE_{\cD} \Big[\lambda \cdot \Big| \int_{\cW} \frac{\delta \Psi(X,Z;f(\cdot; \beta_s))}{\delta f}(w') \rd w' \Big| \Big] \cdot \underset{w \in \cW}{ \sup}\norm[\big]{\nabla^{2}_{\theta,\theta} \phi(w; \theta)}_{F}^2 \nonumber \\
& \le \alpha B_2 \cdot \EE_{\cD} \Bigl[  \lambda C_{\Psi} \bigl| f(W ; \beta_s) \bigr| + C_{2} \big| g(Z; \beta_s) \big| \Bigr].
\end{align}
where the first inequality follows from Assumption \ref{asp:reg}, and second inequality comes from the integrability conditions in Assumption \ref{asp:compact}. Thus, it suffices to upper bound $ f(w; \beta_s)$ and $ g(z; \beta_s) $ for all $(w,z) \in \cW \times \cZ$. For $f(w; \beta_s) $, we have that
\begin{align}
	\label{eq:maa2}
	\bigl| f(w; \beta_s) \bigr| &= \alpha \cdot \Bigl|\int \phi(w \theta) \rd \beta_s(\theta, \ow) \Bigr| = \alpha \cdot \Bigl|\int \phi(w; \theta) ~\rd (\beta_s - \rho_0)(\theta, \ow) \Bigr| \nonumber \\
	& \le \alpha B_{1} \cdot \cW_1(\beta_s, \rho_0) \le \alpha B_{1} \cdot \cW_2(\beta_s, \rho_0).
\end{align}
Moreover, it holds that
\begin{align}
	\label{eq:maa3}
	\cW_2(\beta_s, \rho_0) \le \cW_2(\beta_s, \rho^*) + \cW_2(\rho^*, \rho_0) \le \cW_2(\rho_t, \rho^*) + \cW_2(\rho_0, \rho^*) \le 3 \alpha^{-1} \bar D,
\end{align}
where the second inequality follows from the fact that $\beta_s, s\in [0, 1]$ is the geodesic connecting $\rho_t$ and $\rho^*$ and the last inequality follows from (ii) in Lemma~\ref{lem:stat}.
Plugging \eqref{eq:maa3} into \eqref{eq:maa2}, we have that
\begin{align}
\label{eq:maa31}
\bigl| f(w; \beta_s) \bigr| \le \cO(1), \quad \forall w \in \cW.
\end{align}
Through a similar argument, such an upper bound can also be established for $g(z; \beta_s)$ for all $z \in \cZ$,
\begin{align}
\label{eq:maa32}
\bigl| g(z; \beta_s) \bigr| \le \cO(1), \quad z \in \cZ.
\end{align}
Plugging \eqref{eq:maa31} and \eqref{eq:maa32} into \eqref{eq:maa1}, we establish an upper bound for $\norm[\big]{\nabla_\theta v^f(\theta; \beta_s)}_{F}$,
\begin{align}
\label{eq:maa4}
\norm[\big]{\nabla_\theta v^f(\theta; \beta_s)}_{F} \le \cO(\alpha).
\end{align}
Similarly, by the definition of  $ v^g $ in \eqref{eq:dir} we have that
\begin{align}
\label{eq:maa5}
\norm[\big]{\nabla_\ow v^g(\ow; \beta_s)}_{F} &\le \alpha \cdot  \EE_{\cD}\Bigl[ \big| \Phi(X,Z; f(\cdot; \beta_s))\big| + \big| g(Z; \beta_s)\big| \Bigr] \cdot \underset{z \in \cZ}{ \sup}\norm[\big]{\nabla^{2}_{\ow,\ow} \psi(z; \ow)}_{F}^2 \nonumber \\
& \le \alpha B_2 \cdot \Big( \EE_{\cD} \Bigl[ \big|\Phi(X,Z; \boldsymbol{0})\big| + C_{2} \bigl|f(W; \beta_s) \bigr|  + \big| g(Z; \beta_s)\big| \Big) = \cO(\alpha).
\end{align}
Combining the bound from \eqref{eq:maa4} and \eqref{eq:maa5} and plugging into\eqref{eq:ma22}, it holds that
\begin{align}
	\label{eq:maa6}
	\norm[\Big]{\nabla v(\theta, \ow; \beta_s)}_{F}^{2} =  \norm[\Big]{\nabla_{\theta} v^{f}(\theta; \beta_s)}_{F}^{2} + \norm[\Big]{\nabla_{\ow} v^{g}(\ow; \beta_s)}_{F}^{2}  \le \cO(\alpha^2).
\end{align}
Equation \eqref{eq:ma21} and \eqref{eq:maa6} provide upper bounds on the two terms involved in \eqref{eq:ma2}. Plugging the upper bounds that we have achieved, it holds that
\begin{align} \label{eq:ma20}
	\int \inp[\Big]{v(\theta, w; \beta_s)}{\partial_s(u_s(\theta, \ow)\beta_s(\theta, \ow))} \rd (\theta, \ow) \le \cO(\alpha^{-1}).
\end{align}
Now combining \eqref{eq:ma1} and \eqref{eq:ma20}, we have that
\begin{align*}
\frac{1}{2}\frac{\rd \cW_2(\rho_t, \rho^*)^2}{\rd t}  & \; \le - \eta \cdot \EE_{\cD} \Bigl[ \lambda \Psi(X, Z; f(\cdot, \rho_t) - f(\cdot; \rho^*)) + \bigl(g(Z; \rho_t) - g(Z; \rho^*)\bigr)^2  \Bigr] + C_{*} \cdot \eta\cdot \alpha^{-1}.
\end{align*}
where $C_{*} = C_{*}\bigl(B_0, B_{1}, B_2, C, \lambda, \bar D \bigr) > 0$ is a constant. This completes the proof of Lemma~\ref{lem:main}. 
\end{proof} 
\vspace{0.5cm}
\noindent We are now ready to present the proof of Theorem~\ref{th:main} with the help of Lemma \ref{lem:main}.
\begin{proof}
We define
\begin{align}
\label{eq:ma4}
    t^{*} = \inf \Bigl\{ \tau \in \RR_{+} \bigggiven \EE_{\cD} \bigl[ \lambda \Psi(X, Z; f(\cdot, \rho_\tau) - f(\cdot; \rho^*)) + \bigl(g(Z; \rho_\tau) - g(Z; \rho^*)\bigr)^2 \bigr] < C_{*} \cdot \alpha^{-1} \Bigr\}
\end{align}
Also, we define
\begin{align}
\label{eq:ma5}
t_{*} = \inf \Bigl\{ \tau \in \RR_{+} \bigggiven \cW_{2}(\rho_{\tau},\rho^{*}) > 2 \cW_{2}(\rho_{0}, \rho^{*}) \Bigr\}
\end{align}
In other words, \eqref{eq:main lemma} of Lemma \ref{lem:main} holds for $t \leq t_{*}$, and for $0 \leq t \leq \min\{t_{*}, t^{*}\}$, we have
\begin{align*}
 \frac{1}{2}\frac{\rd \cW_2(\rho_t, \rho^*)^2}{\rd t}  \le - \eta \cdot \EE_{\cD} \Bigl[  \lambda \Psi(X, Z; f(\cdot, \rho_t) - f(\cdot; \rho^*)) + \bigl(g(Z; \rho_t) - g(Z; \rho^*)\bigr)^2   \Bigr] + C_{*} \cdot \eta\alpha^{-1} \leq 0
\end{align*}
We now show that $t_{*} > t^{*}$ by contradiction. By the continuity of $\cW_{2}(\rho_{t}, \rho^{*})^{2}$ with respect to $t$ \cite{ambrosio2008gradient}, since $\cW_{2}(\rho_{0}, \rho^{*}) < 2 \cW_{2}(\rho_{0}, \rho^{*})$, it holds that $t_{*} > 0$. Let's assume $t_{*} \leq t^{*}$, then $t_{*} = \min\{t_{*}, t^{*}\}$. Thus, by \eqref{eq:main lemma}, \eqref{eq:ma4}, \eqref{eq:ma5}, it holds that for $0 \leq t \leq t_{*}$ that
\begin{align*}
\frac{1}{2}\frac{\rd \cW_2(\rho_t, \rho^*)^2}{\rd t} \leq 0
\end{align*}
which further implies that $\cW_{2}(\rho_{t_{*}}, \rho^{*}) \leq \cW_{2}(\rho_{0}, \rho^{*})$. This contradicts the definition of $t_{*}$ in \eqref{eq:ma5}. Thus, it holds that $t_{*} \geq t^{*}$, which implies that \eqref{eq:main lemma} of Lemma~\ref{lem:main} holds for any $0 \leq t \leq t^{*}$.
We now discuss two different situations.  \vspace{5pt}

\noindent \textbf{Scenario (i)} If $t_{*} \leq T$, then it holds that
\begin{align}
\label{eq:ma8}
& \inf_{t \in [0,T]} \EE_{\cD} \Bigl[ \lambda \Psi(X, Z; f(\cdot, \mu_t) - f^*) + \bigl(g(Z; \nu_t) - g^* \bigr)^2  \Bigr] \nonumber \\
& \quad \leq \EE_{\cD} \Bigl[ \lambda \Psi(X, Z; f(\cdot, \mu_{t_*}) - f^*) + \bigl(g(Z; \nu_{t_*}) - g^* \bigr)^2  \Bigr] \nonumber \\
& \quad < C_{*} \alpha^{-1} = \cO(T^{-1} + \alpha ^{-1}).
\end{align}
Therefore, \eqref{eq:ma8} implies Theorem \ref{th:main} in this scenario. \vspace{5pt}

\noindent \textbf{Scenario (ii)} If $t_{*} > T$, then \eqref{eq:main lemma} in Lemma~\ref{lem:main} holds for $0 \leq t \leq T$. Re-arranging the terms, we have the following inequality for all $0 \leq t \leq T$,
\begin{align}
\label{eq:ma9}
\EE_{\cD} \Bigl[ \lambda \Psi(X, Z; f(\cdot, \mu_t) - f^*) + \bigl(g(Z; \nu_t) - g^* \bigr)^2  \Bigr] \leq - \eta^{-1} \cdot \frac{1}{2}\frac{\rd \cW_2(\rho_t, \rho^*)^2}{\rd t} + C_{*} \cdot \alpha^{-1}
\end{align}
This further suggests the following upper bound,
\begin{align}
\label{eq:ma10}
& \inf_{t \in [0,T]} \EE_{\cD} \Bigl[ \lambda \Psi(X, Z; f(\cdot, \mu_t) - f^*) + \bigl(g(Z; \nu_t) - g^* \bigr)^2  \Bigr] \nonumber \\
& \quad \leq T^{-1} \cdot \int_{0}^{T} \EE_{\cD} \Bigl[ \lambda \Psi(X, Z; f(\cdot, \mu_t) - f^*) + \bigl(g(Z; \nu_t) - g^* \bigr)^2 \Bigr] \mathrm{d} t\nonumber \\
& \quad \leq 1/2 \cdot \eta^{-1} \cdot T^{-1} \cdot \cW_{2}(\rho_{0}, \rho^{*})^2 + C_{*} \cdot \alpha^{-1} \nonumber \\
& \quad \leq 1/2 \cdot \alpha^{-2} \cdot \bar D^{2} \cdot \eta^{-1} \cdot T^{-1} + C_{*} \cdot \alpha^{-1} = \cO(T^{-1} + \alpha^{-1}),
\end{align}
where the second inequality comes from integrating \eqref{eq:ma9} in for $t \in [0,T]$, the third inequality comes from (ii) in Lemma~\ref{lem:stat} and last equality comes from setting $\eta$ to $\alpha^{-2}$. Therefore, \eqref{eq:ma10} implies Theorem \ref{th:main} in this scenario. \vspace{5pt}

\noindent Based on the discussion of scenarios (i) and (ii) above, we finish the proof of Theorem \ref{th:main}.
\end{proof}

\subsection{Proof of Theorem \ref{th:global}} \label{sec:pf-th-global}
\begin{proof}
We now prove Theorem \ref{th:global}. For notation simplicity, we denote $f_t = f(\cdot; \mu_t)$ as the estimator at time $t$. Recall the definition of $J(f)$ from \eqref{eq:mse} and $\bar \delta (z;f)$ from \eqref{eq:cmrerr}.
\begin{align*}
J(f) = \EE_{\cD} \bigl[ 1/2 \cdot \bar\delta(Z; f)^2 + \lambda \cdot \Psi(X,Z;f) \bigr], \quad  \bar \delta (z; f) = \EE_{X|Z} \bigl[\Phi(X, Z; f) \biggiven Z = z \bigr].
\end{align*}
Plugging the definition of $J(f)$, it holds that 
\begin{align}
\label{eq:global-diff}
& \inf_{t \in [0,T]} J(f_{t}) - J(f^{*}) \nonumber \\
& \qquad = \inf_{t \in [0,T]} \EE_{\cD} \Big[ 1/2 \cdot \Big(\bar \delta(Z, f_{t})^{2} - \bar \delta (Z, f^{*})^{2}\Big) + \lambda \Big(\Psi(X,Z;f_t) - \Psi(X,Z;f^*) \Big) \Big].
\end{align}
Similar to the proof of Theorem \ref{th:main}, we define $t_{*}$ as,
\begin{align*}
t_{*} = \inf \Bigl\{ \tau \in \RR_{+} \bigggiven \cW_{2}(\rho_{\tau},\rho^{*}) > 2 \cW_{2}(\rho_{0}, \rho^{*}) \Bigr\}.
\end{align*} 
We will upper-bound the term in \eqref{eq:global-diff} separately in two different scenarios, depending on the value of $t_{*}$ compared with $T$.

\noindent \textbf{Scenario (i)} If $t_{*} \leq T$, then we have that
\begin{align}
\label{eq:global-eq1}
\inf_{t \in [0,T]} J(f_{t}) - J(f^{*}) \leq J(f_{t_{*}}) - J(f^{*}).
\end{align}
In order to upper-bound right-hand side of \eqref{eq:global-eq1}, we need to uniformly upper-bound $f_{t_{*}}(w)$ and $f^{*}(w)$ for all $w \in \cW$. For $f_{t_{*}}(w) = f(w; \mu_{t_{*}})$, we have that
\begin{align}
\label{eq:global-eq2}
\underset{w \in \cW}{\sup} |f(w; \mu_{t_{*}})| & \; = \alpha \cdot \underset{w \in \cW}{\sup} \Bigl| \int \phi(w;\theta) \mathrm{d} \mu_{t_{*}}(\theta) \Bigr| = \alpha \cdot \underset{w \in \cW}{\sup} \Bigl| \int \phi(w;\theta) \mathrm{d}(\mu_{t_{*}}-\mu_{0}) (\theta) \Bigr| \nonumber \\
& \; \leq \alpha B_{1} \cdot \cW_{1}(\mu_{t_{*}}, \mu_{0}) \leq \alpha B_{1} \cdot \cW_{2}(\mu_{t_{*}}, \mu_{0})\; \leq \alpha B_{1} \cdot \Bigl( \cW_{2}(\rho_{t_{*}}, \rho^{*}) + \cW_{2}(\rho_{0}, \rho^{*})\Bigr) \nonumber \\ 
& \; \leq 3B_{1} \cdot \bar D = \cO(1).
\end{align}
where the first inequality follows from Lemma \ref{lem:dual-w1}, the second inequality follows from Lemma \ref{lem:w2prod}. The last inequality follows from (ii) in Lemma \eqref{lem:stat} and definition of $t_{*}$. For $f^{*}$, a similar chain of inequalities would apply,
\begin{align}
\label{eq:global-eq3}
\underset{w \in \cW}{\sup} |f(w; \mu^{*})| & \; = \alpha \cdot \underset{w \in \cW}{\sup} \Bigl| \int \phi(w;\theta) \mathrm{d} \mu^{*}(\theta) \Bigr| = \alpha \cdot \underset{w \in \cW}{\sup} \Bigl| \int \phi(w;\theta) \mathrm{d}(\mu^{*} - \mu_{0}) (\theta) \Bigr| \nonumber \\
& \; \leq \alpha B_{1} \cdot \cW_{1}(\mu^{*}, \mu_{0}) \leq \alpha B_1 \cdot \cW_{2}(\mu^{*}, \mu_{0}) \; \leq \alpha B_{1} \cdot \cW_{2}(\rho^{*}, \rho_{0}) \nonumber \\
& \; \leq B_{1} \cdot \bar D = \cO(1).
\end{align}
With uniform bounds on $f_{t_{*}}$ and $f^*$, we are now ready to upper-bound $\inf_{t \in [0,T]} J(f_{t}) - J(f^{*})$ through upper-bounding $J(f_{t_{*}}) - J(f^{*})$,
\begin{align}
\label{eq:global-eq4}
J(f_{t_{*}}) - J(f_{*}) & \; \leq \EE_{\cD} \Big[ \bar \delta(Z; f_{t_{*}}) \cdot \EE_{X|Z} \big[ \tilde \Phi(X,Z; f_{t_{*}} - f^*) | Z \big] + \lambda \cdot \inp[\Big]{\frac{\delta \Psi(X,Z; f_{t_{*}})}{\delta f}}{f_{t_{*}} - f^{*}}_{L^2} \Big] \nonumber \\
& \; \leq \Big(\underset{x,z}{\sup} |\Phi(x,z; 0)| + C_{\Phi} \cdot \underset{w \in \cW}{\sup}|f(w; \mu_{t_{*}})| \Big) \cdot \EE_{\cD}\Big[ C_{\Phi} \cdot |f(W; \mu_{t_{*}}) - f(W; \mu^*)| \Big] \nonumber \\
& \qquad +  \lambda C_{\Psi} \cdot \EE_{\cD}\Big[ |f(W; \mu_{t_{*}}) - f(W; \mu^*)| \Big] \nonumber \\
& \; \leq  B_{*} \cdot \EE_{\cD} \Big[|f(W; \mu_{t_{*}}) - f(W; \mu^*)| \Big] \leq B_{*} \cdot \Big(\EE_{\cD} \Big[\lambda |f(W; \mu_{t_{*}}) - f(W; \mu^*)|^{2} \Big]\Big)^{1/2} \nonumber \\
& \; \leq B_{*} \cdot \Big(\EE_{\cD}\Big[\Psi(X,Z; f_{t_{*}} - f^*)\Big]\Big)^{1/2} \leq B_{*} \cdot \alpha^{-1/2},
\end{align}
where $B_{*} = B_{*}(\Phi, c_{\phi}, C_{\Phi}, C_{\Psi}, \lambda, C, B_{1}, \bar D, C_{*}) > 0 $ is a constant and its values changes from line to line. The second inequality follows from \eqref{eq:global-eq2} and \eqref{eq:global-eq3}. The last inequality follows from \eqref{eq:ma8} in the proof of Theorem \eqref{th:main}. Therefore, in this scenario, we have that
\begin{align}
\label{eq:global-additional}
\inf_{t \in [0,T]} J(f_{t}) - J(f_{*}) \leq J(f_{t_{*}}) - J(f_{*}) \leq \cO(T^{-1/2} + \alpha^{-1/2}).
\end{align}
Equation \eqref{eq:global-additional} concludes the proof of Theorem \ref{th:global} in the scenario of $t_{*} \leq T$. \vspace{5pt}

\noindent \textbf{Scenario (ii)} If $t_{*} > T$, by definition of $t_{*}$, we have that
\begin{align*}
% \label{eq:global-eq5}
\cW_{2}(\mu_{t}, \mu^{*}) \leq \cW_{2}(\rho_{t}, \rho^{*}) \leq 2\cW_{2}(\rho_{0}, \rho^{*}) = 2\alpha \cdot \bar D, \quad \forall \; t \in [0,T].
\end{align*}
Following the same arguments in \eqref{eq:global-eq2} and \eqref{eq:global-eq3}, we have a uniform upper-bound for $f_{t}$ for all $t \in [0,T]$ and $f^*$ that writes,
\begin{align*}
\underset{w \in \cW}{\sup} \bigl| f(w; \mu_{t})\bigr| + \bigl| f(w; \mu^{*})\bigr| \leq 4 B_{1} \cdot \bar D = \cO(1), \quad \forall \; t \in [0,T].
\end{align*}
Following the same derivation of \eqref{eq:global-eq4}, we have that
\begin{align}
\label{eq:global-eq7}
\inf_{t \in [0,T]} J(f_{t}) - J(f_{*}) & \; \leq B_{*} \cdot \inf_{t \in [0,T]} B_{*} \cdot \Big(\EE_{\cD}\Big[\Psi(X,Z; f_{t} - f^*)\Big]\Big)^{1/2} \nonumber \\
& \; \leq B_{*} \cdot \Big(\inf_{t \in [0,T]} \EE_{\cD}\Big[\Psi(X,Z; f_{t} - f^*)\Big]\Big)^{1/2} \nonumber \\
& \; \leq B_{*} \cdot \Bigl( T^{-1} \cdot \int_{0}^{T} \EE_{\cD}\Big[\Psi(X,Z; f_{t} - f^*)\Big] \mathrm{d} t \Bigr)^{1/2} \nonumber \\
& \; \leq B_{*} \cdot \sqrt{1/2 \cdot \bar D^{2} \cdot T^{-1} + C_{*} \cdot \alpha^{-1}} = \cO(T^{-1/2} + \alpha^{-1/2}),
\end{align}
where the last inequality follows from \eqref{eq:ma9} and \eqref{eq:ma10} in the proof of Theorem \ref{th:main}. Equation \eqref{eq:global-eq7} concludes the proof of Theorem \eqref{th:global} in the scenario of $t_{*} > T$. \vspace{5pt}

\noindent Based on the discussion of scenarios (i) and (ii) above, we finish the proof of Theorem \ref{th:global}.
\end{proof}
\section{Mean Field Limit of Neural Networks} \label{sec:pf-prop-weak}
In this section, we prove Proposition \ref{prop:weak-formal-formal}. The formal version is presented as follows. Let $\rho_{t}(\theta, \ow) = \mu_t(\theta) \otimes \nu_{t}(\ow)$, where $(\mu_{t}, \nu_{t})$ is the PDE solution in \eqref{eq:pde} and $\hat \rho_{k}(\theta, \ow) = N^{-1} \cdot \sum_{i = 1}^{N} \delta_{\theta_{k}^{i}}(\theta) \cdot \delta_{\ow_{k}^{i}}(\ow)$ is the empirical distribution of $(\btheta_{k}, \bow_{k})$. Here we omit the dependence of the empirical distribution $\hat \rho_{k}$ on $N$ and stepsize scale $\epsilon$ for notational simplicity.

\begin{proposition}[Formal Version of Proposition \ref{prop:weak-formal-formal}] \label{prop:weak-formal}
Let $h: \RR^{D} \times \RR^{D} \rightarrow \RR$ by any continuous function such that $\|h\|_{\infty} \leq 1$ and $\operatorname{Lip}(h) \leq 1$. Under Assumption \ref{asp:reg}, \ref{asp:compact}, with probability at least $1 - 5\delta$, it holds that 
\begin{align*}
&\sup _{\substack{k \leq T / \epsilon \\
(k \in \mathbb{N})}}\left|\int h(\theta, w) \mathrm{d} \rho_{k \epsilon}(\theta, w)-\int h(\theta, w) \mathrm{d} \widehat{\rho}_k(\theta, w)\right| \leq B \cdot e^{B T} \cdot\Bigl(\sqrt{\log (N / \delta) / N}+\sqrt{\epsilon \cdot(D+\log (N / \delta))}\Bigr).
\end{align*}
Here $B$ is a constant that depends on $\alpha, \eta, \lambda, B_0, B_1$ and $B_2$.
\end{proposition}
\noindent The proof of Proposition \ref{prop:weak-formal} based heavily on \cite{mei2018mean, mei2019mean, araujo2019mean, zhang2020can}, which make use of the propagation of chaos arguments in \cite{sznitman1991topics}. Recall that $(v^{f}(\cdot; \rho), v^{g}(\cdot;\rho))$ is the a vector field defined as,
\begin{align}
\label{eq:dir-restate}
v^f(\theta; \rho) &= \alpha\EE_{\cD} \Bigl[ -g(Z; \rho) \cdot \Big\langle \frac{\delta \Phi(X,Z; f(\cdot; \rho))}{\delta f}, \nabla_{\theta} \phi(\cdot; \theta) \Big\rangle_{L^{2}} - \lambda \cdot \Big\langle \frac{\delta \Psi(X,Z; f(\cdot; \rho))}{\delta f}, \nabla_{\theta} \phi(\cdot; \theta) \Big\rangle_{L^{2}} \Bigr], \nonumber \\
v^g(w; \rho) & = \alpha \EE_{\cD} \Bigl[ \Phi(X,Z;f(\cdot, \rho)) \cdot \nabla_{\ow} \psi(Z; \ow)  - g(Z; \rho) \cdot \nabla_{\ow} \psi(Z;\ow) \Bigr].
\end{align}
From now on, we equivalently write $\theta_{k}^{i} = \theta_{i}(k)$, $\ow_{k}^{i} = \ow_{i}(k)$ to emphasize the dependence on iterations. For abbreviation, we denote $\theta^{(N)}(k) = \{\theta_{i}(k)\}_{i = 1}^{N}$ and $\ow^{(N)}(k) = \{\ow_{i}(k)\}_{i = 1}^{N}$. We recall the finite-width representation of $f(\cdot; \theta^{(N)})$ and $g(\cdot; \ow^{(N)})$ are,
\begin{align*}
& f(\cdot, \theta^{(N)}) = \frac{\alpha}{N} \cdot \sum_{i = 1}^{N} \phi(\cdot; \theta_{i}), \qquad  g(\cdot, \ow^{(N)}) = \frac{\alpha}{N} \cdot \sum_{i = 1}^{N} \psi(\cdot; \ow_{i}).
\end{align*}
Correspondingly, we defined the finite-width counter-part of $v^{f}$ and $v^{g}$ as following, 
\begin{align}
\label{eq:dir-fin}
\hat v^f(\theta; \theta^{(N)}, \ow^{(N)}) &= \alpha\EE_{\cD} \bigg[ -g(Z; \ow^{(N)}) \cdot \big\langle \frac{\delta \Phi(X,Z; f(\cdot; \theta^{(N)}))}{\delta f}, \nabla_{\theta} \phi(\cdot; \theta) \big\rangle_{L^{2}} \nonumber \\
& \quad \quad - \lambda \cdot \big\langle \frac{\delta \Psi(X,Z; f(\cdot; \theta^{(N)}))}{\delta f}, \nabla_{\theta} \phi(\cdot; \theta) \big\rangle_{L^{2}} \bigg], \nonumber \\
\hat v^g(w; \theta^{(N)}, \ow^{(N)}) & = \alpha \EE_{\cD} \bigg[ \Phi(X,Z;f(\cdot,  \theta^{(N)})) \cdot \nabla_{\ow} \psi(Z; \ow)  - g(Z; \ow^{(N)}) \cdot \nabla_{\ow} \psi(Z;\ow) \bigg].
\end{align}
And we also defined the stochastic counter-part,
\begin{align}
\label{eq:dir-sto}
\hat V_{k}^{f}(\theta; \theta^{(N)}, w^{(N)})  &= \alpha \bigg[ -g(z_k ; \ow^{(N)}) \cdot \Big\langle \frac{\delta \Phi(x_k, z_k; f(\cdot; \theta^{(N)}))}{\delta f}, \nabla_{\theta} \phi(\cdot; \theta) \Big\rangle_{L^{2}} \nonumber \\
& \quad \quad - \lambda \cdot \Big\langle \frac{\delta \Psi(x_k, z_k; f(\cdot; \theta^{(N)}))}{\delta f}, \nabla_{\theta} \phi(\cdot; \theta) \Big\rangle_{L^{2}} \bigg], \nonumber \\
\hat V_{k}^g(\ow; \theta^{(N)}, w^{(N)}) & = \alpha \Big( \Phi(x_k, z_k; f(\cdot;  \theta^{(N)})) \cdot \nabla_{\ow} \psi(z_k; \ow)  - g(z_k; \ow^{(N)}) \cdot \nabla_{\ow} \psi(z_k;\ow) \Big).
\end{align}
where $(x_{k}, z_k) \sim \cD$. Following from \cite{mei2019mean, araujo2019mean}, we consider the following four dynamics.
\begin{itemize}
    \item \textbf{Stochastic Gradient Descent Ascent (SGDA).} We consider the following SGDA dynamics for $\theta^{(N)}(k)$ and $\ow^{(N)}(k)$, where $k \in \mathbb{N}$, with $\theta_i(0) \stackrel{\mathrm{i.i.d.}}{\sim} \mu_0, w_i(0) \stackrel{\mathrm{ i.i.d. }}{\sim} \nu_0 \; (i \in[N])$ as its initialization,
    \begin{align}
    \label{eq:weak-sgd}
    	\theta_{i}(k+1) & = \theta_{i}(k) + \eta \epsilon \cdot \hat V_{k}^{f}(\theta_{i}(k); \theta^{(N)}(k), \ow^{(N)}(k)), \nonumber \\
        \ow_{i}(k+1) & = \ow_{i}(k) + \eta \epsilon \cdot \hat V_{k}^{g}(\ow_{i}(k); \theta^{(N)}(k), \ow^{(N)}(k)).
    \end{align}
    Note that this dynamics is equivalent to \eqref{eq:gdnn-fin}.
    \item \textbf{Population Gradient Descent Ascent (PGDA).} We consider the following population gradient descent ascent dynamics for $\Breve{\theta}^{(N)}(k)$ and $\Breve{\ow}^{(N)}(k)$,  where $k \in \mathbb{N}$, with $\Breve{\theta}_{i}(0) = \theta_{i}(0)$, $\Breve{\ow}_{i}(0) = \ow_{i}(0) \; (i \in[N])$ as its initialization,
    \begin{align}
        \label{eq:weak-pgd}
    	\Breve{\theta}_{i}(k+1) & = \Breve{\theta}_{i}(k) + \eta \epsilon \cdot \hat v^{f}(\Breve{\theta}_{i}(k); \Breve{\theta}^{(N)}(k), \Breve{\ow}^{(N)}(k)), \nonumber \\
        \Breve{\ow}_{i}(k+1) & = w_{i}(k) + \eta \epsilon \cdot \hat v^{g}(\Breve{\ow}_{i}(k); \Breve{\theta}^{(N)}(k), \Breve{\ow}^{(N)}(k)).
    \end{align}
    \item \textbf{Continuous-time Population Gradient Descent Ascent (CTPGDA).} We consider the following continuous time population gradient descent ascent dynamics for $\tilde{\theta}^{(N)}(t)$ and $\tilde{\ow}^{(N)}(t)$, where $t \in \RR_{+}$, with $\tilde{\theta}_{i}(0) = \theta_{i}(0)$, $\tilde{\ow}_{i}(0) = \ow_{i}(0) \; (i \in[N])$ as initialization,
    \begin{align}
        \label{eq:weak-ctpgd}
    	\frac{\mathrm{d}}{\mathrm{d} t} \tilde{\theta}_i(t) & = \eta  \cdot \hat v^{f}(\tilde{\theta}_{i}(t); \tilde{\theta}^{(N)}(t), \tilde{\ow}^{(N)}(t)), \qquad \frac{\mathrm{d}}{\mathrm{d} t} \tilde{\ow}_i(t) & = \eta \cdot \hat v^{g}(\tilde{\ow}_{i}(t); \tilde{\theta}^{(N)}(t), \tilde{\ow}^{(N)}(t)).
    \end{align}
    \item \textbf{Ideal particle (IP). } We consider the following ideal particle dynamics for $\bar \theta^{(N)}(t)$ and $\bar w^{(N)}(t)$, where $t \in \RR_{+}$, with $\bar \theta_{i}(0) = \theta_{i}(0)$, $\bar w_{i}(0) = w_{i}(0) \; (i \in[N])$ as initialization,
    \begin{align}
    \label{eq:weak-ip}
    \frac{\mathrm{d}}{\mathrm{d} t} \bar \theta_i(t)   = \eta  \cdot v^{f}(\bar \theta_{i}(t); \rho_{t}), \qquad 
    \frac{\mathrm{d}}{\mathrm{d} t} \bar \ow_i(t)  = \eta \cdot v^{g}(\bar \ow_{i}(t); \rho_{t}).
    \end{align}
\end{itemize}
We aim to prove that $\hat \rho_{k} = N^{-1} \cdot \sum_{i = 1}^{N} \delta_{\theta_{i}(k)} \cdot \delta_{w_{i}(k)}$ weakly converges to $\rho_{k\epsilon}$. For any continuous function $h$ that satisfies the assumptions of Proposition \ref{prop:weak-formal}, using the IP, CTPGDA, and PGDA dynamics as interpolating dynamics, we have,
\begin{align}
\label{eq:four-term}
& \overbrace{\Bigl| \int h(\theta, \ow) \mathrm{d} \rho_{k \epsilon}(\theta, \ow)-\int h(\theta,\ow ) \mathrm{d} \hat{\rho}_k(\theta, \ow) \Bigr|}^{\mathrm{PDE}-\mathrm{SGDA}} \nonumber \\
&\qquad \leq \underbrace{\left|\int h(\theta, \ow) \mathrm{d} \rho_{k \epsilon}(\theta)-N^{-1} \cdot \sum_{i=1}^N h\left(\bar{\theta}_i(k \epsilon), \bar{\ow}_{i}(k \epsilon) \right)\right|}_{\mathrm{PDE}-\mathrm{IP}}+\underbrace{\left\|(\bar{\theta}, \bar{\ow})^{(N)}(k \epsilon)-(\tilde{\theta}, \tilde{\ow})^{(N)}(k \epsilon)\right\|_{(N)}}_{\mathrm{IP}-\mathrm{CTPGDA}} \nonumber \\
& \qquad +\underbrace{\left\|(\tilde{\theta}, \tilde{\ow})^{(N)}(k \epsilon)-(\Breve{\theta}, \Breve{\ow}) ^{(N)}(k)\right\|_{(N)}}_{\mathrm{CTPGDA}-\mathrm{PGDA}}+\underbrace{\left\|(\Breve{\theta}, \Breve{\ow})^{(N)}(k)-(\theta, \ow)^{(N)}(k)\right\|_{(N)}}_{\mathrm{PGDA} - \mathrm{SGDA}}.
\end{align}
The last inequality follows from the fact that $\operatorname{Lip}(h) \leq 1$. Here the norm $\|\cdot\|_{(N)}$ denotes the supremum norm over the sequence of vectors $(\theta, w)^{(N)} = \{(\theta_{i}, w_{i})\}_{i = 1}^{N}$,
\begin{align}
\label{eq:sup-norm}
\Bigl\|(\theta, \ow)^{(N)}\Bigr\|_{(N)} = \sup_{i \in [N]} ~ \Bigl\|(\theta_{i}, \ow_{i})\Bigr\|.
\end{align}
In what follows, we define $B > 0$ as a constant with its value varying from line to line. We establish the following lemmas as upper-bound of the four terms on right-hand side of \eqref{eq:four-term}.
\begin{lemma}[Upper Bound of $\mathrm{PDE} - \mathrm{IP}$]
\label{lem:pde-ip}
Under Assumption \ref{asp:reg} and \ref{asp:compact}, with probability at least $1 - \delta$, it holds that
\begin{align}
\label{eq:lem-pde-ip}
\sup_{t \in [0,T]} \Bigl| \int h(\theta, \ow) \mathrm{d}\rho_{t}(\theta, \ow) - N^{-1} \sum_{i = 1}^{N} h\bigl(\bar \theta_{i}(t), \bar \ow_{i}(t)\bigr) \Bigr| \leq B \cdot \sqrt{\log (N T / \delta) / N}. 
\end{align}

\end{lemma}
\begin{lemma}[Upper Bound of $\mathrm{IP} - \mathrm{CTPGDA}$]
\label{lem:ip-ctpgd}
Under Assumption \ref{asp:reg} and \ref{asp:compact}, with probability at least $1 - 2\delta$, it holds that
\begin{align}
\label{eq:lem-ip-ctpgd}
\sup_{t \in [0,T]} \Bigl\|(\bar \theta, \bar \ow)^{(N)}(t) - (\tilde \theta, \tilde \ow)^{(N)}(t) \Bigr\|_{(N)} \leq B \cdot e^{BT} \cdot \sqrt{\log(N/\delta)/N}.
\end{align}
\end{lemma}
\begin{lemma}[Upper Bound of $\mathrm{CTPGDA} - \mathrm{PGDA}$]
\label{lem:ctpgd-pgd}
Under Assumption \ref{asp:reg} and \ref{asp:compact}, it holds that
\begin{align}
\label{eq:lem-ctpgd-pgd}
\sup_{k \leq T/\epsilon} \Bigl\|(\tilde \theta, \tilde \ow)^{(N)}(k\epsilon) - (\Breve{\theta}, \Breve{\ow})^{(N)}(k))\Bigr\|_{(N)} \leq B \cdot e^{BT} \cdot \epsilon.
\end{align}
\end{lemma}
\begin{lemma}[Upper Bound of $\mathrm{PGDA} - \mathrm{SGDA}$]
\label{lem:pgd-sgd}
Under Assumption \ref{asp:reg} and \ref{asp:compact}, with probability at least $1 - 2\delta$, it holds that
\begin{align}
\label{eq:lem-pgd-sgd}
\sup_{k \leq T/\epsilon} \Bigl\|(\Breve{\theta}, \Breve{\ow})^{(N)}(k)) - (\theta, w)^{(N)}(k)\Bigr\|_{(N)} \leq B \cdot e^{BT} \cdot \sqrt{\epsilon \cdot (D + \log (N/\delta)}.
\end{align}
\end{lemma}
\noindent With these lemmas, we are now ready to present the proof of Proposition \ref{prop:weak-formal}. \\

\begin{proof}
See \S \ref{pf:pde-ip}, \ref{pf:ip-ctpgd}, \ref{pf:ctpgd-pgd}, \ref{pf:pgd-sgd} for detailed proofs for Lemma \ref{lem:pde-ip} to Lemma \ref{lem:pgd-sgd}. \vspace{5pt}

\noindent Plug in \eqref{eq:lem-pde-ip}, \eqref{eq:lem-ctpgd-pgd}, \eqref{eq:lem-ctpgd-pgd} and \eqref{eq:lem-pgd-sgd} to \eqref{eq:four-term} and condition on the intersection of events in Lemma \ref{lem:pde-ip}, \ref{lem:ip-ctpgd}, \ref{lem:ctpgd-pgd} and \ref{lem:pgd-sgd}, we have that
\begin{align*}
\Bigl| \int h(\theta, \ow) \mathrm{d} \rho_{k \epsilon}(\theta, \ow)-\int h(\theta,\ow ) \mathrm{d} \hat{\rho}_k(\theta, \ow) \Bigr| \leq B \cdot e^{BT} \cdot \Big(\sqrt{\log (N / \delta) / N}+\sqrt{\epsilon \cdot(D+\log (N / \delta))}\Big),
\end{align*}
with probability at least $1 - 5 \delta$. Thus, we complete the proof of Proposition \ref{prop:weak-formal}.
\end{proof}

\subsection{Proofs of Lemmas \ref{lem:pde-ip}-\ref{lem:pgd-sgd}}
In this section, we present the proofs of Lemmas \ref{lem:pde-ip}-\ref{lem:pgd-sgd}, which based heavily on \cite{mei2018mean, mei2019mean, araujo2019mean, zhang2020can}. The required supporting technical lemmas are in \S\ref{sec:pf-supp}. The constant $B$ presented in the proof is a positive constant whose values varies from line to line for notational simplicity.

\subsubsection{Proof of Lemma \ref{lem:pde-ip}} \label{pf:pde-ip}
\begin{proof}
We first consider the ideal particle dynamics in \eqref{eq:weak-ip}. It holds that $\bar \theta_{i}(t) \sim \mu_{t}, \bar \ow_{i}(t) \sim \nu_{t}, \; (i \in [N])$ (Proposition 8.1.8 in \cite{ambrosio2008gradient}). Since the randomness of $\bar \theta_{i}(t)$ and $\bar \ow_{i}(t)$ comes from $\theta_{i}(0)$ and $\ow_{i}(0)$ respectively while $\theta_{i}(0)$ and $\ow_{i}(0) \; (i \in [N])$ are independent, $\bar{\theta}_i(t) \stackrel{\mathrm{i.i.d.}}{\sim} \mu_t$, $\bar{\ow}_i(t) \stackrel{\mathrm{i.i.d.}}{\sim} \nu_t \; (i \in [N])$. Due to independence of $\bar \theta_{i}(t)$ and $\bar \ow_{i}(t)$, we also have $(\bar \theta_{i}(t), \bar \ow_{i}(t)) \stackrel{\mathrm{i.i.d.}}{\sim} \rho_{t} \; (i \in [N])$. This implies the following,
\begin{align*}
    \EE_{\rho_{t}} \Bigl[N^{-1} \cdot \sum_{i = 1}^{N}h(\bar \theta_{i}(t), \bar w_{i}(t))\Bigr] = \int h(\theta,\ow) \mathrm{d} \rho_{t}(\theta, \ow).
\end{align*}
For notational simplicity, we denote $\gamma_{i} = (\theta_{i}, \ow_{i})$, similar notations also generalize to $\bar \gamma_{i}, \tilde \gamma_{i}, \Breve{\gamma}_{i}$. Let $\gamma^{1, (N)} = \{\gamma_{1}, \dots, \gamma_{i}^{1}. \dots, \gamma_{N}\}$ and $\gamma^{2,(N)} = \{\gamma_{1}, \dots, \gamma_{i}^{2}, \dots, \gamma_{N}\}$ be two sets of variables that only differ in the $i$-th element. Then, by the assumption that $\|f\|_{\infty} \leq 1$, we have the following bounded difference property,
\begin{align*}
    \Bigl|N^{-1} \sum_{j = 1}^{N} h(\gamma_{j}^{1}) - N^{-1} \sum_{j = 1}^{N} h(\gamma_{j}^{2})\Bigr| = N^{-1} \cdot |h(\gamma_{i}^{1}) - h(\gamma_{i}^{2})| \leq 2/N.
\end{align*}
Applying McDiarmid's inequality \citep{wainwright2019high}, we have for a fixed $t \in [0,T]$ that 
\begin{align}
\label{eq:mcdiarmid}
\mathbb{P}\left(\left|N^{-1} \sum_{i=1}^N h\left(\bar{\gamma}_i(t)\right)-\int h(\gamma) \mathrm{d} \rho_t(\gamma)\right| \geq p\right) \leq \exp \left(-N p^2 / 4\right).
\end{align}
Moreover, we have for any $s, t \in [0,T]$ that,
\begin{align*}
&\left|\Bigl| N^{-1} \sum_{i=1}^N h\left(\bar{\gamma}_i(t)\right)-\int h(\gamma) \mathrm{d} \rho_t(\gamma)\Bigr|-\Bigl| N^{-1} \sum_{i=1}^N h\left(\bar{\gamma}_i(s)\right)-\int h(\gamma) \mathrm{d} \rho_s(\gamma)\Bigr|\right| \\
& \qquad  \leq \Big|N^{-1} \sum_{i=1}^N h\left(\bar{\gamma}_i(t)\right)-N^{-1} \sum_{i=1}^N h\left(\bar{\gamma}_i(s)\right)\Big|+\Big|\int h(\gamma) \mathrm{d} \rho_t(\gamma)-\int h(\gamma) \mathrm{d} \rho_s(\gamma)\Big| \\
& \qquad \leq\left\|\bar{\gamma}^{(N)}(t)-\bar{\gamma}^{(N)}(s)\right\|_{(N)}+\mathcal{W}_1\left(\rho_t, \rho_s\right) \leq\left\|\bar{\gamma}^{(N)}(t)-\bar{\gamma}^{(N)}(s)\right\|_{(N)}+\mathcal{W}_2\left(\rho_t, \rho_s\right) \\
& \qquad\leq \left\|\bar{\theta}^{(N)}(t)-\bar{\theta}^{(N)}(s)\right\|_{(N)}+ \left\|\bar{w}^{(N)}(t)-\bar{w}^{(N)}(s)\right\|_{(N)} + \mathcal{W}_2\left(\mu_t, \mu_s\right) + \mathcal{W}_2\left(\nu_t, \nu_s\right).
\end{align*}
where the second inequality follows from the fact that $\operatorname{Lip}(h) \leq 1$ and Lemma \ref{lem:dual-w1}. The last inequality follows from the definition of $\gamma^{(N)}$, \eqref{eq:sup-norm} and Lemma \ref{lem:w2prod}. Applying \eqref{eq:ct-bd1}, \eqref{eq:ct-bd5} of Lemma \ref{lem:weak-ct-bd}, we have for any $s, t \in [0,T]$ that
\begin{align*}
\left| \Bigl| N^{-1} \sum_{i = 1}^{N} h(\bar \gamma_{i}(t)) - \int h(\gamma) \mathrm{d} \rho_{t} \Bigl| - \Bigl| N^{-1} \cdot \sum_{i = 1}^{N} h(\bar \gamma_{i}(s)) - \int h(\gamma) \mathrm{d} \rho_{s} \Bigl| \right| \leq B \cdot \Bigl|t-s\Bigr|.
\end{align*}
Apply the union bound to \eqref{eq:mcdiarmid} for $t \in \iota \cdot\{0,1, \ldots,\lfloor T / \iota\rfloor\}$, we have that
\begin{align*}
\mathbb{P}\left(\sup _{t \in[0, T]}\left|N^{-1} \sum_{i=1}^N h\left(\bar{\gamma}_i(t)\right)-\int h(\gamma) \mathrm{d} \rho_t(\gamma)\right| \geq p+B \cdot \iota\right) \leq(T / \iota+1) \cdot \exp \left(-N p^2 / 4\right).
\end{align*}
Setting $\iota = N^{-1/2}$ and $p = B \cdot \sqrt{\log(NT/\delta)/N}$, we have that
\begin{align*}
\sup_{t \in[0, T]}\left|N^{-1} \sum_{i=1}^N h\left(\bar{\theta}_i(t), \bar{\ow}_{i}(t) \right)-\int h(\theta, \ow) \mathrm{d} \rho_t \right| \leq B \cdot \sqrt{\log (N T / \delta) / N}.
\end{align*}
with probability at least $1 - \delta$. Thus, we complete the proof of Lemma \ref{lem:pde-ip}.
\end{proof}

\subsubsection{Proof of Lemma \ref{lem:ip-ctpgd}} \label{pf:ip-ctpgd}
Following from the definition of $\tilde \theta_{i}(t)$, $\tilde w_{i}(t)$ and $\bar \theta_{i}(t)$, $\bar w_{i}(t)$ in \eqref{eq:weak-ctpgd} and \eqref{eq:weak-ip}. We have for any $i \in [N]$ and $t \in [0,T]$ that
\begin{align}
\bigl\|\bar \theta_{i}(t) - \tilde \theta_{i}(t)\bigr\| & \; \leq \int_{0}^{t} \Bigl\|\frac{\mathrm{d} \tilde \theta_{i}(s)}{\mathrm{d}s} - \frac{\mathrm{d} \bar \theta_{i}(s)}{\mathrm{d}s}\Bigr\| \mathrm{d} s \nonumber \\
& \; \leq \eta \cdot \int_{0}^{t} \Bigl\|\hat v^{f}(\tilde \theta_{i}(s); \tilde \theta^{(N)}(s), \tilde\ow^{(N)}(s)) - \hat v^{f}(\bar \theta_{i}(s); \bar \theta^{(N)}(s), \bar \ow^{(N)}(s)) \Bigr\| \mathrm{d}s \nonumber \\ 
& \qquad + \eta \cdot \int_{0}^{t} \Bigl\|\hat v^{f}(\bar \theta_{i}(s); \bar \theta^{(N)}(s), \bar \ow^{(N)}(s)) - v^{f}(\bar \theta_{i}(s); \rho_{s}) \Bigr\| \mathrm{d} s \nonumber \\
& \; \leq B \cdot \int_{0}^{t} \Bigl\|\bar \theta^{(N)}(s) - \tilde \theta^{(N)}(s) \Bigr\|_{(N)} + \Bigl\|\bar \ow^{(N)}(s) - \tilde \ow^{(N)}(s) \Bigr\|_{(N)} \mathrm{d} s \nonumber \\
& \qquad + \eta \cdot \int_{0}^{t} \Bigl\|\hat v^{f}(\bar \theta_{i}(s); \bar \theta^{(N)}(s), \bar \ow^{(N)}(s)) - v^{f}(\bar \theta_{i}(s); \rho_{s}), \Bigr\| \mathrm{d} s \label{eq:ip-ctpgd-bd1}
\end{align}
where the last inequality follows from \eqref{eq:bd8} of Lemma \ref{lem:weak-all-bd}. 
Similarly, we have that
\begin{align}
\bigl\|\bar \ow_{i}(t) - \tilde \ow_{i}(t)\bigr\| & \; \leq B \cdot \int_{0}^{t} \Bigl\|\bar \theta^{(N)}(s) - \tilde \theta^{(N)}(s) \Bigr\|_{(N)} + \Bigl\|\bar \ow^{(N)}(s) - \tilde \ow^{(N)}(s) \Bigr\|_{(N)} \mathrm{d} s \nonumber \\
& \qquad + \eta \cdot \int_{0}^{t} \Bigl\|\hat v^{g}(\bar \ow_{i}(s); \bar \theta^{(N)}(s), \bar \ow^{(N)}(s)) - v^{g}(\bar \ow_{i}(s); \rho_{s}) \Bigr\| \mathrm{d} s, \label{eq:ip-ctpgd-bd2}
\end{align}
where the inequality follows from \eqref{eq:bd9}. We now upper-bound the second term of \eqref{eq:ip-ctpgd-bd1} and \eqref{eq:ip-ctpgd-bd2}. We start with \eqref{eq:ip-ctpgd-bd1}. Following from the definition of $v^{f}$ and $\hat v^{f}$ in \eqref{eq:dir-restate} and \eqref{eq:dir-fin}, we have for any $s \in [0,T]$ and $i \in [N]$ that
\begin{align}
\label{eq:ip-ctpgd-z}
\Bigl\|\hat v^{f}(\bar \theta_{i}(s); \bar \theta^{(N)}(s), \bar w^{(N)}(s)) - v^{f}(\bar \theta_{i}(s); \rho_{s}) \Bigr\| = \alpha^{2} \cdot \Bigl\|N^{-1} \cdot \sum_{j = 1}^{N} Z_{i}^{j}(s) \Bigr\|,
\end{align}
where $Z_{i}^{j}(s)$ is given by,
\begin{align*}
Z_{i}^{j}(s) & = \EE_{\cD} \bigg[ \Big\langle \Big( \int \psi(Z; \ow) \rd \nu_s(\ow) - \psi(Z; \bar \ow_j(s))\Big) \cdot \frac{\delta \Phi(X,Z; f)}{\delta f} , \nabla_{\theta} \phi(\cdot; \bar \theta_i (s)) \Big\rangle_{L^{2}} \nonumber \\
& \qquad \qquad + \lambda \cdot \Big\langle \frac{\delta \Psi(X,Z; \int \phi(\cdot; \theta) \rd \mu_s(\theta))}{\delta f} - \frac{\delta \Psi(X,Z; \phi(\cdot; \bar \theta_j (s)))}{\delta f}, \nabla_{\theta} \phi(\cdot; \bar \theta_i(s)) \Big\rangle_{L^{2}}\bigg].
\end{align*}
Following from Assumption \ref{asp:reg} and \ref{asp:compact}, we have that $\|Z_{i}^{j}(s)\| \leq B$. When $j \neq i$, since $\bar \theta_{j}(s) \stackrel{\mathrm{i.i.d.}}{\sim} \mu_{s}, \bar \ow_{j}(s) \stackrel{\mathrm{i.i.d.}}{\sim} \nu_{s} \; (j \in [N])$, it holds that $\EE[Z_{i}^{j}(s) \biggiven \bar \theta_{i}(s)] = 0$. Following from Lemma \ref{lem:mc-bd}, we have for fixed $s \in [0,T]$ and $i \in [N]$ that
\begin{align}
\mathbb{P}\bigg(\Bigl\|N^{-1} \cdot \sum_{j \neq i} Z_i^j(s)\Bigr\| \geq B \cdot\left(N^{-1 / 2}+p\right)\bigg) &=\mathbb{E}\Bigl[\mathbb{P}\Bigl(\Bigl\|N^{-1} \cdot \sum_{j \neq i} Z_i^j(s)\Bigr\| \geq B \cdot\left(N^{-1 / 2}+p\right) \Big|~ \bar{\theta}_i(s)\Bigr)\Bigr] \nonumber\\
& \leq \exp \left(-N p^2\right). \label{eq:ip-ctpgd-z-bd}
\end{align}
From Lemma \ref{lem:dual-w1} and \eqref{eq:ct-bd5} of Lemma \ref{lem:weak-ct-bd}, we have that
\begin{align*}
\sup_{w \in \cW} \Bigl|\int \phi(w;\theta)\mathrm{d}\mu_{s}(\theta) - \int \phi(w;\theta)\mathrm{d}\mu_{t}(\theta)\Bigr| \leq B \cdot \cW_{1}(\mu_{s}, \mu_{t}) \leq B \cdot \cW_{2}(\mu_{s}, \mu_{t}) \leq B \cdot \bigl| s - t\bigr|.
\end{align*}
Following from Assumption \ref{asp:reg} and \ref{asp:compact}, Lemma \ref{lem:weak-ct-bd}, we have for any $s,t \in [0,T]$ that,
\begin{align*}
\left| \Bigl\| N^{-1} \cdot \sum_{j \neq i} Z_{i}^{j}(s) \Bigr\| - \Bigl\| N^{-1} \cdot \sum_{j \neq i} Z_{i}^{j}(t) \Bigr\|\right| \leq B \cdot \Bigl|t-s\Bigr|.
\end{align*}
Applying the union bound to \eqref{eq:ip-ctpgd-z-bd} for $i \in [N]$ and $t \in \iota \cdot\{0,1, \ldots,\lfloor T / \iota\rfloor\}$, we have that
\begin{align*}
\mathbb{P}\Bigl(\sup _{\substack{i \in[N] \\ s \in[0, T]}}\bigl\|N^{-1} \cdot \sum_{j \neq i} Z_i^j(s)\bigr\| \geq B \cdot\left(N^{-1 / 2}+p\right)+B \iota\Bigr) \leq N \cdot(T / \iota+1) \cdot \exp \left(-N p^2\right).
\end{align*}
Setting $\iota = N^{-1/2}$ and $p = B \cdot \sqrt{\log(NT/\delta)/N}$, we have that
\begin{align}
\label{eq:ip-ctpgd-z-supbd}
\sup_{\substack{i \in[N] \\ s\in [0,T]}} \Bigl\|N^{-1} \cdot \sum_{j \neq i} Z_{i}^{j}(s) \Bigr\| \leq B \cdot \sqrt{\log(NT/\delta)/N}.
\end{align}
with probability at least $1 - \delta$. Following from Assumption \ref{asp:reg}, when $i = j$, $\|N^{-1} Z_{i}^{i}(s) \| \leq B/N$ in \eqref{eq:ip-ctpgd-z}. Plugging \eqref{eq:ip-ctpgd-z-supbd} into \eqref{eq:ip-ctpgd-z}, with probability at least $1- \delta$, we have that
\begin{align}
\label{eq:ip-ctpgd-vfbd}
\sup_{\substack{i \in [N] \\ s \in [0,T]}} \Bigl\|\hat v^{f}(\bar \theta_{i}(s); \bar \theta^{(N)}(s), \bar \ow^{(N)}(s)) - v^{f}(\bar \theta_{i}(s); \rho_{s}) \Bigr\| & \leq  \sup_{i \in [N], s \in [0,T]} \alpha^{2} \cdot \Bigl\|N^{-1} \sum_{j = 1}^{N} Z_{i}^{j}(s) \Bigr\| \nonumber \\
& \leq B \cdot \sqrt{\log(NT/\delta)/N}.
\end{align}
Through similar arguments, with probability at least $1- \delta$, the second term of \eqref{eq:ip-ctpgd-bd2} holds
\begin{align}
\label{eq:ip-ctpgd-vgbd}
\sup_{\substack{i \in [N] \\ s \in [0,T]}} \Bigl\|\hat v^{g}(\bar w_{i}(s); \bar \theta^{(N)}(s), \bar \ow^{(N)}(s)) - v^{g}(\bar \ow_{i}(s); \rho_{s}) \Bigr\| \leq B \cdot \sqrt{\log(NT/\delta)/N}.
\end{align}
Now, conditioning on the intersection of event in \eqref{eq:ip-ctpgd-vfbd} and event in \eqref{eq:ip-ctpgd-vgbd}, the following holds simultaneously for any $t \in [0,T]$
\begin{align}
& \left\|\tilde{\theta}^{(N)}(t)-\bar{\theta}^{(N)}(t)\right\|_{(N)} \leq B \cdot \int_0^t\left\|\tilde{\theta}^{(N)}(s)-\bar{\theta}^{(m)}(s)\right\|_{(N)} \mathrm{d} s+B T \cdot \sqrt{\log (N T / \delta) / N} \label{eq:ip-ctpgd-gw1}\\
& \left\|\tilde{\ow}^{(N)}(t)-\bar{\ow}^{(N)}(t)\right\|_{(N)} \leq B \cdot \int_0^t \left\|\tilde{\ow}^{(N)}(s)-\bar{\ow}^{(N)}(s)\right\|_{(N)} \mathrm{d} s+B T \cdot \sqrt{\log (N T / \delta) / N} \label{eq:ip-ctpgd-gw2}
\end{align}
Summing \eqref{eq:ip-ctpgd-gw1} and \eqref{eq:ip-ctpgd-gw2} and applying Gronwall's Lemma \citep{holte2009discrete}, with probability at least $1 - 2\delta$, for any $t \in [0,T]$, it holds that
\begin{align}
\left\|\tilde{\theta}^{(N)}(t)-\bar{\theta}^{(N)}(t)\right\|_{(N)} + \left\|\tilde{\ow}^{(N)}(t)-\bar{\ow}^{(N)}(t)\right\|_{(N)} \leq & B \cdot e^{Bt} \cdot 2BT \cdot \sqrt{\log(NT/\delta)/N} \nonumber \\
\leq & B \cdot e^{BT} \cdot \sqrt{\log(N/\delta)/N}. \label{eq:ip-ctpgd-finishpf}
\end{align}
The last inequality holds since $B$ as a constant represents values changing from line to line. Therefore, equation \eqref{eq:ip-ctpgd-finishpf} implies \eqref{eq:lem-ip-ctpgd}. Thus, we complete the proof of Lemma \ref{lem:ip-ctpgd}.

\subsubsection{Proof of Lemma \ref{lem:ctpgd-pgd}} \label{pf:ctpgd-pgd}
By the definition of $\hat v^{f}, \hat v^{g}$ in \eqref{eq:dir-fin}, $\Breve{\theta}_{i}(t), \Breve{\ow}_{i}(t)$ in \eqref{eq:weak-pgd}, $\tilde{\theta}_{i}(t), \tilde \ow_{i}(t)$ in \eqref{eq:weak-ctpgd}, it holds that the distances $\left\|\widetilde{\theta}_i(k \epsilon)-\breve{\theta}_i(k)\right\|$ and $\left\|\widetilde{\ow}_i(k \epsilon)-\breve{\ow}_i(k)\right\|$ satisfy
\begin{align}
& \left\|\widetilde{\theta}_i(k \epsilon)-\breve{\theta}_i(k)\right\| \nonumber \\
& \qquad \leq  \eta \cdot \int_0^{k \epsilon}\left\|\widehat{v}^{f} \left(\widetilde{\theta}_i(s) ; \widetilde{\theta}^{(N)}(s), \widetilde{\ow}^{(N)}(s)\right)-\widehat{v}^{f}\left(\widetilde{\theta}_i(\lfloor s / \epsilon\rfloor\epsilon) ; \widetilde{\theta}^{(N)}(\lfloor s / \epsilon\rfloor \epsilon), \widetilde{\ow}^{(N)}(\lfloor s / \epsilon\rfloor \epsilon)\right)\right\| \mathrm{d} s \nonumber \\
& \qquad \qquad +\eta \cdot \sum_{\ell=0}^{k-1}\left\|\widehat{v}^{f}\left(\widetilde{\theta}_i(\ell \epsilon) ; \widetilde{\theta}^{(N)}(\ell \epsilon), \widetilde{\ow}^{(N)}(\ell \epsilon)\right)-\widehat{v}^{f}\left(\breve{\theta}_i(\ell);\breve{\theta}^{(N)}(\ell), 
\breve{\ow}^{(N)}(\ell)\right)\right\| \nonumber \\
& \qquad \leq  B \cdot k \cdot \epsilon^2+B \cdot \sum_{\ell=0}^{k-1}\Big( \left\|\widetilde{\theta}^{(N)}(\ell \epsilon)-\breve{\theta}^{(N)}(\ell)\right\|_{(N)} + \left\|\widetilde{\ow}^{(N)}(\ell \epsilon)-\breve{\ow}^{(N)}(\ell)\right\|_{(N)} \Big). \label{eq:ctpgd-pgd-bd1}
\end{align}
\begin{align}
& \left\|\widetilde{\ow}_i(k \epsilon)-\breve{\ow}_i(k)\right\| \nonumber \\
& \qquad \leq \eta \cdot \int_0^{k \epsilon}\left\|\widehat{v}^{g} \left(\widetilde{\ow}_i(s) ; \widetilde{\theta}^{(N)}(s),\widetilde{\ow}^{(N)}(s)\right)-\widehat{v}^{g}\left(\widetilde{\ow}_i(\lfloor s / \epsilon\rfloor \epsilon) ; \widetilde{\theta}^{(N)}(\lfloor s / \epsilon\rfloor \epsilon), \widetilde{\ow}^{(N)}(\lfloor s / \epsilon\rfloor \epsilon)\right)\right\| \mathrm{d} s \nonumber \\
& \qquad \qquad +\eta \cdot \sum_{\ell=0}^{k-1}\left\|\widehat{v}^{g}\left(\widetilde{w}_i(\ell \epsilon) ; \widetilde{\theta}^{(N)}(\ell \epsilon), \widetilde{\ow}^{(N)}(\ell \epsilon)\right)-\widehat{v}^{g}\left(\breve{\ow}_i(\ell);\breve{\theta}^{(N)}(\ell), \breve{\ow}^{(N)}(\ell)\right)\right\| \nonumber \\
& \qquad \leq B \cdot k \cdot \epsilon^2+B \cdot \sum_{\ell=0}^{k-1}\Big( \left\|\widetilde{\theta}^{(N)}(\ell \epsilon)-\breve{\theta}^{(N)}(\ell)\right\|_{(N)} + \left\|\widetilde{\ow}^{(N)}(\ell \epsilon)-\breve{\ow}^{(N)}(\ell)\right\|_{(N)} \Big). \label{eq:ctpgd-pgd-bd2}
\end{align}
where \eqref{eq:ctpgd-pgd-bd1} follows from \eqref{eq:bd8} of Lemma \ref{lem:weak-all-bd} and \eqref{eq:ct-bd3} of Lemma \ref{lem:weak-ct-bd}, \eqref{eq:ctpgd-pgd-bd2} follows from \eqref{eq:bd9} of Lemma \ref{lem:weak-all-bd} and \eqref{eq:ct-bd3} of Lemma \ref{lem:weak-ct-bd}. Combining the inequalities in  \eqref{eq:ctpgd-pgd-bd1} and \eqref{eq:ctpgd-pgd-bd2}, it holds for any $k \leq T/\epsilon \;(k \in \mathbb{N})$ that
\begin{align}
& \left\|\widetilde{\theta}^{(N)}(k \epsilon)-\breve{\theta}^{(N)}(k)\right\|_{(N)} + \left\|\widetilde{\ow}^{(N)}(k \epsilon)-\breve{\ow}^{(N)}(k)\right\|_{(N)} \nonumber \\
& \qquad \leq 2B T \epsilon +B \cdot \sum_{\ell=0}^{k-1}\left\|\widetilde{\theta}^{(m)}(\ell \epsilon)-\breve{\theta}^{(N)}(\ell)\right\|_{(N)} +  B \cdot \sum_{\ell=0}^{k-1} \left\|\widetilde{\ow}^{(N)}(\ell \epsilon)-\breve{\ow}^{(N)}(\ell)\right\|_{(N)}. \label{eq:ctpgd-pgd-gw}
\end{align}
Applying the discrete Gronwall's lemma \citep{holte2009discrete} to \eqref{eq:ctpgd-pgd-gw} , we have that
\begin{align*}
\sup _{\substack{k \leq T / \epsilon \\(k \in \mathbb{N})}}\left\|\widetilde{\theta}^{(N)}(k \epsilon)-\breve{\theta}^{(N)}(k)\right\|_{(N)} + \left\|\widetilde{\ow}^{(N)}(k \epsilon)-\breve{\ow}^{(N)}(k)\right\|_{(N)} \leq 2B^2 \cdot T \cdot \epsilon \cdot e^{B T} \leq B \cdot e^{B T} \cdot \epsilon,
\end{align*}
where the inequalities hold since we allow the value of $B$ to vary from line to line. Thus, we complete the proof of Lemma \ref{lem:ctpgd-pgd}.
\subsubsection{Proof of Lemma \ref{lem:pgd-sgd}} \label{pf:pgd-sgd}
\begin{proof}
Let $\cG_{k} = \sigma(\theta^{(N)}(0), w^{(N)}(0), u_{0}, \dots, u_{k})$ be the $\sigma-$algebra generated by $\theta^{(N)}(0), w^{(N)}(0)$ and $u_{\ell} = (x_{\ell}, z_{\ell}) \; (\ell \leq k)$. Following from the definition of $\hat V^{f}_{k}, \hat V^{g}_{k}$ and $\hat v^{f}, \hat v^{g}$ in \eqref{eq:dir-fin} and \eqref{eq:dir-sto}, we have for any $i \in [N]$ and $k \in \mathbb{N}_{+}$ that 
\begin{align*}
& \EE \Bigl[\hat V_{k}^{f}(\theta_{i}(k); \theta^{(N)}(k), \ow^{(N)}(k)) \biggiven \cG_{k-1} \Bigr] = \hat v^{f}(\theta_{i}(k); \theta^{(N)}(k), \ow^{(N)}(k)), \\
& \EE \Bigl[\hat V_{k}^{g}(\ow_{i}(k); \theta^{(N)}(k), \ow^{(N)}(k)) \biggiven \cG_{k-1} \Bigr] = \hat v^{g}(\ow_{i}(k); \theta^{(N)}(k), \ow^{(N)}(k)).
\end{align*}
Recall the definition of $\theta^{(N)}, \ow^{(N)}$ and $\Breve{\theta}^{(N)}, \Breve{\ow}^{(N)}$ as the SGDA and PGDA dynamics defined in \eqref{eq:weak-sgd} and \eqref{eq:weak-pgd}. We have for any $i \in [N]$, $k \in \mathbb{N}_{+}$ that
\begin{align}
& \left\|\breve{\theta}_i(k)-\theta_i(k)\right\| \nonumber \\
& \qquad \leq \eta \epsilon \cdot\left\|\sum_{\ell=0}^{k-1} X_i(\ell)\right\|+\eta \epsilon \cdot \sum_{\ell=0}^{k-1}\left\|\widehat{v}^{f}\left(\breve{\theta}_i(\ell) ; \breve{\theta}^{(N)}(\ell),\breve{\ow}^{(N)}(\ell) \right)-\widehat{v}^{f}\left(\theta_i(\ell) ; \theta^{(N)}(\ell), \ow^{(N)}(\ell)\right)\right\| \nonumber \\
& \qquad \leq \eta \epsilon \cdot\left\|A_i(k)\right\|+B \epsilon \cdot  \sum_{\ell=0}^{k-1}\left\|\breve{\theta}^{(m)}(\ell)-\theta^{(m)}(\ell)\right\|_{(N)} + \left\|\breve{\ow}^{(N)}(\ell)-\ow^{(N)}(\ell)\right\|_{(N)}, \label{eq:pgd-sgd-eq3}
\end{align}
where the last inequality follows from \eqref{eq:bd8} of Lemma \ref{lem:weak-all-bd}. $X_{i}(\ell)$ and $A_{i}(k)$ are defined as,
\begin{align*}
&X_i(\ell)=\widehat{V}^{f}_{\ell}\left(\theta_i(\ell) ; \theta^{(N)}(\ell), \ow^{(N)}(\ell)\right)-\mathbb{E}\left[\widehat{V}^{f}_{\ell}\left(\theta_i(\ell) ; \theta^{(N)}(\ell), \ow^{(N)}(\ell)\right) \biggiven \mathcal{G}_{\ell-1}\right] \quad \forall \ell \geq 1, \\
&X_{i}(0) = 0, \quad A_i(k)=\sum_{\ell=0}^{k-1} X_i(\ell).
\end{align*}
Following from \eqref{eq:bd7} of Lemma \ref{lem:weak-all-bd}, it holds that $\|X_{i}(\ell)\| \leq B$, thus the stochastic process $\{A_{i}(k)\}_{k \in \mathbb{N}_{+}}$ is a martingale with $\|A_{i}(k) - A_{i}(k-1)\| \leq B$. Applying the Azuma-Hoeffding bound in Lemma \ref{lem:ah-bd}, we have that
\begin{align}
\label{eq:pgd-sgd-ah-bd1}
\mathbb{P} \Bigl(\underset{\substack{k \leq T / \epsilon \\\left(k \in \mathbb{N}_{+}\right)}}{\max}\left\|A_i(k)\right\| \geq B \cdot \sqrt{T / \epsilon} \cdot(\sqrt{D}+p)\Bigr) \leq \exp \left(-p^2\right) .
\end{align}
Apply the union bound to \eqref{eq:pgd-sgd-ah-bd1} for $i \in [N]$, we have that
\begin{align*}
\mathbb{P} \Bigl(\underset{\substack{i \in [N] \\ k \leq T / \epsilon, \left(k \in \mathbb{N}_{+}\right)}}{\max}\left\|A_i(k)\right\| \geq B \cdot \sqrt{T / \epsilon} \cdot(\sqrt{D}+p)\Bigr) \leq N \cdot \exp \left(-p^2\right).
\end{align*}
Setting $p = \sqrt{\log(N/\delta)}$, with probability at least $1 - \delta$, it holds that
\begin{align}
\label{eq:pgd-sgd-ah-bd3}
\left\|A_i(k)\right\| \leq B \cdot \sqrt{T / \epsilon} \cdot(\sqrt{D}+\sqrt{\log (N / \delta)}), \quad \forall i \in[N], k \leq T / \epsilon\left(k \in \mathbb{N}_{+}\right).
\end{align}
Plug \eqref{eq:pgd-sgd-ah-bd3} into \eqref{eq:pgd-sgd-eq3} and taking supremum norm over $i \in [N]$, we have that
\begin{align}
\Bigl\|\Breve{\theta}^{(N)}(k) - \theta^{(N)}(k) \Bigr\|_{(N)} & \leq B \epsilon \cdot  \sum_{\ell=0}^{k-1}\bigg(\left\|\breve{\theta}^{(m)}(\ell)-\theta^{(m)}(\ell)\right\|_{(N)} + \left\|\breve{\ow}^{(N)}(\ell)-\ow^{(N)}(\ell)\right\|_{(N)} \bigg) \nonumber \\
& \qquad + B \cdot \sqrt{T\epsilon} \cdot (\sqrt{D} + \sqrt{\log(N/\delta)}). \label{eq:pgd-sgd-eq4}
\end{align}
Through similar arguments, for $\Breve{w}_{i}(k)$ and $w_{i}(k)$, with probability at least $1 - \delta$, 
\begin{align}
\Bigl\|\Breve{\ow}^{(N)}(k) - \ow^{(N)}(k) \Bigr\|_{(N)} & \leq B \epsilon \cdot  \sum_{\ell=0}^{k-1}\bigg(\left\|\breve{\theta}^{(m)}(\ell)-\theta^{(m)}(\ell)\right\|_{(N)} + \left\|\breve{\ow}^{(N)}(\ell)-\ow^{(N)}(\ell)\right\|_{(N)}\bigg) \nonumber \\
& \qquad + B \cdot \sqrt{T\epsilon} \cdot (\sqrt{D} + \sqrt{\log(N/\delta)}). \label{eq:pgd-sgd-eq5}
\end{align}
Conditioning on the intersection of event in \eqref{eq:pgd-sgd-eq4} and event in \eqref{eq:pgd-sgd-eq5}, summing \eqref{eq:pgd-sgd-eq4}, \eqref{eq:pgd-sgd-eq5}, and applying the discrete Gronwall's lemma \citep{holte2009discrete}, for any $k \leq T/\epsilon, k \in \mathbb{N}_{+}$, the following inequality holds with probability at least $1 - 2\delta$,
\begin{align*}
\Bigl\|\Breve{\theta}^{(N)}(k) - \theta^{(N)}(k) \Bigr\|_{(N)} +  \Bigl\|\Breve{\ow}^{(N)}(k) - \ow^{(N)}(k) \Bigr\|_{(N)} & \leq B \cdot e^{Bk\epsilon} \cdot B \cdot \sqrt{T\epsilon} \cdot (\sqrt{D} + \sqrt{\log(N/\delta)}) \nonumber \\
& \leq B \cdot e^{BT} \cdot \sqrt{\epsilon \cdot (D + \log(N/\delta))}.
\end{align*}
Here the last inequality holds since we allow the value of $B$ to vary from line to line. Thus, we complete the proof of Lemma \ref{lem:pgd-sgd}.
\end{proof}

\section{Supporting Lemmas} \label{sec:pf-supp}
\subsection{Supporting Lemmas for \S \ref{sec:pf-prop-weak}} \label{sec:pf-sup-mf}
In what follows, we presented the technical lemmas heavily used in $\S$ \ref{sec:pf-prop-weak}. We recall the definition of $v^{f},v^{g}$, $\hat v^{f}, \hat v^{g}$ and $\hat V^{f}_{k}, \hat V^{g}_{k}$ as in \eqref{eq:dir-restate}, \eqref{eq:dir-fin}, and \eqref{eq:dir-sto} respectively. Let $B > 0$ be a constant depending on $\alpha, \eta, B_{0}, B_{1}, B_{2}, C$, whose value varies from line to line. Recall that $f(\cdot;\theta^{(N)})$ and $g(\cdot; \ow^{(N)})$ are the finite width representation with parameters $\theta^{(N)}, \ow^{(N)}$, whose definitions are given by
\begin{align*}
&f(\cdot; \theta^{(N)}) = \frac{\alpha}{N}\cdot \sum_{i = 1}^{N} \phi(\cdot;\theta_{i}),  \quad g(\cdot; \ow^{(N)}) = \frac{\alpha}{N}\cdot \sum_{i = 1}^{N} \psi(\cdot;\ow_{i}).
\end{align*}
\begin{lemma}
\label{lem:weak-all-bd}
Under Assumption \ref{asp:reg} and \ref{asp:compact}, it holds that for any $\theta^{(N)} = \{\theta_{i}\}_{i = 1}^{N}$, $\underline{\theta}^{(N)} = \{\underline{\theta}_{i}\}_{i = 1}^{N}$, $w^{(N)} = \{w_{i}\}_{i = 1}^{N}$, $\underline{w}^{(N)} = \{\underline{w}_{i}\}_{i = 1}^{N}$, that, $f(\cdot; \theta^{(N)})$ and $g(\cdot; \ow^{(N)})$ are uniformly bounded and Lipschitz in $\theta, \ow$ respectively, which is given by the following, 
\begin{align}
& \sup_{w \in \cW} \bigl|f(w;\theta^{(N)})\bigr| + \sup_{z \in \cZ} \bigl|g(z;\ow^{(N)})\bigr| \leq B, \label{eq:bd1}\\
& \sup_{w \in \cW} \bigl|f(w;\theta^{(N)}) - f(w; \underline{\theta}^{(N)})\bigr| \leq B \cdot \bigl\|\theta^{(N)} - \underline{\theta}^{(N)}\bigr\|_{(N)}, \label{eq:bd2} \\
& \sup_{z \in \cZ} \bigl|g(z;\ow^{(N)}) - g(z; \underline{\ow}^{(N)})\bigr| \leq B \cdot \bigl\|\ow^{(N)} - \underline{\ow}^{(N)}\bigr\|_{(N)}. \label{eq:bd3}
\end{align}
Recall the definition of $\hat v^{f}, \hat v^{g}$ and $\hat V^{f}_{k}, \hat V^{g}_{k}$ in \eqref{eq:dir-fin}, \eqref{eq:dir-sto}, the finite width representation of the velocity field and its stochastic counter-part, when evaluated at arbitrary $\theta_i, \ow_i$, are also uniformly bounded and lipschitz in $\theta, \ow$ respectively. This means for $\hat V^{f}_{k}, \hat V^{g}_{k}$, the following inequalities hold,
\begin{align}
& \bigl\|\hat V^{f}_{k}(\theta_{i}; \theta^{(N)}, \ow^{(N)}) \bigr\| + \bigl\|\hat V^{g}_{k}(\ow_{i}; \theta^{(N)}, w^{(N)}) \bigr\| \leq B, \label{eq:bd4} \\
& \bigl\|\hat V^{f}_{k}(\theta_{i}; \theta^{(N)}, \ow^{(N)}) - \hat V^{f}_{k}(\underline{\theta}_{i}; \underline{\theta}^{(N)}, \underline{\ow}^{(N)})  \bigr\| \leq B \cdot \Bigl( \bigl\|\theta^{(N)} - \underline{\theta}^{(N)} \bigr\|_{(N)} + \bigl\| \ow^{(N)} - \underline{\ow}^{(N)} \bigr\|_{(N)} \Bigr), \label{eq:bd5} \\
& \bigl\|\hat V^{g}_{k}(\ow_{i}; \theta^{(N)}, w^{(N)}) - \hat V^{g}_{k}(\underline{\ow}_{i}; \underline{\theta}^{(N)}, \underline{\ow}^{(N)})  \bigr\| \leq B \cdot \Bigl( \bigl\|\theta^{(N)} - \underline{\theta}^{(N)} \bigr\|_{(N)} + \bigl\| \ow^{(N)} - \underline{\ow}^{(N)} \bigr\|_{(N)} \Bigr). \label{eq:bd6}
\end{align}
A similar series of inequalities also hold for $\hat v^{f}, \hat v^{g}$,
\begin{align}
& \bigl\|\hat v^{f}(\theta_{i}; \theta^{(N)}, \ow^{(N)}) \bigr\| + \bigl\|\hat v^{g}(\ow_{i}; \theta^{(N)}, \ow^{(N)}) \bigr\| \leq B, \label{eq:bd7} \\
& \bigl\|\hat v^{f}(\theta_{i}; \theta^{(N)}, \ow^{(N)}) - \hat v^{f}_{k}(\underline{\theta}_{i}; \underline{\theta}^{(N)}, \underline{\ow}^{(N)})  \bigr\| \leq B \cdot \Bigl( \bigl\|\theta^{(N)} - \underline{\theta}^{(N)} \bigr\|_{(N)} + \bigl\| \ow^{(N)} - \underline{\ow}^{(N)} \bigr\|_{(N)} \Bigr), \label{eq:bd8} \\
& \bigl\|\hat v^{g}(\ow_{i}; \theta^{(N)}, \ow^{(N)}) - \hat v^{g}(\underline{\ow}_{i}; \underline{\theta}^{(N)}, \underline{\ow}^{(N)})  \bigr\| \leq B \cdot \Bigl( \bigl\|\theta^{(N)} - \underline{\theta}^{(N)} \bigr\|_{(N)} + \bigl\| \ow^{(N)} - \underline{\ow}^{(N)} \bigr\|_{(N)} \Bigr). \label{eq:bd9}
\end{align}
As a corollary of the inequalities stated above, the uniform bounds in fact hold for any $f,g \in \cF$, which says,
\begin{align}
\label{eq:bd10}
\sup_{w\in \cW} \bigl |f(w)\bigr | + \sup_{z \in \cZ} \bigl |g(z)\bigr| \leq B.
\end{align}
Similarly, the uniform bounds also hold for the velocity field $v^{f}, v^{g}$, such that for any $\rho \in \sP_{2}(\RR^{D} \times \RR^{D})$, it holds that
\begin{align}
\label{eq:bd11}
\bigl\| v^{f}(\theta; \rho) \bigr\| + \bigl\|v^{g}(\ow; \rho) \bigr\| \leq B.
\end{align}
\end{lemma}

\begin{proof} 
We will prove these results separately. \vspace{5pt}

\noindent \textbf{(i) Proof of \eqref{eq:bd1}, \eqref{eq:bd2}, and \eqref{eq:bd3}} \vspace{5pt}

\noindent For \eqref{eq:bd1} of Lemma \ref{lem:weak-all-bd}, since $\phi$, $\psi$ are bounded as is assumed in Assumption \ref{asp:reg}, we have for any $w \in \cW, z \in \cZ$, any $\theta^{(N)}$ and $\ow^{(N)}$ that
\begin{align*}
\bigl|f(w;\theta^{(N)}) \bigr| + \bigl|g(z;\ow^{(N)})\bigr| \leq \alpha \cdot N^{-1} \sum_{i = 1}^{N} \bigl|\phi(w;\theta_{i})\bigr| + \bigl|\psi(z;\ow_{i})\bigr|  \leq B.
\end{align*}
For \eqref{eq:bd2}, and \eqref{eq:bd3} of Lemma \ref{lem:weak-all-bd}, since for any $w \in \cW$, $z \in \cZ$, $\phi(w; \theta)$ has a bounded gradient in $\theta$, $\psi(z;\ow)$ has a bounded gradient in $\ow$. The uniform upper bound of the gradient controls the Lipschitz constant of the function, thus it holds for any $w \in \cW, z \in \cZ$, any $\theta^{(N)}, \underline{\theta}^{(N)}$ and $\ow^{(N)}, \underline{\ow}^{(N)}$ that
\begin{align*}
    & \bigl|f(w;\theta^{(N)}) - f(w; \underline{\theta}^{(N)}) \bigr| \leq \alpha N^{-1} \cdot B_1 \sum_{i = 1}^{N} \bigl| \theta_{i} - \underline{\theta}_{i} \bigr| \leq B \cdot \bigl\| \theta^{(N)} - \underline{\theta}^{(N)} \bigr\|_{(N)}, \\
& \bigl|g(z;\ow^{(N)}) - g(z; \underline{\ow}^{(N)})\bigr| \leq \alpha N^{-1} \cdot B_1 \sum_{i = 1}^{N} \bigl| \ow_{i} -  \underline{\ow}_{i} \bigr| \leq B \cdot \bigl\| \ow^{(N)} - \underline{\ow}^{(N)} \bigr\|_{(N)}.
\end{align*}

\noindent \textbf{(ii) Proof of \eqref{eq:bd4}, \eqref{eq:bd5} and \eqref{eq:bd6}} \vspace{5pt}

\noindent For \eqref{eq:bd4} of Lemma \ref{lem:weak-all-bd}, recall the definition of $\hat V^{f}_{k}$, $\hat V^{g}_{k}$ in \eqref{eq:dir-sto}, for any $\theta^{(N)}$ and $\ow^{(N)}$, 
\begin{align*}
\Bigl\| \hat V^{f}_{k}(\theta_{i}; \theta^{(N)}, \ow^{(N)}) \Bigr\| & \; \leq \alpha \cdot \sup_{w \in \cW} \big\|\nabla_{\theta} \phi(w; \theta_i) \big\| \cdot \sup_{z \in \cZ} \big|g(z; \ow^{(N)})\big| \cdot \int_{\cW} \Big|\frac{\delta \Phi(x_k, z_k, f(\cdot; \theta^{(N)}))}{\delta f}(w')\Big| \rd w' \\ 
& \qquad + \alpha \cdot \sup_{w \in \cW} \big\|\nabla_{\theta} \phi(w; \theta_i) \big\| \cdot  \lambda \cdot \int_{\cW} \Big|\frac{\delta \Psi(x_k, z_k, f(\cdot; \theta^{(N)}))}{\delta f}(w')\Big| \rd w' \leq B, \\
\Bigl\| \hat V^{g}_{k}(\ow_i; \theta^{(N)}, \ow^{(N)}) \Bigr\| & \; \leq \alpha \cdot \Big( \big|\Phi(x_k, z_k; f(\cdot; \theta^{(N)}))\big| + \sup_{z \in \cZ} \big|g(z;\ow^{(N)})\big| \Big) \cdot \sup_{z \in \cZ} \big\| \nabla_{\ow} \psi(z; \ow_i) \big\| \leq B.
\end{align*}
For notational simplicity, we further define
\begin{align*}
& u^{f}(\theta^{(N)}, w^{(N)}) = -\alpha g(z_k ; \ow^{(N)}) \cdot \frac{\delta \Phi(x_k, z_k; f(\cdot; \theta^{(N)}))}{\delta f} -  \alpha \lambda \cdot \frac{\delta \Psi(x_k, z_k; f(\cdot; \theta^{(N)}))}{\delta f}, \\
& u^{g}(\theta^{(N)}, w^{(N)}) = \alpha \Phi(x_k, z_k; f(\cdot;  \theta^{(N)}))- \alpha g(z_k; \ow^{(N)}).
\end{align*}
For \eqref{eq:bd5} of Lemma \ref{lem:weak-all-bd}, following from Assumption \ref{asp:compact} and the definition of $\hat V^{f}_{k}$ in \eqref{eq:dir-sto}, we have for any $\theta^{(N)}, \underline{\theta}^{(N)}$ and $\ow^{(N)}, \underline{\ow}^{(N)}$ that
\begin{align*}
& \Bigl\| \hat V^{f}_{k}(\theta_{i}; \theta^{(N)}, \ow^{(N)}) - \hat V^{f}_{k}(\underline{\theta}_{i}; \underline{\theta}^{(N)}, \underline{\ow}^{(N)})\Bigr\| \\
& \qquad \leq \Bigl\| \hat V^{f}_{k}(\theta_{i}; \theta^{(N)}, \ow^{(N)}) - \hat V^{f}_{k}(\theta_{i}; \underline{\theta}^{(N)}, \underline{\ow}^{(N)})\Bigr\| + \Bigl\| \hat V^{f}_{k}(\theta_{i}; \underline{\theta}^{(N)}, \underline{\ow}^{(N)}) - \hat V^{f}_{k}(\underline{\theta}_{i}; \underline{\theta}^{(N)}, \underline{\ow}^{(N)})\Bigr\| \\
& \qquad \leq |u^{f}(\theta^{(N)}, \ow^{(N)}) - u^{f}(\underline{\theta}^{(N)}, \underline{\ow}^{(N)})| \cdot \sup_{w \in \cW} \|\nabla_{\theta} \phi(w; \theta_i) \| + \Bigl\|\inp[\Big]{u^{f}(\underline{\theta}^{(N)}, \underline{\ow}^{(N)})}{\nabla_{\theta}\phi(\cdot; \theta_i) - \nabla_{\theta} \phi(\cdot; \underline{\theta}_i)}_{L^{2}}\Bigr\|.
\end{align*}
Moreover, $u^{f}(\theta^{(N)}, \ow^{(N)})$ is also Lipschitz in $(\theta^{(N)}, \ow^{(N)})$ since
\begin{align*}
|u^{f}(\theta^{(N)}, \ow^{(N)}) - u^{f}(\underline{\theta}^{(N)}, \underline{\ow}^{(N)})| & \; \leq  B \cdot  |f(w_k; \theta^{(N)}) - f(w_k; \underline{\theta}^{(N)})| + B \cdot |g(z_k; \ow^{(N)}) - g(z_{k}; \underline{\ow}^{(N)})| \\
& \; \leq B \cdot \Bigl( \bigl\|\theta^{(N)} - \underline{\theta}^{(N)} \bigr\|_{(N)} + \bigl\| \ow^{(N)} - \underline{\ow}^{(N)} \bigr\|_{(N)} \Bigr),
\end{align*}
where the second inequality is achieved by applying \eqref{eq:bd2}, \eqref{eq:bd3}. Therefore, the fact that $\hat V^{f}_{k}(\theta_{i}; \theta^{(N)}, \ow^{(N)})$ is Lipschitz in $(\theta^{(N)}, \ow^{(N)})$ is due to $\|\nabla_{\theta} \phi(w; \theta_i)\|$ and $\big|\int u^{f}(\theta^{(N)}, \ow^{(N)})(w') \rd w' \big|$ is uniformly bounded. \vspace{5pt}

\noindent For \eqref{eq:bd6} of Lemma \ref{lem:weak-all-bd}, following from Assumption \ref{asp:compact} and the definition of $\hat V^{g}_{k}$ in \eqref{eq:dir-sto}, through a similar argument as is in the proof of \eqref{eq:bd5}, we have for any $\theta^{(N)}, \underline{\theta}^{(N)}$ and $\ow^{(N)}, \underline{\ow}^{(N)}$ that
\begin{align*}
&\Bigl\| \hat V^{g}_{k}(\ow_{i}; \theta^{(N)}, \ow^{(N)}) - \hat V^{g}_{k}(\underline{\ow}_{i}; \underline{\theta}^{(N)}, \underline{\ow}^{(N)})\Bigr\| \\
& \qquad \leq \Bigl\| \hat V^{g}_{k}(\ow_{i}; \theta^{(N)}, \ow^{(N)}) - \hat V^{g}_{k}(\ow_{i}; \underline{\theta}^{(N)}, \underline{\ow}^{(N)})\Bigr\| + \Bigl\| \hat V^{g}_{k}(\ow_{i}; \underline{\theta}^{(N)}, \underline{\ow}^{(N)}) - \hat V^{g}_{k}(\underline{\ow}_{i}; \underline{\theta}^{(N)}, \underline{\ow}^{(N)})\Bigr\| \\
& \qquad \leq |u^{g}(\theta^{(N)}, \ow^{(N)}) - u^{f}(\underline{\theta}^{(N)}, \underline{\ow}^{(N)})| \cdot \sup_{z \in \cZ} \|\nabla_{\ow} \psi(z; \ow_i) \| + \Bigl\|\inp[\Big]{u^{g}(\underline{\theta}^{(N)}, \underline{\ow}^{(N)})}{\nabla_{\ow}\psi(\cdot; \ow_i) - \nabla_{\ow} \psi(\cdot; \underline{\ow}_i)}_{L^{2}}\Bigr\|.
\end{align*}
Again, $u^{g}(\theta^{(N)}, \ow^{(N)})$ is Lipschitz in $(\theta^{(N)}, \ow^{(N)})$ since
\begin{align*}
|u^{g}(\theta^{(N)}, \ow^{(N)}) - u^{g}(\underline{\theta}^{(N)}, \underline{\ow}^{(N)})| & \; \leq B \cdot  |f(w_k; \theta^{(N)}) - f(w_k; \underline{\theta}^{(N)})| +  B \cdot |g(z_k; \ow^{(N)}) - g(z_{k}; \underline{\ow}^{(N)})| \\
& \; \leq B \cdot \Bigl( \bigl\|\theta^{(N)} - \underline{\theta}^{(N)} \bigr\|_{(N)} + \bigl\| \ow^{(N)} - \underline{\ow}^{(N)} \bigr\|_{(N)} \Bigr).
\end{align*}
Therefore, the Lipschtizness of $\hat V^{g}_{k}(\ow_{i}; \theta^{(N)}, \ow^{(N)})$ in $(\theta^{(N)}, \ow^{(N)})$ comes from $\|\nabla_{\ow} \psi(z; \ow_i) \|$ and $\big|\int u^{g}(\theta^{(N)}, \ow^{(N)})(z') \rd z'\big|$ is uniformly bounded. \vspace{5pt}

\noindent \textbf{(iii) Proof of \eqref{eq:bd7}, \eqref{eq:bd8}, and \eqref{eq:bd9}} \vspace{5pt}

\noindent Equations \eqref{eq:bd7}, \eqref{eq:bd8}, \eqref{eq:bd9} of Lemma \ref{lem:weak-all-bd} for $\hat v^{f}$ and $\hat v^{g}$ follow from the fact that
\begin{align*}
\hat v^{f}(\theta_{i}; \theta^{(N)}, \ow^{(N)}) = \EE_{\cD} \Big[ \hat V_{k}^{f}(\theta_{i}; \theta^{(N)}, \ow^{(N)})\Big], \quad  
\hat v^{g}(\ow_{i}; \theta^{(N)}, w^{(N)}) = \EE_{\cD} \Big[ \hat V_{k}^{g}(\ow_{i}; \theta^{(N)}, \ow^{(N)})\Big].
\end{align*}
Therefore, \eqref{eq:bd7} follows from \eqref{eq:bd4} and triangle inequality,
\begin{align*}
\bigl\|\hat v^{f}(\theta_{i}; \theta^{(N)}, \ow^{(N)}) \bigr\| + \bigl\|\hat v^{g}(\ow_{i}; \theta^{(N)}, \ow^{(N)}) \bigr\|& \; \leq \EE_{\cD} \Bigl[ \bigl\| \hat V_{k}^{f}(\theta_{i}; \theta^{(N}), \ow^{(N)})\bigl\| \Bigr] + \EE_{\cD} \Bigl[ \bigl\| \hat V_{k}^{g}(\ow_{i}; \theta^{(N)}, \ow^{(N)})\bigl\| \Bigr] \leq B.
\end{align*}
Equations \eqref{eq:bd8} and \eqref{eq:bd9} follows from \eqref{eq:bd5}, \eqref{eq:bd6} and triangle inequality,
\begin{align*}
\bigl\|\hat v^{f}(\theta_{i}; \theta^{(N)}, \ow^{(N)}) - \hat v^{f}_{k}(\underline{\theta}_{i}; \underline{\theta}^{(N)}, \underline{\ow}^{(N)})  \bigr\| & \; \leq \EE_{\cD} \Big[ \bigl\|\hat V^{f}_{k}(\theta_{i}; \theta^{(N)}, \ow^{(N)}) - \hat V^{f}_{k}(\underline{\theta}_{i}; \underline{\theta}^{(N)}, \underline{\ow}^{(N)})  \bigr\| \Big] \\
& \; \leq B \cdot \Bigl( \bigl\|\theta^{(N)} - \underline{\theta}^{(N)} \bigr\|_{(N)} + \bigl\| \ow^{(N)} - \underline{\ow}^{(N)} \bigr\|_{(N)} \Bigr), \\
\bigl\|\hat v^{g}(\ow_{i}; \theta^{(N)}, \ow^{(N)}) - \hat v^{g}_{k}(\underline{\ow}_{i}; \underline{\theta}^{(N)}, \underline{\ow}^{(N)})  \bigr\| & \; \leq \EE_{\cD} \Big[ \bigl\|\hat V^{g}_{k}(\ow_{i}; \theta^{(N)}, \ow^{(N)}) - \hat V^{g}_{k}(\underline{\ow}_{i}; \underline{\theta}^{(N)}, \underline{\ow}^{(N)})  \bigr\| \Big] \\
& \; \leq B \cdot \Bigl( \bigl\|\theta^{(N)} - \underline{\theta}^{(N)} \bigr\|_{(N)} + \bigl\| \ow^{(N)} - \underline{\ow}^{(N)} \bigr\|_{(N)} \Bigr).
\end{align*}
\noindent \textbf{(iv) Proof of \eqref{eq:bd10}, and \eqref{eq:bd11}} \vspace{5pt}

\noindent Equation \eqref{eq:bd10} follows from the definition of $\cF$ in \eqref{eq:func-class} and the uniform bounds of neuron functions $\phi$ and $\psi$. For any $f, g \in \cF$, there exists probability measures $\hat \mu, \hat \nu$ over the parameter space such that
\begin{align*}
    f(w) = \int \phi(w; \theta) \hat \mu( \rd \theta), \quad g(z) = \int \psi(z; \ow) \hat \nu (\rd \ow), \quad \forall w \in \cW, z \in \cZ.
\end{align*}
We apply the triangle inequality and achieve,
\begin{align*}
    \sup_{w \in \cW} |f(w)| + \sup_{z \in \cZ} |g(z)| \leq \int \sup_{w \in \cW} |\phi(w; \theta)| \hat \mu(\rd \theta) + \int \sup_{z \in \cZ} |g(z)| |\psi(z; \ow)| \hat \nu(\rd \ow) \leq B.
\end{align*}
\noindent Equation \eqref{eq:bd11} follows from the definition of $v^{f}, v^{g}$ in \eqref{eq:dir-restate} and the proof of \eqref{eq:bd4} and \eqref{eq:bd7}. Proof of \eqref{eq:bd11} is the same as the proof for \eqref{eq:bd4} and \eqref{eq:bd7}, except for the fact that a uniform bound is needed for the infinite width representation of $f$ and $g$, which is proved in \eqref{eq:bd10}. \vspace{7pt}

\noindent Based on proofs for items (i), (ii), (iii), and (iv) above, we finish the proof of Lemma \eqref{lem:weak-all-bd}.
\end{proof} 
\vspace{7pt}

\noindent Now, recall $\rho_{t}$ is the PDE solution to \eqref{eq:pde}, $\bar \theta^{(N)}(t), \bar w^{(N)}(t)$ is the IP dynamics defined in \eqref{eq:weak-ip}, $\tilde \theta^{(N)}(t), \tilde w^{(N)}(t)$ is the CTPGDA dynamics defined in \eqref{eq:weak-ctpgd}. We have the following lemma that also bound the difference of iterates for IP, CTPGDA dynamics between time $s$ and $t$.
\begin{lemma}
\label{lem:weak-ct-bd}
Under Assumption \ref{asp:reg} and \ref{asp:compact}, it holds for any $s, t \in [0,T]$ that,
\begin{align}
& \bigl\| \bar \theta^{(N)}(t) - \bar \theta^{(N)}(s) \bigr\|_{(N)} + \bigl\| \bar \ow^{(N)}(t) - \bar \ow^{(N)}(s) \bigr\|_{(N)}  \leq B \cdot \bigl| t - s\bigr|, \label{eq:ct-bd1} \\
& \bigl\| \tilde \theta^{(N)}(t) - \tilde \theta^{(N)}(s) \bigr\|_{(N)} + \bigl\| \tilde \ow^{(N)}(t) - \tilde \ow^{(N)}(s) \bigr\|_{(N)} \leq B \cdot \bigl| t - s\bigr|, \label{eq:ct-bd3} \\
& \cW_{2}(\mu_{t}, \mu_{s})) + \cW_{2}(\nu_{t}, \nu_{s})) \leq B \cdot \bigl| t - s\bigr|. \label{eq:ct-bd5}
\end{align}
\end{lemma}

\begin{proof}
For \eqref{eq:ct-bd1} of Lemma \ref{lem:weak-ct-bd}, by the definition of $\bar \theta_{i}(t)$ and $\bar \ow_{i}(t)$ in \eqref{eq:weak-ip} and \eqref{eq:bd11} of Lemma \ref{lem:weak-all-bd}, we have for any $s, t \in [0,T]$ and $i \in [N]$ that
\begin{align*}
& \bigl \| \bar \theta_{i}(t) - \bar \theta_{i}(s) \bigr\| \leq \eta \cdot \int_{s}^{t} \bigl\| v^{f}(\bar \theta_{i}(\tau); \rho_{\tau}) \bigr\| \mathrm{d} \tau \leq B \cdot \bigl| t -s \bigr| \\
& \bigl \| \bar \ow_{i}(t) - \bar \ow_{i}(s) \bigr\| \leq \eta \cdot \int_{s}^{t} \bigl\| v^{g}(\bar \ow_{i}(\tau); \rho_{\tau}) \bigr\| \mathrm{d} \tau \leq B \cdot \bigl| t -s \bigr|
\end{align*}
Similarly, for \eqref{eq:ct-bd3} of Lemma \ref{lem:weak-ct-bd}, by the definition of $\tilde \theta_{i}(t)$ and $\tilde \ow_{i}(t)$ in \eqref{eq:weak-ctpgd}, and \eqref{eq:bd7} of Lemma \ref{lem:weak-all-bd}, we have for any $s,t \in [0,T]$ and $i \in [N]$,
\begin{align*}
\big\|\tilde \theta_{i}(t) - \tilde \theta_{i}(s) \big\| \leq B \cdot \big|t -s\big|, \quad \big\|\tilde \ow_{i}(t) - \tilde \ow_{i}(s) \big\| \leq B \cdot \big|t -s\big|.
\end{align*}

\noindent For \eqref{eq:ct-bd5} of Lemma \ref{lem:weak-ct-bd}, following from the fact that $\bar{\theta}_i(t) \stackrel{\mathrm{i.i.d.}}{\sim} \mu_t$, $\bar{\ow}_i(t) \stackrel{\mathrm{i.i.d.}}{\sim} \nu_t$ and the definition of $\cW_{2}$ in \eqref{eq:w-p}, it holds that for any $s,t \in [0,T]$ that
\begin{align*}
& \cW_{2}(\mu_{t}, \mu_{s}) \leq \EE \Bigl[\bigl\| \bar \theta_{i}(t) - \bar \theta_{i}(s) \bigr\|^{2}\Bigr]^{1/2} \leq B \cdot |t -s|\\
& \cW_{2}(\nu_{t}, \nu_{s}) \leq \EE \Bigl[\bigl\| \bar \ow_{i}(t) - \bar \ow_{i}(s) \bigr\|^{2}\Bigr]^{1/2} \leq B \cdot |t -s|
\end{align*}
Therefore, we complete the proof of Lemma \ref{lem:weak-ct-bd}.
\end{proof}

\begin{lemma} \label{lem:mc-bd}
Let $\{X_{i}\}_{i = 1}^{N}$ be i.i.d. random variables with $\|X_{i}\| \leq \xi$ and $\EE[X_{i}] = 0.$ Then it holds for any $p > 0$, there exists $C>0$ being an absolute constant that
\begin{align*}
\mathbb{P}\left(\left\|N^{-1} \cdot \sum_{i=1}^N X_i\right\| \geq C \xi \cdot\left(N^{-1 / 2}+p\right)\right) \leq \exp \left(-N p^2\right),
\end{align*}

\end{lemma}

\begin{proof}
See Lemma 30 in \cite{mei2019mean}
\end{proof}

\begin{lemma}[Azuma-Hoeffding bound] \label{lem:ah-bd}
Let $X_{k} \in \RR^{D}$ be a martingale with respect to the filtration $\cG_{k} \; (k \geq 0)$ with $X_{0} = 0$. We assume for $\xi > 0$ and any $\lambda \in \RR^{D}$ that,
\begin{align*}
\mathbb{E}\left[\exp \left(\left\langle\lambda, X_k-X_{k-1}\right\rangle\right) \mid \mathcal{G}_{k-1}\right] \leq \exp \left(\xi^2 \cdot\|\lambda\|^2 / 2\right)
\end{align*}
Then it holds that, with $C > 0$ being an absolute constant.
\begin{align*}
\mathbb{P}\left(\max _{\substack{k \leq n \\(k \in \mathbb{N})}}\left\|X_k\right\| \geq C \xi \cdot \sqrt{n} \cdot(\sqrt{D}+p)\right) \leq \exp \left(-p^2\right)    
\end{align*}
\end{lemma}
\begin{proof}
See Lemma 31 in \cite{mei2019mean} and Lemma A.3 in \cite{araujo2019mean}.
\end{proof}

\section{Technical Results}\label{sec:tech}
\subsection{Universal Function Approximation Theorem} \label{sec:uat}
In what follows, we introduce the universal function approximation theorem \citep{pinkus1999approximation}. For any given activation function $\sigma: \RR\rightarrow \RR$, we consider the following function class,
\begin{align*}
	% \label{eq:uat-func}
	\cG(\sigma) = \Bigl\{ \sum_{i=1}^{r} c_i \sigma(x^\top w^i + \theta_i) \Biggiven c_i, \theta_i \in \RR, w^i \in \RR^d  \Bigr\}.
\end{align*}
We denote by $\sC(\RR^d)$ the class of continuous functions over $\RR^d$. Then, the following theorem holds.

\begin{lemma}[Universal Function Approximation Theorem, Theorem~3.1 in \cite{pinkus1999approximation}]
	\label{lem:uat}
	If the activation function $\sigma \in \sC(\RR)$ is not a polynomial, the function class $\cG(\sigma)$ is dense in $\sC(\RR^d)$ in the topology of uniform convergence on a compact set.
\end{lemma}

\subsection{Wasserstein Space}
We use the definition of absolutely continuous curves in $\sP_2(\RR^D)$ in \cite{ambrosio2008gradient} and introduce the following lemmas.
\begin{lemma}
	\label{lem:w2prod}
	For any probability measures $ \mu, \nu, \mu', \nu' \in \sP_2(\RR^D)$, it holds that
	\begin{align*}
		\cW_2(\mu \otimes\nu, \mu'\otimes \nu')^2 \le \cW_2(\mu, \mu')^2 +\cW_2(\nu, \nu')^2. 
	\end{align*} 
\end{lemma}
\begin{lemma}[First Variation Formula, Theorem 8.4.7 in \cite{ambrosio2008gradient}]
	\label{lem:diff}
	Given $\nu \in \sP_2(\RR^D)$ and an absolutely continuous curve $\mu: [0, T] \rightarrow \sP_2(\RR^D)$, let $\beta: [0,1] \rightarrow \sP_2(\RR^D)$ be the geodesic connecting $\mu_t$ and $\nu$. It holds that 
	\begin{align*}
		%\label{eq:variation}
		\frac{\rd }{\rd t}\frac{ \cW_2(\mu_t, \nu)^2}{2} = -\inp{\dot{\mu}_t}{\dot{\beta}_0}_{\mu_t}.
	\end{align*}
where $\dot{\mu_t} = \partial_t \mu_t$, $\dot{\beta}_0 = \partial_s \beta_s |_{s = 0}$.
\end{lemma}

\begin{lemma}[Benamou-Brenier formula, Proposition 2.30 in \cite{ambrosio2013user}]
	% \label{lem:bb}
	Let $\mu^0, \mu^1 \in \sP_2(\RR^D)$. Then, it holds that
	\begin{align*}
	% \label{eq:bb-lem}
	\cW_2(\mu^0, \mu^1) = \inf \biggl\{ \int_0^1 \norm{\dot{\mu}_t}_{\mu_t} \,\rd t \bigggiven \mu:[0,1] \rightarrow \sP_2(\RR^D), \mu_0 = \mu^0, \mu_1 = \mu^1 \biggr\}.
	\end{align*}
\end{lemma}

\begin{lemma}[Talagrand's Inequality, Corollary 2.1 in \cite{otto2000generalization}]
	\label{lem:talagrand}
	Let $\nu$ be $N(0, \kappa \cdot I_D)$. It holds for any $\mu \in \sP_2(\RR^D)$ that
	\begin{align*}
		\cW_2(\mu, \nu)^2 \le 2 D_{\rm KL}(\mu \,\|\, \nu) / \kappa.
	\end{align*}
\end{lemma}

\begin{lemma}[Eulerian Representation of Geodesics, Proposition 5.38 in \cite{villani2003topics}]
	\label{lem:euler}
	Let $\beta: [0, 1] \rightarrow \sP_2(\RR^D)$ be a geodesic and $u$ be the corresponding vector field such that $\partial_t \beta_t = - \dive(\beta_t \cdot u_t)$. It holds that
	\begin{align*}
		%\label{eq:lem-euler}
		\partial_t(\beta_t \cdot u_t) = - \dive(\beta_t \cdot u_t\otimes u_t).
	\end{align*}
where $\otimes$ is the outer product of two vectors. 
\end{lemma}

\begin{lemma}[Dual Representation of the first order Wasserstein Distance, \cite{villani2008optimal}]
	\label{lem:dual-w1}
The first order Wasserstein distance has the following dual representation form
\begin{align*}
	\cW_1(\mu, \nu) = \sup\biggl\{\int f(x) \rd (\mu - \nu)(x) \bigggiven f:\RR^D \rightarrow \RR \text{ that is 1-Lipschitz continuous}\biggr\}
\end{align*}
for any two probability measures \(\mu, \nu \in \sP_1(\RR^D)\).
\end{lemma}

\end{document}